\documentclass[runningheads]{llncs} 
\usepackage{eccv}
\usepackage{eccvabbrv}

\usepackage{subcaption}
\usepackage[export]{adjustbox}

\usepackage{fontawesome5}

\DeclareMathOperator*{\argmax}{arg\,max}

\newcommand{\tdgg}{\textsuperscript{\textdagger}}
\newcommand{\ttdgg}{\textsuperscript{\textdaggerdbl}}

\let\oldtimes\times
\def\times{{\mkern1mu\oldtimes\mkern1mu}}

\usepackage{graphicx}

\usepackage{tikz}

\usetikzlibrary{
  backgrounds,
  calc,
  fit,
  intersections,
  positioning,
  spy,
}

\newlength{\fsz}

\makeatletter
\pgfkeys{/pgf/.cd,
  parallelepiped offset x/.initial=2mm,
  parallelepiped offset y/.initial=2mm
}
\pgfdeclareshape{parallelepiped} {
  \inheritsavedanchors[from=rectangle] 
  \inheritanchorborder[from=rectangle]
  \inheritanchor[from=rectangle]{north}
  \inheritanchor[from=rectangle]{north west}
  \inheritanchor[from=rectangle]{north east}
  \inheritanchor[from=rectangle]{center}
  \inheritanchor[from=rectangle]{west}
  \inheritanchor[from=rectangle]{east}
  \inheritanchor[from=rectangle]{mid}
  \inheritanchor[from=rectangle]{mid west}
  \inheritanchor[from=rectangle]{mid east}
  \inheritanchor[from=rectangle]{base}
  \inheritanchor[from=rectangle]{base west}
  \inheritanchor[from=rectangle]{base east}
  \inheritanchor[from=rectangle]{south}
  \inheritanchor[from=rectangle]{south west}
  \inheritanchor[from=rectangle]{south east}
  \backgroundpath{
    \southwest \pgf@xa=\pgf@x \pgf@ya=\pgf@y
    \northeast \pgf@xb=\pgf@x \pgf@yb=\pgf@y
    \pgfmathsetlength\pgfutil@tempdima{\pgfkeysvalueof{/pgf/parallelepiped offset x}}
    \pgfmathsetlength\pgfutil@tempdimb{\pgfkeysvalueof{/pgf/parallelepiped offset y}}
    \def\ppd@offset{\pgfpoint{\pgfutil@tempdima}{\pgfutil@tempdimb}}
    \pgfpathmoveto{\pgfqpoint{\pgf@xa}{\pgf@ya}}
    \pgfpathlineto{\pgfqpoint{\pgf@xb}{\pgf@ya}}
    \pgfpathlineto{\pgfqpoint{\pgf@xb}{\pgf@yb}}
    \pgfpathlineto{\pgfqpoint{\pgf@xa}{\pgf@yb}}
    \pgfpathclose
    \pgfpathmoveto{\pgfqpoint{\pgf@xb}{\pgf@ya}}
    \pgfpathlineto{\pgfpointadd{\pgfpoint{\pgf@xb}{\pgf@ya}}{\ppd@offset}}
    \pgfpathlineto{\pgfpointadd{\pgfpoint{\pgf@xb}{\pgf@yb}}{\ppd@offset}}
    \pgfpathlineto{\pgfpointadd{\pgfpoint{\pgf@xa}{\pgf@yb}}{\ppd@offset}}
    \pgfpathlineto{\pgfqpoint{\pgf@xa}{\pgf@yb}}
    \pgfpathmoveto{\pgfqpoint{\pgf@xb}{\pgf@yb}}
    \pgfpathlineto{\pgfpointadd{\pgfpoint{\pgf@xb}{\pgf@yb}}{\ppd@offset}}
  }
}
\makeatother

\pgfdeclarelayer{background}
\pgfdeclarelayer{foreground}
\pgfdeclarelayer{lvl1}
\pgfdeclarelayer{lvl2}
\pgfdeclarelayer{lvl3}
\pgfsetlayers{background,main,foreground,lvl3,lvl2,lvl1} 

\usepackage{pgfplots}
\usepgfplotslibrary{
  colorbrewer,
}
\pgfplotsset{
  cycle list/Dark2,
}

\usepackage[l3]{csvsimple}

\usepackage{booktabs}

\colorlet{highlight}{BurntOrange!15}

\usepackage{tabularray}

\usepackage[inline,shortlabels]{enumitem}
\setlist{nosep, labelindent=0pt, leftmargin=*}

\usepackage[numbers,sort]{natbib}
\makeatletter
\def\NAT@spacechar{~}% NEW
\makeatother

\usepackage{amsmath,amsfonts,bm,amssymb,thmtools}
\usepackage{mdframed}

\usepackage[accsupp]{axessibility}  % Improves PDF readability for those with disabilities.

\usepackage[breaklinks]{hyperref}
\hypersetup{
  colorlinks,
  linkcolor = BrickRed,
  citecolor = RoyalBlue,
  urlcolor  = WildStrawberry,
}

\usepackage{orcidlink}
\usepackage{colortbl}
\UseTblrLibrary{
  booktabs,
  siunitx
}
\newcommand{\bftab}{\fontseries{b}\selectfont}
\newcommand{\bfs}{\bftab}
\robustify{\bfs}

\DeclareMathOperator{\spos}{{(\mathrm{pos})}}
\DeclareMathOperator{\scol}{{(\mathrm{col})}}
\DeclareMathOperator{\sgrad}{{(\mathrm{grad})}}
\usepackage{pifont}% http://ctan.org/pkg/pifont
\newcommand{\cmark}{\ding{51}}%
\newcommand{\xmark}{\ding{55}}%

\newcommand\nomarkfootnote[1]{%
  \begingroup
  \renewcommand\thefootnote{}\footnote{#1}%
  \addtocounter{footnote}{-1}%
  \endgroup
}

\pgfplotsset{compat=1.18}

\newif\ifsingle
\singletrue % set \singletrue for unified document

\newcommand{\appref}[1]{%
  \ifsingle
    \ref{#1}%
  \else
    \ref{app-#1}%
  \fi
}

\ifsingle
\else
    \usepackage{xr}
    \newcommand*{\addFileDependency}[1]{% argument=file name and extension
    \typeout{(#1)}
    \@addtofilelist{#1}
    \IfFileExists{#1}{}{\typeout{No file #1.}}
    }\makeatother
\fi
    
\begin{document}

\title{A Spitting Image: Modular Superpixel Tokenization in Vision Transformers} 
\titlerunning{Modular Superpixel Tokenization for ViTs}

\author{
Marius Aasan\inst{1,2}\orcidlink{0000-0003-2353-9984} \and
Odd Kolbjørnsen\inst{1,2,3}\orcidlink{0000-0001-8159-352X} \and
Anne Schistad Solberg\inst{1,2}\orcidlink{0000-0002-6149-971X} \and
Adín Ramirez Rivera\inst{1,2}\orcidlink{0000-0002-4321-9075}
}
\authorrunning{M.~Aasan et al.}
\institute{
University of Oslo, Box 1072 Blindern, 0316 Oslo, Norway
\email{\{mariuaas,anne,oddkol,adinr\}@uio.no}\\
\and
SFI Visual Intelligence, Box 6050 Langnes, 9037 Tromsø, Norway\\ \and
Aker BP ASA, Box 65, 1324 Lysaker, Norway
\email{odd.kolbjornsen@akerbp.com}\\
}

\maketitle

\begin{abstract}
    Vision Transformer (ViT) architectures traditionally employ a grid-based approach to tokenization independent of the semantic content of an image.
    We propose a modular superpixel tokenization strategy which decouples tokenization and feature extraction; a shift from contemporary approaches where these are treated as an undifferentiated whole.
    Using on-line content-aware tokenization and scale- and shape-invariant positional embeddings, we perform experiments and ablations that contrast our approach with patch-based tokenization and randomized partitions as baselines.
    We show that our method significantly improves the faithfulness of attributions, gives pixel-level granularity on zero-shot unsupervised dense prediction tasks, while maintaining predictive performance in classification tasks.
    Our approach provides a modular tokenization framework commensurable with standard architectures, extending the space of ViTs to a larger class of semantically-rich models.
    \keywords{ViT \and Tokenization \and Superpixels \and XAI \and Saliency}
    \protect\nomarkfootnote{Code available at: \url{https://github.com/dsb-ifi/SPiT}}
    \ifsingle
        \protect\nomarkfootnote{To appear in ECCV (MELEX) 2024 Workshop Proceedings.}
    \fi
\end{abstract}

\section{Introduction}
\label{sec:introduction}
Vision Transformers~\citep{origvit} (ViTs) have become the cynosure of vision tasks in the wake of convolutional architectures. 
In the original transformer for language~\citep{origtransformer,bert}, \textit{tokenization} serves as a crucial preprocessing step, with the aim of optimally partitioning data based on a predetermined entropic measure~\citep{bpetokenizer,wordpiece}. 
As models were adapted to vision, tokenization was simplified to partitioning images into square patches. 
This approach proved effective~\citep{swin,cait,deit1,deit3,deepervit,detr}, and soon became canonical; an integral part of the architecture. 

Despite apparent successes, we argue that patch-based tokenization has inherent limitations.
Firstly, the scale of the tokens are rigidly linked to the model architecture by a fixed patch size, ignoring any redundancy in the original images.
These limitations result in a significant increase in computation for larger resolutions, as complexity and memory scales quadratically with the number of tokens.
Moreover, regular partitioning assumes an inherent uniformity of the distribution of semantic content while effectively reducing spatial resolution.

Several works have since leveraged attention maps to visualize class token attributions for interpretability~\citep{dino,dinov2}, which has been exploited in dense prediction tasks~\citep{stego}. 
However, attention maps with square partitions incur a loss of resolution in the patch representation, and subsequently do not inherently capture the resolution of the original images.
For dense predictions with pixel level granularity, a separate decoder for upscaling is required~\citep{segformer,tokencut,segmentanything}. 

\setlength{\fsz}{0.19\linewidth}
\begin{figure}[tb]
\centering
\footnotesize
\begin{tblr}{
  colspec={Q[c,m]Q[c,m]Q[c,m]Q[c,m]Q[c,m]},
  rowsep=2pt,
  colsep=2pt,
  row{1}={font=\scriptsize},
  column{1}={font=\scriptsize},
}
& { Token.\ Image} & \textsc{Att.\ Flow} & \textsc{Proto.\ PCA} &  LIME (SLIC) \\
ViT &
\adjustbox{valign=m}{\includegraphics[width=\fsz]{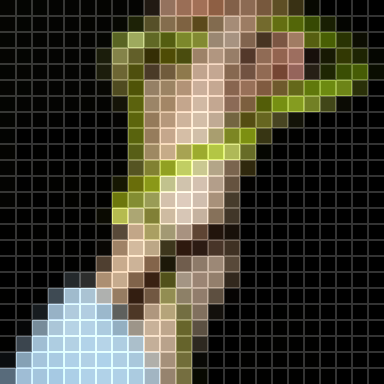}} &
\adjustbox{valign=m}{\includegraphics[width=\fsz]{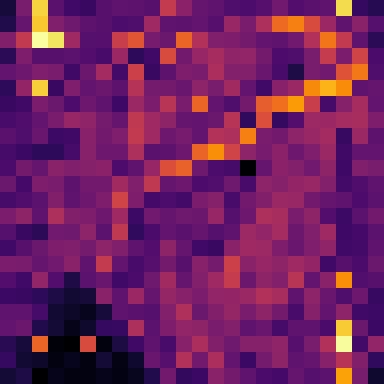}} &
\adjustbox{valign=m}{\includegraphics[width=\fsz]{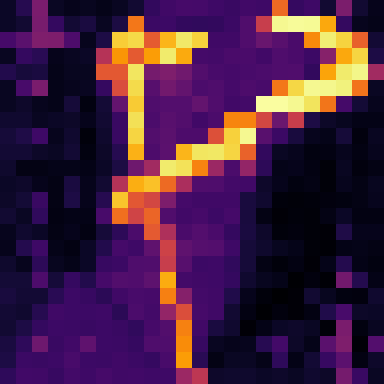}} &
\adjustbox{valign=m}{\includegraphics[width=\fsz]{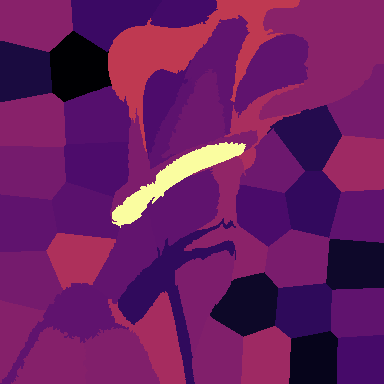}} \\
RViT &
\adjustbox{valign=m}{\includegraphics[width=\fsz]{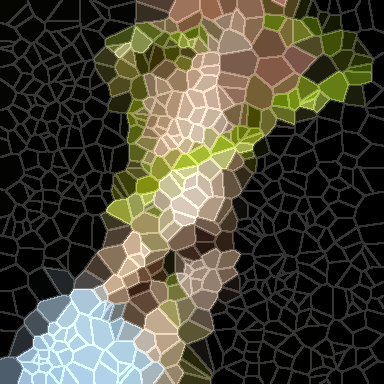}} &
\adjustbox{valign=m}{\includegraphics[width=\fsz]{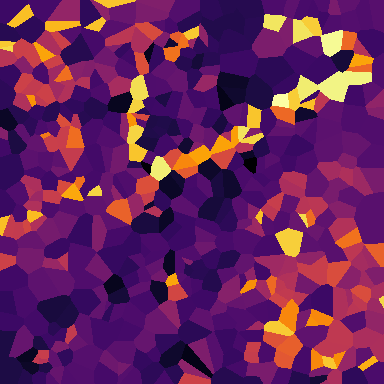}} &
\adjustbox{valign=m}{\includegraphics[width=\fsz]{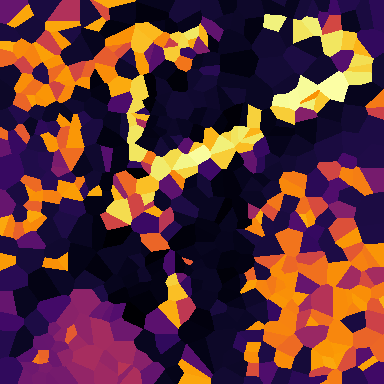}} &
\adjustbox{valign=m}{\includegraphics[width=\fsz]{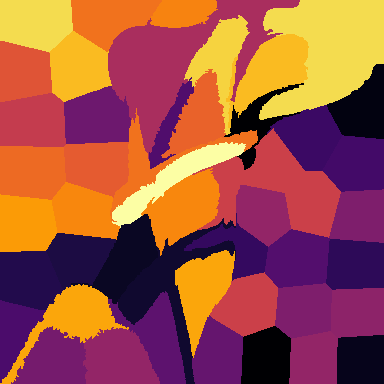}} \\
SPiT &
\adjustbox{valign=m}{\includegraphics[width=\fsz]{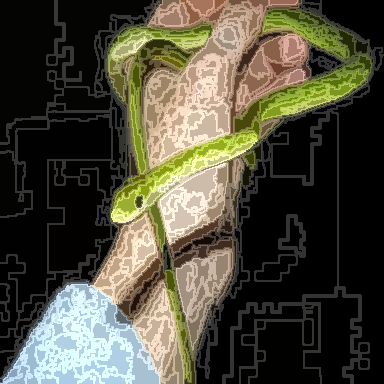}} &
\adjustbox{valign=m}{\includegraphics[width=\fsz]{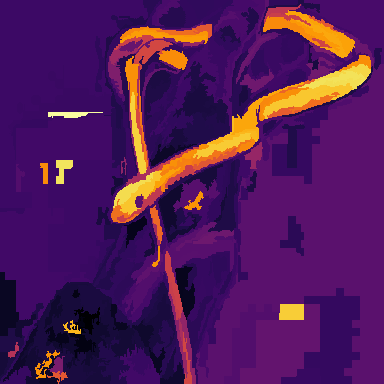}} &
\adjustbox{valign=m}{\includegraphics[width=\fsz]{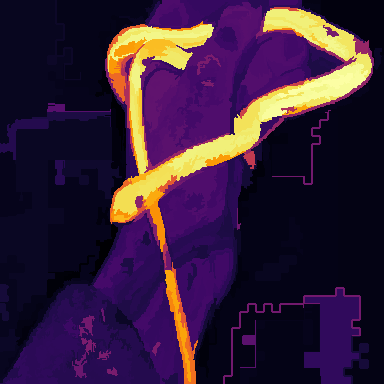}} &
\adjustbox{valign=m}{\includegraphics[width=\fsz]{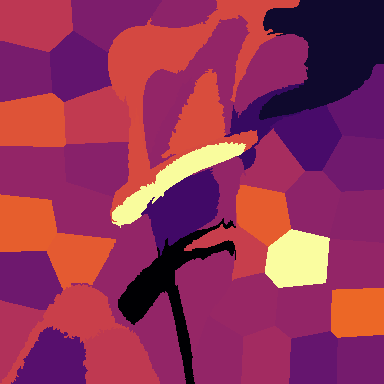}} \\%[-5pt]
\end{tblr}
\caption{
Tokenized image and attributions for prediction ``grass snake'' with different tokenizers: square patches (ViT), Voronoi tesselation (RViT) and superpixels (SPiT).  
We show more results in Appendix~\appref{sec:attmaps}.
% img1: ILSVRC2012_val_00016145
% img2: ILSVRC2012_val_00000278
% img3: ILSVRC2012_val_00000390
}
\label{fig:teaser}
\end{figure}

\subsection{Motivation}
\label{subsec:motivation}
We take a step back from the original ViT architecture to re-evaluate the role of patch-based tokenization.
By focusing on a somewhat overlooked component in the architecture, we look to establish image partitioning as the role of an \emph{adaptive modular tokenizer}; an untapped potential in ViTs.

In contrast to square partitions, \textit{superpixels} offer an opportunity to mitigate the shortcomings of patch-based tokenization by allowing for adaptability in scale and shape while leveraging inherent redundancies in visual data.
Superpixels have been shown to align better with semantic structures within images~\citep{superpixeleval}, providing a rationale for their potential utility in vision transformer architectures. 
We compare the canonical square tokenization in standard ViTs with our proposed superpixel tokenized model (SPiT) as well as a control using random Voronoi tokenization (RViT), selected for being well defined mathematical objects for tessellating a plane. 
The three tokenization schemes are illustrated in Fig.~\ref{fig:teaser}, and their innate segmentation capabilities in Fig.~\ref{fig:salient_segmentation}.

\subsection{Contributions}
\label{subsec:contributions}

Our research induces three specific inquiries: \textbf{(a)} \emph{Is a rigid adherence to square patches necessary?}, \textbf{(b)} \emph{What effect does irregular partitioning have on tokenized representations?}, and \textbf{(c)} \emph{Can tokenization schemes be designed as a modular component in vision models?} 
In this work we establish the following;
\begin{itemize}[nosep, left=0pt]
    \renewcommand\labelitemi{$\bullet$}
    \item \textbf{Generalized Framework:} Superpixel tokenization generalize ViTs in a modular scheme, providing a richer space of transformers for vision tasks where \textit{the transformer backbone is independent of tokenization framework}.
    \item \textbf{Efficient Tokenization:} We propose an efficient on-line tokenization approach which provides \textit{competitive training and inference times} as well as \textit{strong performance in classification tasks}.
    \item \textbf{Refined Spatial Resolution:} Superpixel tokenization provides semantically aligned tokens with pixel-level granularity. We demonstrate that our method yields \textit{significantly more faithful attributions compared to established explainability methods}, as well as \textit{strong results in unsupervised segmentation}.
    \item \textbf{Visual Tokenization:} The main contribution of our work is \emph{the introduction of a novel way of thinking about tokenization in ViTs}, an overlooked but central component of the modeling process---\cf discussion in Section~\ref{sec:discussion}.
\end{itemize}
\vspace{5pt}
Our primary objective is to evaluate tokenization schemes for ViTs, underscoring the intrinsic properties of alternative tokenization.
In the interest of a fair comparative analysis, \textit{we perform our study using vanilla ViT architectures and established training protocols}~\citep{howtrainvit}.
Hence, we design experiments to establish a fair comparison against well-known baselines \emph{without architectural optimizations}.
This controlled comparison is crucial for attributing observed disparities specifically to the tokenization strategy, and eliminates confounding factors from specialized architectures or training regimes.

\vspace{4.0pt}
\textbf{Notation:}
We let $H \times W = \big\{(y, x) : 1 \leq y \leq h, 1 \leq x \leq w\big\}$ denote the coordinates of an image of spatial dimension $(h, w)$, and let $\mathcal I$ be an index set for the mapping $i \mapsto (y, x)$.
We consider a $C$-channel image as a signal ${\xi\colon \mathcal I \to \mathbb R^C}$.
We use the vectorization operator $\mathrm{vec}\colon \mathbb{R}^{d_1 \times \dots \times d_n} \to \mathbb{R}^{d_1 \dots d_n}$, and denote function composition by $f(g(x)) = (f \circ g)(x)$.

\section{Methodology}
\label{sec:methodology}

To evaluate and contrast different tokenization strategies, we require methods for partitioning images and extracting meaningful features from these partitions. 
While these tasks can be performed using a variety of deep architectures, such approaches add a layer of complexity to the final model, which would invalidate any direct comparison between tokenization strategies.
Furthermore, this would also complicate any meaningful transfer learning between architectures.
In line with this reasoning, we construct an effective heuristic superpixel tokenizer, and propose an uninvasive feature extraction method which aligns with the canonical ViT architecture, and facilitates direct comparison.

\subsection{Framework}
\label{subsec:framework}

We generalize the canonical ViT architecture by allowing for a modular tokenizer and different methods of feature extraction.
Note that a canonical ViT is generally presented as a three-component system with a tokenizer-embedder~$g$, a backbone~$f$ consisting of a sequence of attention blocks, and a subsequent prediction head~$h$. 
Contrarily, language transformers explicitly decouples $g$ from the backbone $f$.
Following this lead, we note that we can essentially rewrite a patch embedding module as a three component modular system, featuring a tokenizer $\tau$, a feature extractor $\phi$, and an embedder $\gamma$ such that $g = \gamma \circ \phi \circ \tau$, emphasizing that these are inherent components in the original architecture obscured by a simplified tokenization strategy---\cf Fig~\ref{fig:zoomed-pipeline}.
This provides a more complete assessment of the model as a five component feedforward system
\begin{subequations}
    \label{eqn:pipeline}
    \begin{align}
        \Phi(\xi;\theta) &= (h \circ f \circ g)(\xi; \theta), \\
        &= (h \circ f \circ \gamma \circ \phi \circ \tau)(\xi; \theta),
    \end{align}
\end{subequations}
where $\theta$ denotes the set of learnable parameters of the model.
In a standard ViT model, the tokenizer $\tau$ acts by partitioning the image into fixed-size square partitions.
This directly provides vectorized features since patches are of uniform dimensionality and ordering, hence $\phi = \mathrm{vec}$ in standard ViT architectures.
The embedding $\gamma$ is typically a learnable linear layer, mapping features to the embedding dimension of the specific architecture.
Alternatively, $g$ can be taken as a convolution with kernel size and stride equal to the desired patch size $\rho$.

\tikzstyle{vertex}=[circle, draw=black, fill=white, line width=0.85mm, minimum size=25pt, inner sep=0pt]
\pgfplotstableread[col sep=comma]{exampleimggraph.csv}\coltable
\pgfplotstableread[col sep=comma]{edgelistgraph.csv}\edgetable
\pgfplotstableread[col sep=comma]{exampleimgoutlines.csv}\outtable

\begin{figure}[tb]
    \centering
    \begin{minipage}[c]{0.3\columnwidth}
    \resizebox{\columnwidth}{!}{
    \begin{tikzpicture}[
        node distance=0.5cm,
        every node/.append style={font=\footnotesize},
        proc/.style={
            draw,
            rectangle,
            rounded corners,
            minimum width=2em,
            minimum height=2em,
        },
        arrow/.style={
            ->,
            shorten >= 1pt,
            shorten <= 1pt,
        },
    ]
    % Standard ViT Pipeline
    \node[proc, fill=Salmon!50] (g) {\(g\)};
    \node[proc, fill=SeaGreen!50, right=of g] (backbone) {\(f\)};
    \node[proc, fill=SkyBlue!50, right=of backbone] (head) {\(h\)};
    \draw[arrow] (g) -- (backbone);
    \draw[arrow] (backbone) -- (head);
    % Zoomed-in box for g
    \node[proc, fill=Dandelion!50, below=1cm of g, xshift=0cm] (tokenizer) {\(\tau\)};
    \node[proc, fill=Goldenrod!50, right=of tokenizer] (phi) {\(\phi\)};
    \node[proc, fill=GreenYellow!50, right=of phi] (gamma) {\(\gamma\)};
    \draw[arrow] (tokenizer) -- (phi);
    \draw[arrow] (phi) -- (gamma);
    % Dashed enclosures
    \begin{scope}[on background layer]
    \node[draw, fill=Salmon!50, dashed, rounded corners, fit=(tokenizer) (phi) (gamma), inner sep=5pt] (zoom) {};
    \draw[dashed, rounded corners] (g) -- ++(0, -0.6) -| (zoom.north);
    \end{scope}
    % Annotations
    \node[anchor=north, above=0.2cm of g] {ViT Pipeline};
    \node[anchor=north, above=0.2cm of tokenizer] {Modular $g$};
    \end{tikzpicture}
    }
    \caption{Illustration of modular tokenization in ViT architecture.}
    \label{fig:zoomed-pipeline}
    \end{minipage}
    \hfill
    \begin{minipage}[c]{0.55\columnwidth}
    \resizebox{\columnwidth}{!}{
    \begin{tikzpicture}
      \foreach \x in {1,...,5} {
        \foreach \y in {1,...,5} {
         \pgfplotstablegetelem{\y}{R\x}\of\coltable
         \let\redch\pgfplotsretval
         \pgfplotstablegetelem{\y}{G\x}\of\coltable
         \let\greench\pgfplotsretval
         \pgfplotstablegetelem{\y}{B\x}\of\coltable
         \let\bluech\pgfplotsretval
         \pgfmathtruncatemacro{\th}{1.8/3*\redch+1.8/3*\greench+1.8/3*\bluech}
         \pgfplotstablegetelem{\y}{C\x}\of\outtable
         \let\outcol\pgfplotsretval         
         \ifnum\th = 0
            \definecolor{curtxtcol}{rgb}{1,1,1}
         \else
            \definecolor{curtxtcol}{rgb}{0,0,0}
         \fi
         \definecolor{curcol}{rgb}{\redch,\greench,\bluech}
         \node[vertex, color=\outcol, text=curtxtcol, fill=curcol] (n\x-\y) at (1.7*\x, -1.7*\y) {$v^{(1)}_{\y\x}$};
        }
      }
      \foreach \idxcnt in {0,...,39} {
          \pgfplotstablegetelem{\idxcnt}{ux}\of\edgetable
          \let\edgeux\pgfplotsretval
          \pgfplotstablegetelem{\idxcnt}{uy}\of\edgetable
          \let\edgeuy\pgfplotsretval
          \pgfplotstablegetelem{\idxcnt}{vx}\of\edgetable
          \let\edgevx\pgfplotsretval
          \pgfplotstablegetelem{\idxcnt}{vy}\of\edgetable
          \let\edgevy\pgfplotsretval
          \pgfplotstablegetelem{\idxcnt}{d}\of\edgetable
          \let\edgedir\pgfplotsretval
          \pgfplotstablegetelem{\idxcnt}{c}\of\edgetable
          \let\cossimval\pgfplotsretval
          \pgfplotstablegetelem{\idxcnt}{amu}\of\edgetable
          \let\amaxu\pgfplotsretval
          \pgfplotstablegetelem{\idxcnt}{amv}\of\edgetable
          \let\amaxv\pgfplotsretval
          \pgfmathtruncatemacro{\amaxcomb}{\amaxv + \amaxu}
          \pgfplotstablegetelem{\edgeuy}{C\edgeux}\of\outtable
          \let\outcol\pgfplotsretval
          \ifnum\edgedir=0
            \ifnum\amaxcomb > 0
              \draw[color=\outcol,line width=0.85mm] (n\edgeux-\edgeuy) -- (n\edgevx-\edgevy) node [midway, yshift=3mm, rotate=0] {\textcolor{black}{\tiny$\cossimval$}};
            \else
              \draw[dashed] (n\edgeux-\edgeuy) -- (n\edgevx-\edgevy) node [midway, yshift=3mm, rotate=0] {\tiny$\cossimval$};
            \fi
          \else
            \ifnum\amaxcomb > 0
              \draw[color=\outcol,line width=0.85mm] (n\edgeux-\edgeuy) -- (n\edgevx-\edgevy) node [midway, xshift=3mm, rotate=90] {\textcolor{black}{\tiny$\cossimval$}};
            \else
              \draw[dashed] (n\edgeux-\edgeuy) -- (n\edgevx-\edgevy) node [midway, xshift=3mm, rotate=90] {\tiny$\cossimval$};
            \fi
          \fi
      }
    \end{tikzpicture}
    \hspace{1cm}
    \begin{tikzpicture}
      \foreach \x in {1,...,5} {
        \foreach \y in {1,...,5} {
         \pgfplotstablegetelem{\y}{R\x}\of\coltable
         \let\redch\pgfplotsretval
         \pgfplotstablegetelem{\y}{G\x}\of\coltable
         \let\greench\pgfplotsretval
         \pgfplotstablegetelem{\y}{B\x}\of\coltable
         \let\bluech\pgfplotsretval
         \pgfmathtruncatemacro{\th}{1.8/3*\redch+1.8/3*\greench+1.8/3*\bluech}
         \pgfplotstablegetelem{\y}{C\x}\of\outtable
         \let\outcol\pgfplotsretval         
         \ifnum\th = 0
            \definecolor{curtxtcol}{rgb}{1,1,1}
         \else
            \definecolor{curtxtcol}{rgb}{0,0,0}
         \fi         
         \definecolor{curcol}{rgb}{\redch,\greench,\bluech}
         \ifnum \y = 1
            \ifnum \x = 3
                 \node[vertex, color=\outcol, text=curtxtcol, fill=curcol] (n\x-\y) at (1.7*\x-0.4, -1.7*\y-0.8) {$v^{(2)}_{1}$};
             \fi
         \fi
         \ifnum \y = 2
            \ifnum \x = 4
                 \node[vertex, color=\outcol, text=curtxtcol, fill=curcol] (n\x-\y) at (1.7*\x, -1.7*\y-0.8) {$v^{(2)}_{3}$};
             \fi
         \fi
         \ifnum \y = 3
            \ifnum \x = 3
                 \node[vertex, color=\outcol, text=curtxtcol, fill=curcol] (n\x-\y) at (1.7*\x-0.8, -1.7*\y) {$v^{(2)}_{2}$};
             \fi
         \fi
         \ifnum \y = 4
            \ifnum \x = 2
                 \node[vertex, color=\outcol, text=curtxtcol, fill=curcol] (n\x-\y) at (1.7*\x, -1.7*\y-0.8) {$v^{(2)}_{4}$};
             \fi
            \ifnum \x = 5
                 \node[vertex, color=\outcol, text=curtxtcol, fill=curcol] (n\x-\y) at (1.7*\x-0.8, -1.7*\y) {$v^{(2)}_{6}$};
             \fi
         \fi
         \ifnum \y = 5
            \ifnum \x = 4
                 \node[vertex, color=\outcol, text=curtxtcol, fill=curcol] (n\x-\y) at (1.7*\x, -1.7*\y) {$v^{(2)}_{5}$};
             \fi
         \fi
        }
      }
      \draw[dashed] (n3-1) -- (n4-2);
      \draw[dashed] (n3-1) -- (n3-3);
      \draw[dashed] (n4-2) -- (n3-3);
      \draw[dashed] (n4-2) -- (n5-4);
      \draw[dashed] (n3-3) -- (n2-4);
      \draw[dashed] (n2-4) -- (n4-5);
      \draw[dashed] (n5-4) -- (n4-5);
    \end{tikzpicture}
    }
    \caption{Visualization of superpixel aggregation.}
    \label{fig:spaggregate}
    \end{minipage}
\end{figure}

\subsection{Partitioning and Tokenization}
\label{subsec:partitioning}
Tokenization in language tasks involves partitioning text into optimally informative tokens, analogous to how superpixels~\citep{superpixeleval} partition spatial data into discrete connected regions.
Hierarchical superpixels~\citep{superpixelhierarchy,crtrees} are highly parallelizable graph-based approaches suitable for on-line tokenization.
We introduce a novel method that leverages fully parallel aggregation over batches of image graphs at each step $t$, in addition to regularization for size and compactness---\cf Appendix~\appref{sec:preproc-sp}.
Our method yields a variable number of superpixels at each step, adapting dynamically to the complexity of an image.

\subsubsection{Superpixel Graphs:}
Let $E^{(0)} \subset \mathcal I \times \mathcal I$ denote the four-way adjacency edges under $H \times W$.
We consider a superpixel as a set $S \subset \mathcal I$, and we say that $S$ is connected if for any two pixels $p, q \in S$, there exists a sequence of edges in $\big((i_j, i_{j+1}) \in E^{(0)}\big)_{j=1}^{k-1}$ such that $i_1 = p$ and $i_k = q$.
A set of superpixels form a partition $\pi$ of an image if for any two distinct superpixels $S, S' \in \pi$, their intersection $S \cap S' = \emptyset$, and the union of all superpixels is equal to the set of all pixel positions in the image, i.e., $\bigcup_{S \in \pi^{(t)}} S = \mathcal I$.

Let $\Pi(\mathcal I) \subset 2^{2^{\mathcal I}}$ denote the space of all partitions of an image, and consider a sequence of partitions $(\pi^{(t)})_{t=0}^T$.
We say that a partition $\pi^{(t)}$ is a refinement of another partition $\pi^{(t+1)}$ if for all superpixels $S \in \pi^{(t)}$ there exists a superpixel $S' \in \pi^{(t+1)}$ such that $S \subseteq S'$, and we write $\pi^{(t)} \sqsubseteq \pi^{(t+1)}$. Our goal is to construct a $T$-level hierarchical partitioning of the pixel indices ${\mathcal H = \big( \pi^{(t)} \in \Pi(\mathcal I) : \pi^{(t)} \sqsubseteq \pi^{(t+1)} \big)_{t=0}^T}$ such that each superpixel is connected.

To construct $\mathcal H$, the idea is to successively join vertices by parallel edge contraction to update the partition ${\pi^{(t)} \mapsto \pi^{(t+1)}}$.
We do this by considering each level of the hierarchy as a graph $G^{(t)}$ where each vertex $v \in V^{(t)}$ is the index of a superpixel in the partition $\pi^{(t)}$, and each edge $(u, v) \in E^{(t)}$ represent adjacent superpixels for levels $t = 0, \dots, T$.
The initial image can thus be represented as a grid graph ${G^{(0)} = (V^{(0)}, E^{(0)})}$ corresponding to the singleton partition ${\pi^{(0)} = \big\{\{i\} : i \in \mathcal I \big\}}$.

\subsubsection{Weight function:}
To apply the edge contraction, we define an edge weight functional $w_\xi^{(t)}\colon E^{(t)} \to \mathbb R$.
We retain self-loops in the graph to constrain regions by weighting loop edges by relative size.
This acts as a regularizer by constraining the variance of region sizes.
For non-loop edges, we use averaged features $\mu_\xi^{(t)}(v) = \sum_{i \in \pi^{(t)}_v} \xi(i) / \lvert \pi^{(t)}_v \rvert$ and apply a similarity function $\mathrm{sim}\colon E^{(t)} \to \mathbb{R}$.
Loops are weighted using the empirical mean $\mu^{(t)}_{\lvert \pi \rvert}$ and standard deviation $\sigma^{(t)}_{\lvert \pi \rvert}$ of region sizes at level $t$. 
This gives us weights on the form
\begin{align}
    w_\xi(u, v) = \begin{cases}
    \mathrm{sim}\Big(\mu_\xi^{(t)}(u), \mu_\xi^{(t)}(v)\Big), & \text{for $u \neq v$;}  \vspace{0.15cm}\\
    \Big(\lvert \pi^{(t)}_u \rvert - \mu_{\lvert \pi \rvert}^{(t)}\Big) / \sigma_{\lvert \pi \rvert}^{(t)}, & \text{otherwise.}
    \end{cases}
\end{align}
Compactness can optionally be regulated by computing the infinity norm density
\begin{equation}
\delta_\infty(u, v) = \frac{4 (\lvert \pi_u \rvert^{(t)} + \lvert \pi_v \rvert^{(t)})}{\mathrm{per}_\infty(u,v)^2},
\end{equation}
where $\mathrm{per}_\infty$ is the perimeter of the bounding box that encapsulates superpixels $u$ and $v$.
This emphasizes how tightly two neighbouring superpixels $u$ and $v$ are packed in their bounding box, resulting in a regularized weight functional
\begin{equation}
w_\xi^{(t)}(u,v;\lambda) = \lambda \delta_\infty(u,v) + (1 - \lambda)w_\xi^{(t)}(u, v)
\end{equation}
where $\lambda \in [0,1]$ serves as a hyperparameter for compactness.

\subsubsection{Update rule:}
We use a greedy parallel update rule for the edge contraction, such that each superpixel joins with a neighboring superpixel with the highest edge weights, including self-loops for all $G^{(t)}$ for $t \geq 1$.
Let $\mathfrak{N}^{(t)}(v)$ denote the neighborhood of adjacent vertices of the superpixel with index $v$ at level $t$.
We construct an intermediate set of edges, given by
\begin{align}
    \hat E^{(t)} = \bigg(v, \argmax_{u \in \mathfrak{N}^{(t)}(v)} w_\xi(u, v; \lambda) : v \in V^{(t)}\bigg).
\end{align}
Then the transitive closure $\hat E_+^{(t)}$, \ie the connected components of $\hat E^{(t)}$, explicitly yields a mapping ${V^{(t)} \mapsto V^{(t+1)}}$ such that 
\begin{align}
\pi^{(t+1)}_v = \bigcup_{u \in \hat{\mathfrak{N}}_+^{(t)}(v)} \pi^{(t)}_u,    
\end{align}
where $\hat{\mathfrak{N}}_+^{(t)}(v)$ denotes the connected component of vertex $v$ in $\hat E_+^{(t)}$.
This update rule for the partitions ensures that each partition at level $(t+1)$ is a connected region, as it is formed by merging adjacent superpixels with the highest edge weights.
We illustrate the aggregation step in Fig.~\ref{fig:spaggregate}.

\subsubsection{Iterative refinement:}
We repeat the steps of computing aggregation maps, regularized edge weights, and edge contraction until the desired number of hierarchical levels $T$ is reached.
At each level, the partitions become more coarse, representing larger homogeneous regions in the image.
The hierarchical structure provides a multiscale representation of the image, capturing both local and global structures.
At level $T$ we have obtained a sequence of partitions $(\pi^{(t)})_{t=0}^T$, where each partition at level $t$ is a connected region with ${\pi^{(t)} \sqsubseteq \pi^{(t+1)}}$ for all $t$.

We conduct experiments to empirically verify the relationship between the number of tokens produced by varying the steps $T$ and patch size $\rho$ in canonical ViT tokenizers.
Let $N_\mathrm{SPiT}, N_\mathrm{ViT}$ denote the number of tokens for the SPiT tokenizer and ViT tokenizer respectively.
Remarkably, we are able to show with a high degree of confidence that the relationship is $\mathbb{E}(T \mid N_\mathrm{SPiT} = N_\mathrm{ViT}) = \log_2 \rho$, \textit{regardless of image size}.
Details can be found in Appendix~\appref{sec:preproc-sp}.

\subsection{Feature Extraction with Irregular Patches}
While we conjecture the choice of square patches in the ViT architecture to be motivated by simplicity, it is naturally also a result of the challenge posed by the alternative.
Irregular patches are unaligned, exhibit different shapes and dimensionality, and are generally non-convex.
These factors make the embedding of irregular patches to a common inner product space nontrivial.
In addition to consistency and uniform dimensionality, we propose a minimal set of properties any such features would need to capture; \textit{color, texture, shape, scale},  and \textit{position}.

\subsubsection{Positional Encoding:}
ViTs generally use a learnable positional embedding for each patch in the image grid.
Noting that this corresponds to a histogram over positions over a downsampled image (\cf Prop.~\ref{prop:embedding_equiv}) we can extend learnable positional embeddings to handle more complex shapes, scales, and positions by using a kernelized approach.
We propose applying a joint histogram over the coordinates of a superpixel $S_n$ for each of the $n=1,\dots,N$ partitions.
First, we normalize the positions such that $(y', x') \in [-1, 1]^2$ for all $(y', x') \in S_n$.
We decide on a fixed number of bins $\beta$, denoting the dimensionality of our features in each spatial direction using a Gaussian kernel $K_\sigma$ such that
\begin{equation}
    \hat\xi^{\spos}_{n,y,x} = \mathrm{vec}\Bigg(\sum_{(y_j, x_j) \in S_n} K_\sigma (y - y_j, x - x_j) \Bigg),
\end{equation}
typically with low bandwith $\sigma \in [0.01, 0.05]$.
This, in effect, encodes the position of the patch within the image, as well as its shape and scale.

\subsubsection{Color Features:}
To encode the light intensity information from the raw pixel data into our features, we interpolate the bounding boxes of each patch to a fixed resolution of $\beta \times \beta$ using a bilinear interpolation operator, while masking out the pixel information in other surrounding patches.
These features essentially capture the raw pixel information of the original patches, but resampled and scaled to uniform dimensionality.
We refer to the feature extractor $\phi$ as an \textit{interpolating feature extractor}.
Similar to positional and texture features, the RGB features are normalized to $[-1, 1]$ and vectorized such that $\hat\xi^{\scol} \in \mathbb{R}^{3\beta^2}$.

\subsubsection{Texture Features:}
Gradient operators provides a simple robust method of extracting texture information~\citep{textons,hogfeat}.
We use the gradient operator proposed by \citet{scharr} due to improved rotational symmetry and discretization errors. 
We normalize the operator such that $\nabla \xi \in [-1, 1]^{H\times W\times 2}$, where the last dimensions correspond to gradient directions $\nabla y, \nabla x$.
Mirroring the procedure for the positional features, we then construct a joint histogram with a Gaussian kernel over the gradients within each superpixel $S_n$ such that $\hat\xi^{\sgrad}_n \in \mathbb{R}^{\beta^2}$.

The feature modalities are concatenated as $\hat\xi_n = [\hat\xi^{\scol}_n, \hat\xi^{\spos}_n, \hat\xi^{\sgrad}_n] \in \mathbb{R}^{5\beta^2}$.
While our proposed gradient features are commensurable with the canonical ViT architecture, they represent an additional dimension of information.
We therefore ablate the effect of including or omitting gradient features.
For models where these features are omitted, \ie $\hat\xi_n \setminus \hat\xi^{\sgrad}_n = [\hat\xi^{\scol}_n, \hat\xi^{\spos}_n] \in \mathbb{R}^{4\beta^2}$, we say that the extractor $\phi$ is \textit{gradient excluding}.

\subsection{Generalization of Canonical ViT}
\label{subsec:generalization}
By design, our framework acts as a generalization of the canonical ViT tokenization, and is equivalent to applying an canonical patch embedder using a fixed patch size $\rho$ with interpolated gradient excluding feature extraction.

\begin{restatable}[Embedding Equivalence]{proposition}{embedequiv}
    \label{prop:embedding_equiv}
    Let $\tau^*$ denote an canonical ViT tokenizer with a fixed patch size $\rho$, let $\phi$ denote a gradient excluding interpolated feature extractor, and let $\gamma^*, \gamma$ denote embedding layers with equivalent linear projections $L^*_\theta = L_\theta$.
Let $\hat\xi^{\spos} \in \mathbb{R}^{N \times \beta^2}$ denote a matrix of joint histogram positional embeddings under the partitioning induced by $\tau^*$.
Then for dimensions $H = W = \beta^2 = \rho^2$, the embeddings given by $\gamma \circ \phi \circ \tau^*$ are equivalent to the canonical ViT embeddings given by $\gamma^* \circ \phi^* \circ \tau^*$ up to proportionality.
\end{restatable}

We provide necessary definitions and proofs for Prop.~\ref{prop:embedding_equiv} in Appendix~\appref{sec:equivalence}, demonstrating that our proposed framework includes the canonical ViT architecture as a special case; an essential property for modularity.

\section{Experiments and Results}
\label{sec:experiments}

We train ViTs with different tokenization strategies (ViT, RViT, SPiT) using base (B) and small (S) capacities on a general purpose classification task on ImageNet~\citep{imagenet} (\textsc{IN1k}). 
We design our experiments with the goal of evaluating the quality of the resulting tokenized representations of the images. 
See details about the training setup in Appendix~\appref{sec:training-details}.

\begin{table*}[tb]
  \sisetup{detect-all,
    uncertainty-separator=\pm,
    table-format=1.3,
  }
  \caption{Accuracy (Top 1) for Base (B) capacity models on classification.} 
  \label{tab:results-main}
  \scriptsize
  \centering
  \begin{tabular}
  {l@{\hspace{.5em}}c@{\hspace{.25em}}c@{\hspace{.25em}}SSSSSSSS}
    \toprule
    \multicolumn{3}{c}{Model} &
    \multicolumn{2}{c}{\textsc{INReaL}} &
    \multicolumn{2}{c}{\textsc{IN1k}} &
    \multicolumn{2}{c}{\textsc{Caltech256}} &
    \multicolumn{2}{c}{\textsc{Cifar100}} \\
    \cmidrule(r){1-3} 
    \cmidrule(r){4-5} 
    \cmidrule(r){6-7} 
    \cmidrule(r){8-9} 
    \cmidrule(r){10-11} 
    Name   & Grad. & {Im./s.\ttdgg} & {Lin.} & {kNN} & {Lin.} & {kNN} & {Lin.} & {kNN} & {Lin.} & {kNN} \\
    \midrule
    ViT-B16      &\xmark&793.04&    .853&    .849&    .802&    .737&    .879&    .879&    .892&    .897\\
    ViT-B16      &\cmark&721.12&    .854&    .844&\bfs.805&    .748&\bfs.889&    .885&\bfs.899&\bfs.899\\
    RViT-B16\tdgg&\xmark&619.86&    .843&    .832&    .788&    .718&    .873&    .882&    .894&    .838\\
    RViT-B16\tdgg&\cmark&585.64&    .841&    .836&    .789&    .725&    .864&    .861&    .888&    .762\\
    SPiT-B16     &\xmark&690.72&    .793&    .818&    .760&    .569&    .833&    .829&    .813&    .634\\
    SPiT-B16     &\cmark&640.59&\bfs.858&\bfs.853&    .804&\bfs.752&    .888&\bfs.891&    .884&    .845\\
    \bottomrule
    \multicolumn{11}{l}{\tiny \tdgg Uncertainty measures for RViT are detailed in Appendix Table~\appref{tab:voronoi_uncertainty}.} \\
    \multicolumn{11}{l}{\tiny \ttdgg Median throughput over full training with $4\times$ MI250X GPUs using \texttt{float32} precision.} \\
  \end{tabular}
\end{table*}

\subsection{Classification}
\label{subsec:classification}

We evaluate the models by fine-tuning on \textsc{Cifar100}~\citep{cifar} and \textsc{Caltech256}~\citep{caltech256}, in addition to validation using the \textsc{INReaL} labels~\citep{imagenetreal}, ablating the effect of gradient features.
We also evaluate our models by replacing the linear classifier head with a k-nearest neighbours (kNN) classifier over the representation space of different models, focusing solely on the clustering quality of the class tokens in the embedded space~\cite{dino,dinov2}.
Table~\ref{tab:results-main} gives an overview of the results.
We include results for the Small (S) capacity models in Table~\appref{tab:results-small}.

Our results show that ViTs with superpixel tokenization can be effectively trained for classification tasks.
For models with gradient texture features, superpixel tokenization performs comparably to square partitioning, noting that superpixel tokenization with gradient excluding feature extraction underperforms.
We conjecture that this is likely due to high irregularity in regions, and confirms our conjecture that gradient features can compensate for loss of information from interpolation.
Our findings in Section~\ref{subsec:generalization} also supports this.

When comparing validation results, we note that SPiT performs better than the ViT over \textsc{INReaL}. 
This indicates that the model is more robust to label noise or localized-multiclass tasks, and likely generalizes better in real-world scenarios.
This is further evident by the fact that SPiT performs better with kNN classification for higher resolution images in \textsc{IN1k} and \textsc{CalTech256} than the ViT model.
We note that square tokens perform better on \textsc{Cifar100}. \emph{This is to be expected} as quantization artifacts from low resolution images persist under upscaling, favoring square patches.

Overall, our results indicates that SPiT with gradient features outperforms the vanilla ViT in classification tasks. 
However, when including our proposed gradient features in the standard ViT, \emph{the results are not significant enough to claim a clear benefit on general purpose classification tasks}.
We emphasize that \emph{comparable performance is a positive result}, since our focus is on demonstrating the feasibility of modular superpixel tokenization as a new research direction for vision transformers.
For more details, see Appendix~\appref{sec:ext-discussion}.

\subsection{Evaluating Tokenized Representations}
\label{sec:dense_evaluation}
To evaluate the cohesive quality of the tokenized representations, we look to quantify the \emph{faithfulness of attributions}, and the model's performance on \emph{zero-shot unsupervised segmentation}.
These were selected to give insight into the embedded context of the tokenized representation of the image.

\subsubsection{Faithfulness of Attributions:}
One of the attractive properties of ViTs is the inherent interpretability provided by their attention mechanisms.
Techniques such as attention rollout~\citep{origvit,dino}, attention flow~\citep{attrolloutflow}, and PCA projections~\citep{dinov2} have been leveraged to visualize the reasoning behind the model's decisions.
Unlike gradient-based attributions, which often lack clear causal links to model predictions~\citep{sanitysaliency}, attention based attributions are intrinsically connected to the flow of information in the model, and provide direct insight into the decision-making process in an interpretable manner. 
They are, however, constrained by the granularity and semantic alignment of the original tokenization scheme.
Classical methods such as LIME~\citep{lime} provides a well-established counterfactual framework for post-hoc explainability with superpixel partitions using Quickshift~\citep{quickshift} or SLIC~\citep{slic} with local linear surrogate models.

\begin{table*}[tb]
  \sisetup{detect-all,
    uncertainty-separator=\pm,
    table-format=1.3(3),
    uncertainty-mode=separate,
  }
  \caption{Faithfulness of Attributions, w. CI (95\%).}
  \label{tab:interpret_quant}
  \scriptsize
  \centering
  \begin{tabular}{lSSSSSS}
    \toprule
    \multicolumn{1}{c}{} & \multicolumn{2}{c}{{ViT-B16 (\textsc{IN1k})}} & \multicolumn{2}{c}{{RViT-B16 (\textsc{IN1k})}} & \multicolumn{2}{c}{{SPiT-B16 (\textsc{IN1k})}} \\
    \cmidrule(r){2-3} \cmidrule(r){4-5} \cmidrule(r){6-7}
    {} &
    $\textsc{Comp} \uparrow$ & $\textsc{Suff} \downarrow$ &
    $\textsc{Comp} \uparrow$ & $\textsc{Suff} \downarrow$ &
    $\textsc{Comp} \uparrow$ & $\textsc{Suff} \downarrow$ \\
    \midrule
    \textsc{LIME/SLIC} &\bfs\cellcolor{Dark2-C!15} .244(.004) &\bfs\cellcolor{Dark2-C!15}.543(.006) &  % ViT
                        \bfs\cellcolor{Dark2-C!15}.236(.004) &\bfs\cellcolor{Dark2-C!15}.591(.007) &  % RViT
                         \cellcolor{Dark2-C!15}.244(.005) &\bfs\cellcolor{Dark2-C!15}.520(.006) \\ % SPiT
    \textsc{Att.Flow}  & \cellcolor{Dark2-B!15}.160(.004) & \cellcolor{Dark2-B!15}.664(.006) &  % ViT
                         \cellcolor{Dark2-B!15}.223(.005) & \cellcolor{Dark2-B!15}.685(.007) &  % RViT
                        \bfs\cellcolor{Dark2-A!15}.259(.006) & \cellcolor{Dark2-B!15}.558(.006) \\ % SPiT
    \textsc{Prot.PCA}  & \cellcolor{Dark2-B!15}.206(.005) & \cellcolor{Dark2-B!15}.710(.006) &  % ViT
                         \cellcolor{Dark2-B!15}.209(.005) & \cellcolor{Dark2-B!15}.691(.007) &  % RViT
                         \cellcolor{Dark2-A!15}.256(.005) & \cellcolor{Dark2-B!15}.592(.006) \\ % SPiT
    \bottomrule
    \multicolumn{7}{l}{\tiny \textbf{Color coding:} \color{Dark2-C}{baseline}, \color{Dark2-B}{weaker than baseline}, \color{Dark2-A}{stronger than baseline}.}
  \end{tabular}
\end{table*}

\begin{table*}[tb]
  \colorlet{downcolor}{black!55}
  \sisetup{detect-all,
    uncertainty-separator=\pm,
    table-format=1.3,
    uncertainty-mode=separate,
  }
  \caption{Results for unsupervised salient segmentation with TokenCut. Models using additional postprocessing are included for completeness are colored in \textcolor{downcolor}{gray}.} 
  \label{tab:salient_segmentation}
  \scriptsize
  \centering
  \begin{tabular}{l@{ }@{ }c@{ }@{ }SSSSSSSSS}
    \toprule
    & & \multicolumn{3}{c}{\textsc{ECSSD}} & \multicolumn{3}{c}{\textsc{DUTS}} & \multicolumn{3}{c}{\textsc{DUT-OMRON}} \\
    \cmidrule(r){3-5} \cmidrule(r){6-8} \cmidrule(r){9-11}
    Model & Postproc. & {$\max F_\beta$}  & {IoU} & {Acc.} & {$\max F_\beta$}  & {IoU} & {Acc.} & {$\max F_\beta$}  & {IoU} & {Acc.} \\
    \midrule
    \color{downcolor}DINO-B14\tdgg & \color{downcolor}\cmark
    & \color{downcolor}0.874 & \color{downcolor}0.772 & \color{downcolor}\bfs0.934
    & \color{downcolor}0.755 & \color{downcolor}0.624 & \color{downcolor}\bfs0.914
    & \color{downcolor}0.697 & \color{downcolor}\bfs0.618 & \color{downcolor}\bfs 0.897\\
    DINO-B14\tdgg & \xmark
    & 0.803 & 0.712 & 0.918
    & 0.672 & 0.576 & 0.903
    & 0.600 & 0.533 & 0.880 \\
    SPiT-B16 & \xmark
    & \bfs0.903 & \bfs0.773 & \bfs0.934
    & \bfs0.771 & \bfs0.639 & 0.894
    & \bfs0.711 & 0.564     & 0.868\\
    \bottomrule
    \multicolumn{11}{l}{\tiny \tdgg As reported by \citet{tokencut}.}
  \end{tabular}
\end{table*}

To quantify the faithfulness of interpretations under different tokenization strategies, we compute the attention flow of the model in addition to PCA projected features and contrast this with attributions from LIME with independently computed SLIC superpixels, and measure faithfulness using \textit{comprehensiveness} (\textsc{Comp}) and \textit{sufficiency} (\textsc{Suff})~\citep{erasercompsuff}. 
These metrics have been shown to be the two strongest quantitative measures of attributions for transformers~\citep{faithfulness}.
See Appendix~\appref{sec:attmaps} for details.

The results in Table~\ref{tab:interpret_quant} suggests that predictions extracted from the attention flow and PCA using the SPiT model provide \textit{better comprehensiveness scores} than interpretations from LIME, indicating that SPiT models produce attributions that more effectively exclude irrelevant regions of the image. 
A one-sided $t$-test confirms that the improvement in comprehensiveness between \textsc{Att.Flow} and LIME for the SPiT model is statistically significant.\footnote{One-sided $t$-test (\textsc{Att.Flow} $>$ LIME): $(t=6.54, p < 10^{-10}, \mathrm{df}=49664)$.}
Contrarily, our results show that interpretations extracted from the ViT and RViT models are \emph{less faithful to the predictions than interpretations} procured with LIME\@.
Furthermore, we note that the sufficiency score for SPiT models are closer to the baseline LIME interpretations than what we observe for the ViT, indicating that the interpretations from SPiT model captures the most essential features better than a canonical ViT\@.
Figs.~\ref{fig:teaser}, \appref{fig:attention-maps}, \appref{fig:attention-maps-2}, \appref{fig:occlusion}, and \appref{fig:occlusion2} shows that the granularity of superpixel tokens provide interpretations that closely align with the semantic content of the image.

\definecolor{my-yellow}{RGB}{255,255,191}
\definecolor{my-orange}{RGB}{252,141,89}
\definecolor{my-blue}{RGB}{145,191,219}
\setlength{\fsz}{0.16\linewidth}
\begin{figure}[tb]
    \centering
    \begin{tabular}{c@{\,\,}@{}c@{\,\,}@{}c@{\,\,}@{}c}
        \includegraphics[width=\fsz]{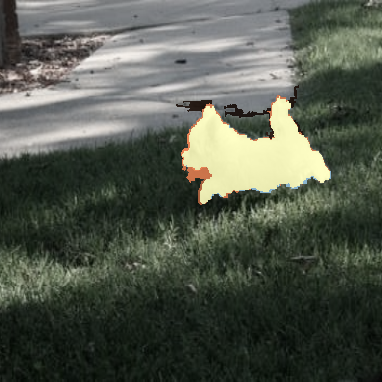}&
        \includegraphics[width=\fsz]{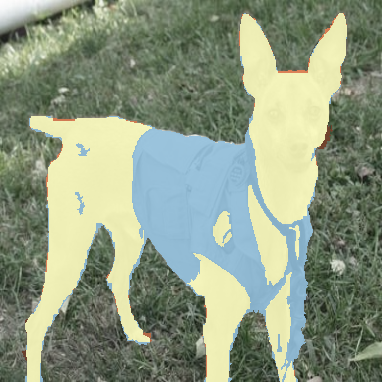}&
        \includegraphics[width=\fsz]{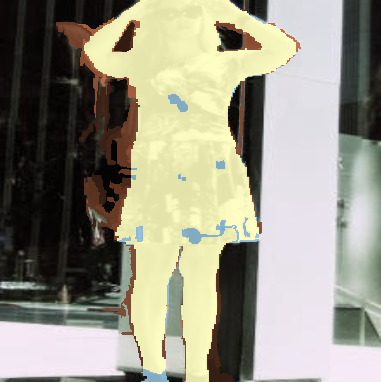}&
        \includegraphics[width=\fsz]{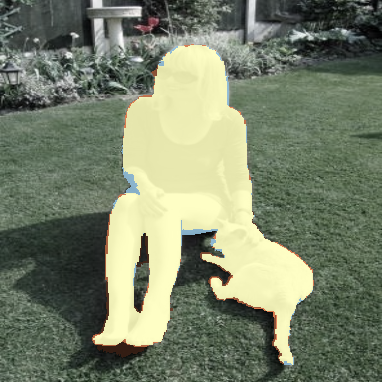} \\
        \includegraphics[width=\fsz]{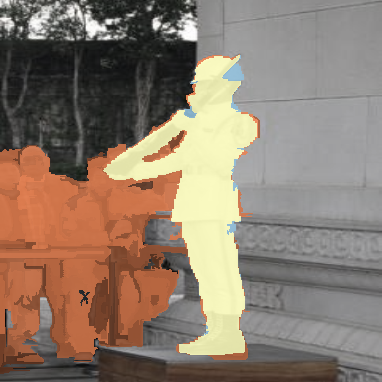}&
        \includegraphics[width=\fsz]{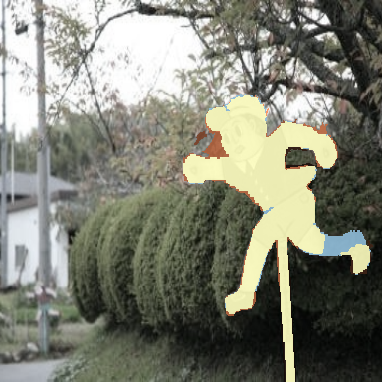}&
        \includegraphics[width=\fsz]{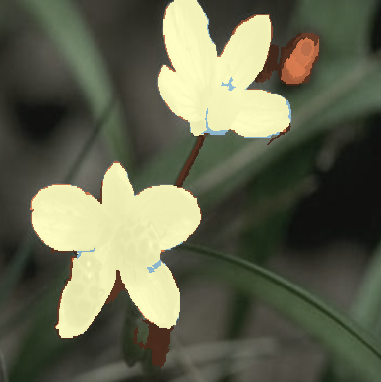}&
        \includegraphics[width=\fsz]{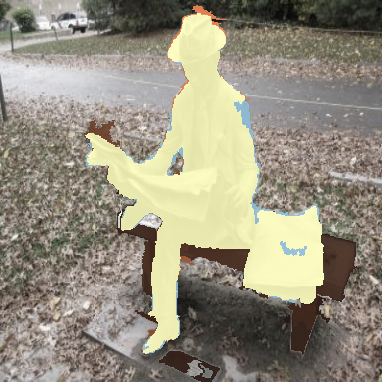} \\
        \multicolumn{4}{r}{\tiny{\colorbox{my-yellow}{\textcolor{black}{True positives}} \colorbox{my-orange}{\textcolor{black}{False positives}}  \colorbox{my-blue}{\textcolor{black}{False negatives}}}}
    \end{tabular}
    \caption{Non-cherry picked samples (\texttt{\{0257..0264\}.jpg}) of unsupervised zero-shot segmentation results on \textsc{ECSSD}\@.}
    \label{fig:salient_segmentation}
    \vspace{-10pt}% pull the text below to avoid the window
\end{figure}

\subsubsection{Unsupervised Segmentation:}
Superpixels have historically been applied in dense prediction tasks such as segmentation and object detection~\citep{spcrfseg,spobjdet} as a lower-dimensional prior for dense prediction tasks.
To evaluate our tokens, we are particularly interested in tasks for which the outputs of the pre-trained model can be leveraged directly, without the addition of a downstream decoder.
\citet{tokencut} propose an unsupervised methodology for extracting salient segmentation maps for any transformer model using normalized graph cut~\citep{normalizedcut}.
We conduct experiments extending this well-established method to showcase preliminary out-of-the-box capabilities on dense prediction tasks, with details of the experimental setup in Appendix~\appref{sec:segmentation}.

Table~\ref{tab:salient_segmentation} shows results for the ECSSD~\citep{ecssd}, DUTS~\citep{duts} and {DUT-OMRON}~\citep{dutomron} datasets, and demonstrates that SPiT compares favorably to DINO~\citep{dino} under the TokenCut framework, \emph{notably without any form of postprocessing}.
The results indicate that our tokenizer has strong semantic alignment with image content, and that our proposed framework is capable of dense predictions without learnable tokenization.
We use the same metrics as the original TokenCut framework; for $\max F_\beta$ we set $\beta = 1/3$ and take the maximum $F$-score over 255 uniformly sampled thresholds.
A series of non-cherry picked results are featured in Fig.~\ref{fig:salient_segmentation}.

\subsection{Ablations}

\subsubsection{Tokenizer Generalization:}
\label{subsec:tokenizer_generalization}
In in Section~\ref{subsec:generalization} we showed that our framework generalizes the canonical ViT\@.
This allows us to contrast different tokenization strategies across models by directly swapping tokenizers, emphasizing the modularity of our framework.
We report the relative change in accuracy ({$\Delta$ Acc.}) of models when swapping tokenizers in Table~\ref{tab:tokenizer_generalization}.

Our results show that ViTs with square tokenization performs poorly when evaluated on irregular patches.
We observe an increase in accuracy for RViT models when evaluated over square patches.
Furthermore, we see that the SPiT models also generalize well to both to square and Voronoi tokens, but is highly dependent on the gradient features.
With gradient features, we note a minor drop in accuracy when evaluating Voronoi tokens with SPiT, and superpixel tokens with RViT.
This supports our conjecture that gradient features help encode texture information for irregular patches.

\begin{table}[tb]
  \centering%
  \begin{minipage}[t]{0.49\linewidth}%
    \centering%
    \caption{Tokenizer Generalization.}%
    \label{tab:tokenizer_generalization}%
    \scriptsize%
    \begin{tabular}{l@{ }@{}S@{ }@{}S@{ }@{}S@{ }@{}S}
      \toprule
      \multicolumn{2}{c}{} & \multicolumn{3}{c}{$\Delta$ Acc. $\uparrow$ (\textsc{IN1k})} \\
      \cmidrule(r){3-5}
      Model &
      {Grad.} &
      {ViT} &
      {RViT} &
      {SPiT} \\
      \midrule
      ViT-B16 & \xmark &\color{gray}.000&    -.551&    -.801\\
      ViT-B16 & \cmark &\color{gray}.000&    -.494&    -.798\\
      RViT-B16& \xmark &\bfs.006&\color{gray}.000&    -.593\\
      RViT-B16& \cmark &    .003&\color{gray}.000&\bfs-.163\\
      SPiT-B16& \xmark &   -.407&    -.464&\color{gray}.000\\
      SPiT-B16& \cmark &   -.200&\bfs-.063&\color{gray}.000\\
      \bottomrule
    \end{tabular}%
  \end{minipage}
  \begin{minipage}[t]{0.49\linewidth}%
    \centering%
    \caption{Superpixel Evaluation.}%
    \label{tab:superpixel_eval}%
    \scriptsize%
    \begin{tabular}{l@{}@{ }c@{}@{ }c@{}@{ }c@{}@{ }c@{}@{ }c@{}}
      \toprule
      \multicolumn{1}{c}{} & \multicolumn{2}{c}{\textsc{BSDS500}} & \multicolumn{2}{c}{\textsc{SBD}} & \multicolumn{1}{c}{Time} \\
      \cmidrule(r){2-3} \cmidrule(r){4-5} \cmidrule(r){6-6}
      {} & $R^2\hspace{-2pt}\uparrow$ & $\mathbb \lvert \pi \rvert\downarrow$ & $R^2\hspace{-2pt}\uparrow$ & $\lvert \pi \rvert\downarrow$ & s/Im. $\downarrow$ \\
      \midrule
      ETPS\tdgg     & 0.924 & 651.0 & 0.955 & 648.1 &  0.3268 \\
      SEEDS\tdgg    & 0.901 & 670.6 & 0.944 & 644.9 &  0.4501 \\
      SLIC\tdgg     & 0.847 & 575.3 & 0.897 & 592.2 &  0.0729 \\
      Watershed\tdgg& 0.803 & 608.1 & 0.871 & 641.1 &  0.0038 \\
      SPiT          & 0.914 & 595.0 & 0.948 & 570.2 &  0.0047 \\
      \bottomrule
      \multicolumn{6}{l}{%
      \tiny \tdgg As reported by \citet{superpixeleval}.
      }\\
    \end{tabular}%
  \end{minipage}%
\end{table}

\setlength{\fsz}{0.165\linewidth}
\begin{figure}[tb]
\centering
\footnotesize
\begin{tblr}{
  colspec={Q[c,m]Q[c,m]Q[c,m]Q[c,m]Q[c,m]},
  rowsep=1.5pt,
  colsep=1.5pt,
  column{1} = {font=\scriptsize},
  column{2} = {leftsep=10pt},
  row{1} = {font=\scriptsize},
}
\SetCell[c=1]{c} \textbf{Source} & 
\SetCell[c=4]{c} \textbf{Feature Correspondences} \\
\adjustbox{valign=m}{\includegraphics[width=\fsz]{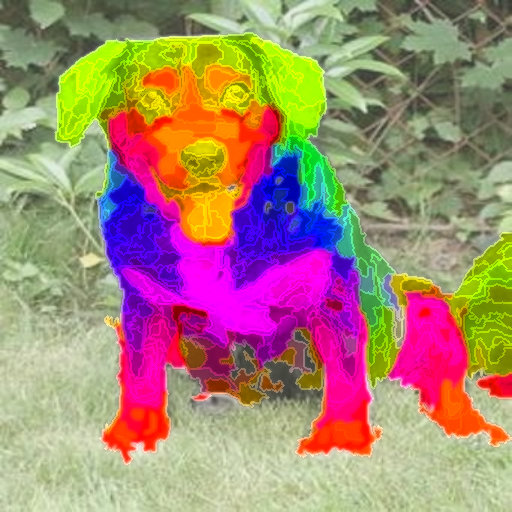}} &
\adjustbox{valign=m}{\includegraphics[width=\fsz]{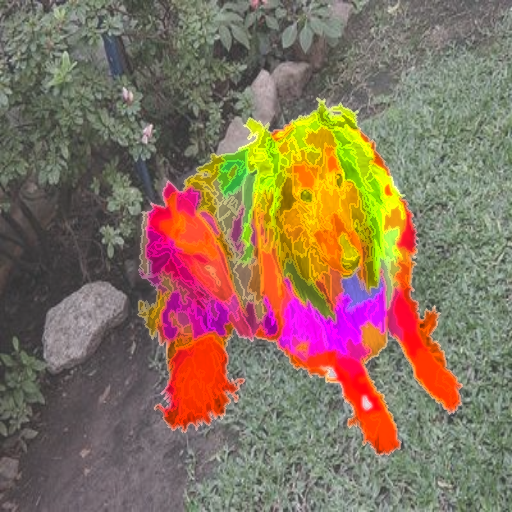}} &
\adjustbox{valign=m}{\includegraphics[width=\fsz]{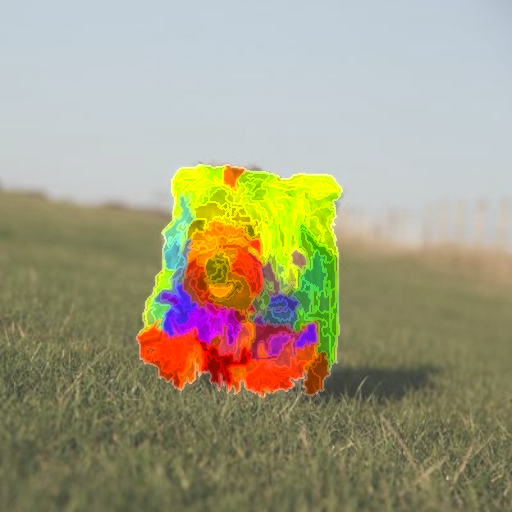}} &
\adjustbox{valign=m}{\includegraphics[width=\fsz]{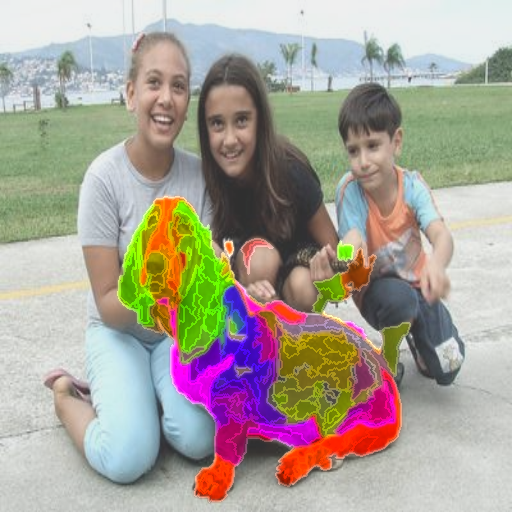}} &
\adjustbox{valign=m}{\includegraphics[width=\fsz]{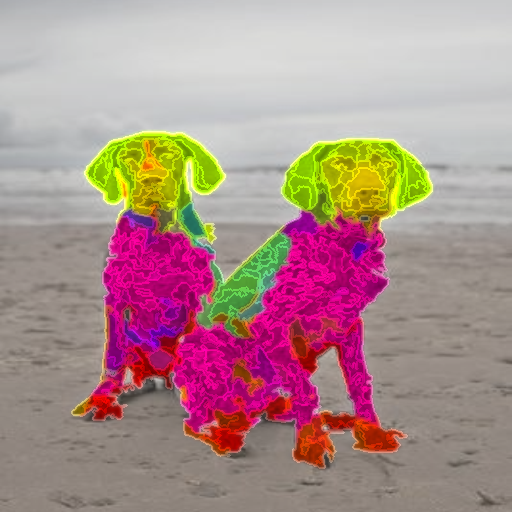}}
\end{tblr}
\caption{
Feature correspondences from a source image (left) to target images (right), mapped via normalized single head cross attention and colored using low rank PCA\@.  
We show more results in Appendix~\appref{sec:ftcorr}. 
}
\label{fig:featcorr_onerow}
\end{figure}

\subsubsection{Quality of Superpixels}
To evaluate the quality of superpixels, we compute the explained variation \citep{splattice,superpixeleval} given by
\begin{align}
    R^2 (\pi \mid \xi)\!=\!\frac{1}{\mathrm{Var}(\xi)}
    \!\sum_{S \in \pi}
    \mathrm{Pr}(S)
    \big(\mathbb E(\xi \cap S) - \mathbb E(\xi)\big)^2,
\end{align}
where $\mathrm{Pr}(S) = \lvert S \rvert/\lvert \xi \rvert$.
The explained variation quantifies how well the superpixels capture the inherent structures in an image by measuring the amount of dispersion which can be attributed to the partitioning $\pi$.
An ideal algorithm would produce a high $R^2$ with a minimal number of superpixels.
We compare our approach with SotA superpixel methods~\citep{superpixeleval} in Table~\ref{tab:superpixel_eval}, demonstrating that our superpixel algorithm performs comparably to top performing methods with substantially lower inference time, which is crucial for on-line tokenization.

\subsubsection{Feature Correspondences}
\citet{dinov2} visualize feature correspondences between images to examine the consistency of token representations across images for models trained with contrastive learning. 
Given the strong attribution scores of superpixel tokenization, we were interested to see how features correspond across images with similar, but not necessarily identical classes. 
We compute cross attention over normalized features between a source and target images, and visualize the correspondences using a low rank PCA with three channels.
Figs.~\ref{fig:featcorr_onerow}, \appref{fig:featcorr_ext}, and~\appref{fig:featcorr_ext2} demonstrates that the features from SPiT provide strong feature correspondence properties without self-supervised pretraining, which is generally considered to provide more robust representations independent of downstream tasks.

\section{Discussion and Related Work}
\label{sec:discussion}

\subsubsection{Related Work}
Interest in adaptive tokenization is burgeoning in the field.
We propose a taxonomy of adaptive tokenization with two main dimensions illustrated in Fig.~\ref{fig:dimensions_transformers}.
The first dimension illustrates the \emph{coupling or integration} of tokenization into the transformer architecture.
Several approaches~\citep{imageaspoints, supertoken, tome} are inherently coupled to the architecture, while others adopt a decoupled approach~\citep{msvit, quadformer} which more closely aligns with our framework.
The taxonomy is extended by a dimension of \emph{token granularity}, measuring the proximity to modelling with pixel-level precision.
Together, these dimensions facilitate an understanding of adaptive tokenization approaches for ViTs.

A significant body of current research is primarily designed to improve scaling and overall compute for attention~\citep{tome,tokenlearner,tokens2token} by leveraging token merging strategies in the transformer layers with square patches, and can as such be considered \emph{low-granularity coupled approaches}.
Distinctively, SuperToken~\citep{supertoken} applies a coupled approach to extract a non-uniform token representation.
The approach is fundamentally patch based, and does not aim for pixel-level granularity.

In contrast, multi-scale tokenization~\citep{msvit,quadformer} apply a \emph{decoupled approach} where the tokenizer is independent of the transformer architecture.
These are commensurable with \emph{any transformer backbone}, and improve computational overhead.
While square tokens operate on a \emph{lower level of granularity}, there is significant potential for synergy between these approaches and our own, particularly given the hierarchical nature of SPiT\@.
On the periphery, \citet{imageaspoints} propose a pixel-level clustering method with a \emph{coupled high granularity approach}.

\begin{figure}[tb]
    \centering
    \resizebox{0.5\linewidth}{!}{
    \setlength{\baselineskip}{5pt}
    \begin{tikzpicture}[
        my label/.style={
            text width=#1,
            execute at begin node=\setlength{\baselineskip}{1pt},
        },
        my label/.default=2cm,
        outer sep=2pt,
    ]
    \begin{axis}[
        xlabel={\textbf{Token Granularity}},
        ylabel={\textbf{Architectural Coupling}},
        xmin=0, xmax=1,
        ymin=0, ymax=1,
        xtick={0.,0.5,...,1},
        ytick={0.,0.5,...,1},
        xticklabels={,,,,,},
        yticklabels={,,,,,},
        x tick label style={
            font=\footnotesize,
        },
        y tick label style={
            font=\footnotesize,
        },
        x label style={yshift= 0em},
        y label style={yshift=-0em},
        ymajorgrids=true,
        xmajorgrids=true,
        grid style=dashed,
        clip=false,
    ]

    % Points
    \addplot[only marks, mark=text, text mark={\textcolor{Dark2-B}{\scriptsize\faIcon{eye}}}] coordinates {
        (0.85,    0.2)%     {SPiT}
    };
    \addplot[only marks, mark=text, text mark={\scriptsize\faIcon{eye}}] coordinates {
        (0.1,     0.7)%     {ToMe}
        (0.5,     0.9)%     {SuperToken}
        (0.1,     0.3)%     {MSViT}
        (0.4,     0.2)%     {Quadformer}
        (0.95,    0.9)%     {ContextCluster}
    };
    \addplot[only marks, mark=text, text mark={\scriptsize\faIcon{book-open}}] coordinates {
        (0.8,     0.1)%    {LLMs}
    };

    % Labels
    \node[my label, align=left] at (axis cs:0.85, 0.2) [anchor=west] {\textcolor{Dark2-B}{\textbf{SPiT}}};

    \node[my label, align=left] at (axis cs:0.1, 0.7) [anchor=west] {ToMe\\\tiny\citep{tome}};
    \node[my label, align=right] at (axis cs:0.5, 0.9) [anchor=east] {SuperToken\\\tiny\citep{supertoken}};

    \node[my label=2.5cm,] at (axis cs:0.1, 0.3) [anchor=west] {MSViT\\\tiny\citep{msvit}};
    \node[my label,] at (axis cs:0.4, 0.2) [anchor=west] {Quadformer\\\tiny\citep{quadformer}};

    \node[my label=2.5cm, align=right] at (axis cs:0.95, 0.9) [anchor=east] {ContextCluster\\\tiny\citep{imageaspoints}};
    \node[my label=5cm, align=right] at (axis cs:0.8, 0.1) [anchor=east] {GPT-3\\\tiny\citep{gpt3}};

    % Scalemarkers
    \node at (axis cs:0.,-.1) [anchor=west] {\large\faIcon{cube}};
    \node at (axis cs:1.,-.1) [anchor=east] {\large\faIcon{cubes}};
    \node at (axis cs:-.1,0.) [anchor=south] {\large\faIcon{unlink}};
    \node at (axis cs:-.1,1.) [anchor=north] {\large\faIcon{link}};

    \end{axis}
    % set the bounding box
    \end{tikzpicture}}
    \caption{\footnotesize Taxonomy of adaptive tokenization in transformers. Tokenization ranges from decoupled (\scriptsize\faIcon{unlink}\footnotesize) to coupled (\scriptsize\faIcon{link}\footnotesize) to the transformer architecture, and from coarse (\scriptsize\faIcon{cube}\footnotesize) to fine (\scriptsize\faIcon{cubes}\footnotesize) token granularity. To contextualize vision models (\scriptsize\faIcon{eye}\footnotesize) with LLMs (\scriptsize\faIcon{book-open}\footnotesize), GPT-3~\citep{gpt3} is included for reference.}
    \label{fig:dimensions_transformers}
    \vspace{-10pt}
\end{figure}
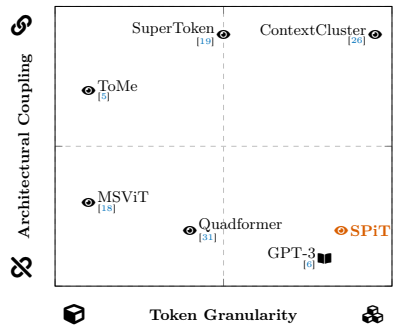

\subsubsection{Limitations}
Our proposed framework is not optimizable with gradient based methods.
Ideally, adaptable tokenization should be learnable in an end-to-end framework.
However, such an approach needs to be carefully designed to not add undue computational overhead, and should ideally not be limited by a predefined number of tokens.
Moreover, we see that irregular tokenization require additional gradient features to perform.
While our framework provides competitive performance, it should be seen as an early step towards more flexible tokenization strategies, with several opportunities for further optimization.
We provide visualization of edge cases for attributions in Fig.~\appref{fig:attention-maps-cont}.

\subsubsection{Further Work} 
Our work is distinguishable as a \textit{decoupled high-granularity apprach} with multiple paths for further work.
We see strong potential in exploring graph neural networks (GNNs) for tokenization, and hierarchical properties could be leveraged in self-supervised frameworks such as DINO~\citep{dino}, or pyramid models~\citep{pvt1,pvt2} in a coupled approach.
The modularity of our framework provides opportunites for research into the dynamic between ViTs and tokenization.
Coupling SPiT with gating~\citep{msvit} or merging~\citep{tome} could further improve scalability, and allow for a learnable framework.
More work can be done in studying the effects of irregularity in feature extraction, as discussed in Section~\ref{subsec:tokenizer_generalization}.

\section{Conclusion}
In this work, we posit tokenization as a modular component that generalize the canonical ViT backbone, and show that irregular tokenization with superpixels is commensurable with transformer architectures.
Our experiments demonstrate that superpixel tokens have a significant impact on extracted attributions for predictions, and are amenable to unsupervised segmentation tasks without a separate decoder model.
Moreover, we show that concatenated gradient features improve performance of base capacity ViTs, and that irregular tokenizers generalize between different tokenization strategies.
Our experiments were performed with standard models and training to limit confounding factors in our results.

\section*{Acknowledgments}

Computations were performed on resources provided by Sigma2 (NRIS, Norway), Project~NN8104K.
We acknowledge Sigma2 for awarding access to the LUMI supercomputer, owned by the EuroHPC Joint Undertaking, hosted by CSC (Finland) and the LUMI consortium through Sigma2, Norway, Project no.~465000262.
This work was funded in part by the Research Council of Norway, via the Visual Intelligence Centre for Research-based Innovation (grant no.~309439), and Consortium Partners.

\clearpage
\nocite{regionsmooth,peronamalik,cutmix,wordnet,imagenetreal,origvit,dinov2,erasercompsuff,tokencut,normalizedcut}
\begingroup
    \let\clearpage\relax
    \vspace{20pt}
    \bibliographystyle{splncs04nat}
    \bibliography{abrv,main}
\endgroup

\ifsingle
    \clearpage
    \appendix

\newcommand{\maincitet}[1]{%
  \ifsingle
    \citet{#1}%
  \else
    \citet{main-#1}%
  \fi
}
\newcommand{\maincitep}[1]{%
  \ifsingle
    \citep{#1}%
  \else
    \citep{main-#1}%
  \fi
}
\newcommand{\mainref}[1]{%
  \ifsingle
    \ref{#1}%
  \else
    \ref{main-#1}%
  \fi
}

\renewcommand\thefigure{\thesection\arabic{figure}}
\renewcommand{\thetable}{\thesection\arabic{table}}
\counterwithin{figure}{section}
\counterwithin{table}{section}

\section{Equivalence of Frameworks}
\label{sec:equivalence}
\begin{definition}[ViT Tokenization]
    Let ${\xi\colon H \times W \to \mathbb{R}^{C}}$ be an image signal with tensor representation ${\vec\xi \in \mathbb R^{H \times W \times C}}$.
    The canonical ViT tokenization operator $\tau^*\colon \mathbb{R}^{H \times W \times C} \to \mathbb{R}^{N \times \rho \times \rho \times C}$ partitions the image into $N = \lceil\frac{H}{\rho}\rceil \cdot \lceil \frac{W}{\rho} \rceil$ non-overlapping $C$-channel square zero-padded patches. For the case where we have $H \bmod \rho = W \bmod \rho = 0$, we get $N = \frac{H}{\rho} \cdot \frac{W}{\rho}$, and no padding is necessary.
\end{definition}

\begin{definition}[ViT Features]
    Let $\rho$ denote the patch dimension of a canonical ViT tokenizer $\tau^*$, and let $M = \rho^2 C$.
    The canonical ViT feature extractor ${\phi^*\colon \mathbb{R}^{N \times \rho \times \rho \times C} \to \mathbb{R}^{N \times M}}$ is given by ${\phi^* = \mathrm{vec}_M}$, where $\mathrm{vec}_M$ denotes the vectorization operator applied to each of the $N$ patches via $\rho \times \rho \times C \mapsto M$.
\end{definition}

\begin{definition}[ViT Embedder]
    Let $\phi^*$ be a canonical ViT feature extractor, and let $Q \in \mathbb{R}^{N \times D}$ denote a positional encoding. The canonical ViT embedder ${\gamma^*\colon \mathbb{R}^{N\times M} \rightarrow \mathbb{R}^{N\times D}}$ is given by
    \[\gamma^*(z) = L_\theta z + Q\]
    where $L_\theta\colon \mathbb{R}^{N\times M} \rightarrow \mathbb{R}^{N\times D}$ is a learnable linear transformation, and $Q$ is either a learnable set of parameters or a function of the positions of the $N$ blocks in the partitioning induced by the canonical tokenizer $\tau^*$.
\end{definition}

\begin{lemma}[Feature Equivalence]
    \label{prop:tokenizer_feature_equiv}
    Let $\tau^*$ denote a canonical ViT tokenizer with a fixed patch size $\rho$, and let $\phi$ denote a gradient excluding interpolating feature extractor with $\beta = \rho$.
    Then the operations $\phi \circ \tau^*$ are equivalent to the canonical ViT operations $\phi^* \circ \tau^*$.
\end{lemma}
\begin{proof}
    The proof is highly trivial but illustrative.
    Note that for each of the $N$ square patches generated by $\tau$, the extractor $\phi$ performs an interpolation to rescale the patch to a fixed resolution of $\beta \times \beta$.
    However, for $\beta = \rho$ the patches already match the target dimensions exactly.
    It follows that the interpolation operation reduces to identity.
    The vectorization operator is equivalent for both mappings, hence $\phi = \mathrm{vec}_N = \phi^*$.
\end{proof}

\ifsingle
    \embedequiv*
\else
    \begin{proposition}[Embedding Equivalence]
        Let $\tau^*$ denote an canonical ViT tokenizer with a fixed patch size $\rho$, let $\phi$ denote a gradient excluding interpolated feature extractor, and let $\gamma^*, \gamma$ denote embedding layers with equivalent linear projections $L^*_\theta = L_\theta$.
Let $\hat\xi^{\spos} \in \mathbb{R}^{N \times \beta^2}$ denote a matrix of joint histogram positional embeddings under the partitioning induced by $\tau^*$.
Then for dimensions $H = W = \beta^2 = \rho^2$, the embeddings given by $\gamma \circ \phi \circ \tau^*$ are equivalent to the canonical ViT embeddings given by $\gamma^* \circ \phi^* \circ \tau^*$ up to proportionality.
    \end{proposition}
\fi
\begin{proof}
    We first note that we can assume $\hat\xi^{\spos}$ is a matrix with single entry components, since under $\beta = \rho$ and $N = \beta^2$, each vectorized histogram feature is a scaled unit vector $c_n \vec{e}_n$ with ${n = 1,\dots,N}$.
    Moreover, since the partitioning inferred by $\tau^*$ exhaustively covers the spatial dimensions ${H \times W}$, the histograms essentially span the standard basis, such that $\hat\xi^{\spos}$ is diagonal.
    Furthermore, since each patch is of the same size we have equal contribution towards each entry, such that $c_n = c_m$ for all $m \neq n$.
    Therefore, without loss of generality, we can ignore the scalars and simply consider $\hat\xi^{\spos} = I$ as an identity matrix.
    From Lemma~\ref{prop:tokenizer_feature_equiv} we have that $z = ({\phi^*} \circ {\tau^*})(\xi) = ({\phi} \circ {\tau^*})(\xi)$.
    Then, since
    \begin{align}
        \gamma^*(z) = L_\theta z + Q &= [L_\theta, Q]\genfrac{[}{]}{0pt}{}{z}{I} = \gamma(z)
    \end{align}
     we have that $\gamma = \gamma^*$ up to proportionality for some constant $c = c_n$.
\end{proof}

\begin{remark}
    While we only demonstrate the equality up to proportionality, this can generally be ignored since we can effectively choose our linear projection under $\gamma$ to be $L_\theta / c$.
    We note that while the equality holds for empirical histograms, equality does not strictly hold for $\hat\xi^{\spos}$ computed using KDE with a Gaussian kernel, however we point out that the contribution from the tails of a kernel $K_\sigma$ with a small bandwidth is effectively negligible.
\end{remark}

\section{Preprocessing and Superpixel Features}
\label{sec:preproc-sp}

Compared to standard preprocessing, we use a modified normalization scheme for the features for improving the superpixel extraction.
We apply a combined contrast adjustment and normalization function using a reparametrized version of the Kumaraswamy CDF. which is computationally efficient and allows more fine-grained control of the distribution of intensities than empirical normalization, which improves the superpixel partitioning.

The normalization uses a set of means \(\mu\) shape parameters \( \lambda \) for normalizing the image and adjusting the contrast.
The normalization is given by
\begin{equation}
    \left(1 - \left(1 - x^\lambda \right)^b \right),
\end{equation}
where \( b \) is defined by
\begin{equation}
    b = -\frac{\ln(2)}{\ln\left(1 - \mu^\lambda \right)},
\end{equation}
and we set means $\mu_r = 0.485, \mu_g = 0.456, \mu_b = 0.406$ and $\lambda_r = 0.539, \lambda_g = 0.507, \lambda_b = 0.404$, respectively.

The features used for the superpixel extraction are further processed using anisotropic diffusion, which smoothes homogeneous regions while avoiding blurring of edges.
This technique was advocated for superpixel segmentation by \maincitet{regionsmooth}.
We use the algorithm proposed by \maincitet{peronamalik} over 4 iterations, with $\kappa=0.1$ and $\gamma=0.5$.
Note that these features are only applied for constructing the superpixels in the tokenizer.
We emphasize that we do not apply anisotropic diffusion for the features in the predictive model.

\begin{table*}[tb]
  \sisetup{detect-all,
    uncertainty-separator=\pm,
    table-format=5.3(1.3),
    uncertainty-mode=separate,
  }
  \caption{Expected no. superpixels with SPiT over \textsc{IN1k} (train, CI 95\%).}
  \label{tab:empirical_spsizes}
  \scriptsize
  \centering
  \begin{tabular}{l@{ }@{ }SSSS}
    \toprule
    Im.Size &  {$\mathbb E(\lvert \pi^{(1)} \rvert)$}  & {$\mathbb E(\lvert \pi^{(2)} \rvert)$} & {$\mathbb E(\lvert \pi^{(3)} \rvert)$}  & {$\mathbb E(\lvert \pi^{(4)} \rvert)$} \\
    \midrule
    224 & 11940.278(2.848)  & 3155.512(0.808) & 794.650(0.209) & 197.411(0.052)  \\
    256 & 15496.020(3.786)  & 4097.510(1.074) & 1031.727(0.277) & 256.051(0.071)  \\
    384 & 34084.297(9.188)  & 9047.289(2.586) & 2287.822(0.669) & 567.69(0.172)  \\
    \bottomrule
  \end{tabular}
\end{table*}

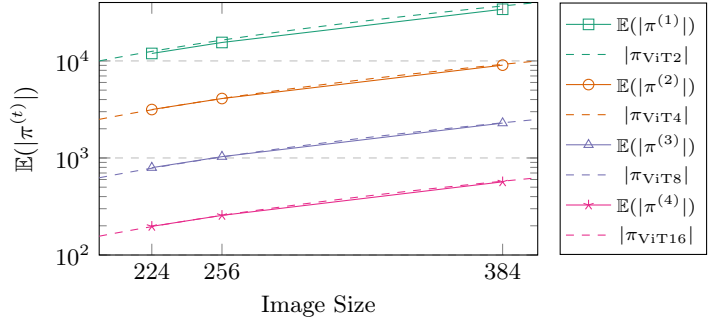
\begin{figure}[tb]
\centering
\begin{tikzpicture}[]
\begin{axis}[
    font=\footnotesize,
    width=.6075\linewidth,
    height=.405\linewidth,
    xlabel={Image Size},
    ylabel={$\mathbb E(\lvert \pi^{(t)} \rvert)$},
    xmin=200, xmax=400,
    ymin=100, ymax=40000,% can't compute log of 0
    xtick={224,256,384},
    ytick={0,100,1000,10000},
    ymode=log,
    legend style={
        at={(1.05,1)},
        anchor=north west,
        font=\scriptsize,
        inner sep=2pt,
    },
    legend columns=1,
    ymajorgrids=true,
    grid style=dashed,
]

% t=1
\addplot[
    color=Dark2-A,
    mark=square,
    ]
    coordinates {
    (224,11940.278)(256,15496.020)(384,34084.297)
    };
    \addlegendentry{$\mathbb E(\lvert \pi^{(1)} \rvert)$}

% Log2 patch size 2
\addplot[
    color=Dark2-A,
    mark=none,
    dashed,
    domain=200:400,
    samples=100,
    ]
    {(x/2)^2};
    \addlegendentry{$\lvert \pi_\mathrm{ViT2}\rvert$}

% t=2
\addplot[
    color=Dark2-B,
    mark=o,
    ]
    coordinates {
    (224,3155.512)(256,4097.510)(384,9047.289)
    };
    \addlegendentry{$\mathbb E(\lvert \pi^{(2)} \rvert)$}

\addplot[
    color=Dark2-B,
    mark=none,
    dashed,
    domain=200:400,
    samples=100,
    ]
    {(x/4)^2};
    \addlegendentry{$\lvert \pi_\mathrm{ViT4}\rvert$}

% t=3
\addplot[
    color=Dark2-C,
    mark=triangle,
    ]
    coordinates {
    (224,794.650)(256,1031.727)(384,2287.822)
    };
    \addlegendentry{$\mathbb E(\lvert \pi^{(3)} \rvert)$}

\addplot[
    color=Dark2-C,
    mark=none,
    dashed,
    domain=200:400,
    samples=100,
    ]
    {(x/8)^2};
    \addlegendentry{$\lvert \pi_\mathrm{ViT8}\rvert$}

% t=4
\addplot[
    color=Dark2-D,
    mark=star,
    ]
    coordinates {
    (224,197.411)(256,256.051)(384,567.69)
    };
    \addlegendentry{$\mathbb E(\lvert \pi^{(4)} \rvert)$}

\addplot[
    color=Dark2-D,
    mark=none,
    dashed,
    domain=200:400,
    samples=100,
    ]
    {(x/16)^2};
    \addlegendentry{$\lvert \pi_\mathrm{ViT16}\rvert$}

\end{axis}
% set the bounding box to the axis
\pgfresetboundingbox
\path (current axis.below south west) rectangle (current axis.above north east);
\end{tikzpicture}
\caption{Expected no. superpixels with SPiT compared with no. ViT patches.}
\label{fig:empirical_spsizes}
\end{figure}

\subsubsection{Number of Superpixels:} In Section~\mainref{subsec:partitioning}, we mention that SPiT gives comparable numbers of partitions to a ViT with different patch sizes.
Table~\ref{tab:empirical_spsizes} shows empirical results for superpixel sizes using the SPiT tokenizer over the training images of \textsc{ImageNet1k}, and Fig.~\ref{fig:empirical_spsizes} compares the results to number of patches with canonical ViT tokenization, demonstrating the validity of our claims.

Importantly, these results also reveal much about effective inference times.
In Table~\mainref{tab:superpixel_eval}, we show that the overhead for constructing the superpixels is very low.
However, the number of tokens depends on the image.
Images with large homogeneous regions will be processed faster, while images with many independent regions will necessary incur a cost.
Nevertheless, the results in Table~\ref{tab:empirical_spsizes} show that we will, on average, have comparable inference times to a canonical ViT due to the beneficial properties of our proposed superpixel tokenization.

\subsubsection{Final Thresholding}:
Adaptable tokenization frameworks does not necessarily entail an overall drop in inference throughput. 
Contrarily, it could potentially be leveraged to substantially improve inference speed by designing learnable methods to lower the number of tokens without decreasing performance, \eg ToMe by \maincitet{tome}.

We apply an additional final merging step where we compute the euclidean distance between adjacent superpixels and merge all superpixels below a given threshold for our SPiT-B16 model with gradient features. 
Noting that a threshold of 0.0 retains the original model design, 
the results in Table~\ref{tab:final_th} indicate that models with superpixel tokenization can be optimized to improve inference throughput. 
We also note that taking the maximum performing tokens over all thresholds achieves an accuracy of 0.817, significantly improving the predictive performance.

\begin{table*}[tb]
  \sisetup{detect-all,
  }
  \caption{Accuracy under final thresholding.}
  \label{tab:final_th}
  \scriptsize
  \centering
  \begin{tabular}{SSSS}
    \toprule
    \textbf{Threshold}&\textbf{No.Tok.}&\textbf{Im./s.}&\textbf{Avg.Acc.}\\
    \midrule
    0.00 &556.5 & 718.7 & 0.804 \\
    0.05 &513.2 & 749.2 & 0.804 \\
    0.10 &441.6 & 844.4 & 0.802 \\
    0.15 &365.0 & 950.5 & 0.797 \\
    0.20 &293.7 & 1038.9 & 0.786 \\
    \bottomrule
  \end{tabular}
\end{table*}

\section{Training Details}
\label{sec:training-details}
As mentioned in Section~\mainref{subsec:contributions}, we use standardized ViT architectures and generally follow the recommendations provided by \maincitet{howtrainvit}.
We provide training logs, pre-trained models, and code for training models from scratch in our GitHub project repository (in the camera ready manuscript).

\subsubsection{Classification:}
Training is performed over 300 epochs using the \textsc{AdamW} optimizer with a cosine annealing learning rate scheduler with 5 epochs of cosine annealed warmup from learning rate $\eta_\mathrm{start} = 1\times 10^{-5}$.
The schedule maxima and minima are given by $\eta_\mathrm{max} = 3\times 10^{-3}$, and $\eta_\mathrm{min} = 1\times 10^{-6}$.
We use a weight decay of $\lambda_\mathrm{dec}=2\times 10^{-2}$ and set the smoothing term $\epsilon = 1\times 10^{-7}$.
In addition, we used stochastic depth dropout with a base probability of $p=0.2$ and layer scaling.
Models were pre-trained with spatial resolution $256 \times 256$.

For augmentations, we randomly select between using the \textsc{RandAug} framework at medium strength or using \textsc{Aug3} framework by \maincitet{deit3} including \textsc{CutMix}~\maincitep{cutmix} with parameter $\alpha = 1.0$.
We use \textsc{RandomResizeCrop} using the standard scale $(0.08, 1.0)$ with randomly sampled interpolation modes.
Since the number of partitions from the superpixel tokenizer are adapted on an image-to-image basis, we effectively constrain the maximum number of tokens during training using token dropout to balance number of tokens.

We found that a naive on-line computation of Voronoi tessellations was unnecessarily computationally expensive, hence we precompute sets of random Voronoi tessellations with 196, 256, and 576 partitions, corresponding to images of $224\times224$, $256\times256$, and $384\times384$ resolutions given patch size $\rho=16$.

All training was performed on AMD MI250X GPUs.
One important distinction is that we do not use quantization with \verb|bfloat16| for training our models, instead opting for the higher 32-bit precision of \verb|float32| since this improves consistency between vendor frameworks.
Inference was carried out on a mixture of NVIDIA A100, RTX 2080Ti, Quadro P6000, and AMD MI250X to validate consistency across vendor frameworks.

\subsubsection{Fine Tuning:}
All base models were fine-tuned over 30 epochs with increased degrees of regularization.
We increase the level of \textsc{RandAug} to ``strong'' using 2 operations with magnitude 20.
Additionally, we increase the stochastic depth dropout to $p=0.4$.
Fine tuning was performed with spatial resolution $384 \times 384$, and we reduce the maximum learning rate to $\eta_{\max}=1\times 10^{-4}$.
For the alternative classification datasets \textsc{Cifar100} and \textsc{Caltech256}, fine tuning was performed by replacing the classification head and fine tuning for 10 epochs using \textsc{AdamW} with learning rate $\eta=1\times 10^{-4}$ and the same weight decay.
No augmentation was used in this process, and images were re-scaled to $256 \times 256$ for training and evaluation.

\begin{table*}[tb]
  \sisetup{detect-all,
    uncertainty-separator=\pm,
    table-format=1.4,
    uncertainty-mode=separate,
  }
  \caption{Accuracy (Top 1) for Small (S) capacity models on classification. Note that the Small capacity models have been trained without final finetuning.} 
  \label{tab:results-small}
  \scriptsize
  \centering
  \begin{tabular}{l@{ }@{}c@{ }@{}S@{ }@{}S@{ }@{}S@{ }@{}S@{ }@{}S@{ }@{}S@{ }@{}S@{ }@{}S}
    \toprule
    \multicolumn{2}{c}{Model} &
    \multicolumn{2}{c}{\textsc{INReaL}} &
    \multicolumn{2}{c}{\textsc{IN1k}} &
    \multicolumn{2}{c}{\textsc{Caltech256}} &
    \multicolumn{2}{c}{\textsc{Cifar100}} \\
    \cmidrule(r){1-2} 
    \cmidrule(r){3-4} 
    \cmidrule(r){5-6} 
    \cmidrule(r){7-8} 
    \cmidrule(r){9-10} 
    Name    & Grad.&  {Lin.} & {kNN} & {Lin.} & {kNN}  & {Lin.} & {kNN}  & {Lin.} & {kNN}\\
    \midrule
    ViT-S16 &\xmark&    .778&    .808&    .765&    .692&    .818&    .827 &    .827&    .833\\
    ViT-S16 &\cmark&    .782&    .811&    .754&    .682&    .824&    .832 &    .830&    .836\\
    RViT-S16&\xmark&\bfs.829&\bfs.814&\bfs.767&    .740&    .852&    .858 &    .856&    .858\\
    RViT-S16&\cmark&    .818&    .812&    .759&\bfs.741&\bfs.856&\bfs.861 &\bfs.856&\bfs.859\\
    SPiT-S16&\xmark&    .746&    .796&    .689&    .628&    .767&    .771 &    .761&    .769\\
    SPiT-S16&\cmark&    .819&    .812&    .750&    .736&    .849&    .851 &    .832&    .839\\
    \bottomrule
    \multicolumn{10}{l}{\tiny \tdgg Uncertainty measures for RViT tokenizer are detailed in Appendix Table~\ref{tab:voronoi_uncertainty}.}
  \end{tabular}
\end{table*}

\section{Interpretability and Attention Maps}
\label{sec:attmaps}

For LIME explanations, we train a linear surrogate model $L_\Phi$ for predicting the output probabilities for the prediction of each model $\Phi$.
To encourage independence between tokenizers and LIME explanations, as well as promote direct comparability, we use SLIC with a target of $\lvert \pi \rvert \approx 64$ superpixels.
We use Monte Carlo sampling of binary features for indicating the presence or omission of each superpixel with stochastic $p \in \mathrm{Uniform}(0.1, 0.3)$, and keep these consistent across model evaluations.
We observed that certain images in the \textsc{IN1k} at times produced less than $5$ superpixels using SLIC, hence these images were dropped from the evaluation.

The attention flow~\maincitep{attrolloutflow} of a transformer differs from the standard attention roll-out by accounting for the contributions of the residual connections in computations.
The attention flow of an $\ell$-layer transformer is given by
\begin{align}
    A_\mathrm{Flow} = \prod_{i = 1}^\ell \big( (1-\lambda) I + \lambda A_i \big).
\end{align}
where we set $\lambda = 0.9$ to account for stochastic depth and layer scaling factors while accentuating the contribution of the attention operators.
We use max-aggregation over the heads to extract a unified representation.
Following \maincitet{origvit} and \maincitet{dino}, we extract the attention for the class token as an interpretation of the model's prediction.

For the PCA projection, we take inspiration from the visualizations technique used in the work of \maincitet{dinov2}.
In this work, the features of multiple images with comparable attributes are concatenated, and projected onto a set of the top principal components of the image.
We compute a set of 5 prototype centroids $\nu \in \mathbb{R}^{1000 \times d \times 5}$ for each class token of each model over ImageNet using KMeans, while enforcing relative subclass orthogonality by introducing a regularization term
\begin{align}
    J(\nu) = \frac{\lambda_{\nu}}{1000} \sum_{c=1}^{1000} \lVert I - \nu_c^\intercal \nu_c \lVert_2^2,
\end{align}
selecting $\lambda_\nu = 0.1$.
Given a prediction $c$, we concatenate the prototypes to the token embeddings to form a matrix $M = [\Phi(\xi; \theta)^\intercal, \nu_c^\intercal]^\intercal$.
Letting $U\Sigma V^\intercal = M - \mu(M)$ be a low-rank SVD of the centered features, we then project the original features to the principal components by $\Phi(\xi; \theta) V$, and use max-aggregation to extract the attribution as an interpretation of the model's prediction.
We experimented with different ranks, but found that simply using the first principal component aligned well with attention maps and LIME coefficients.
This somewhat mirrors the procedure by \maincitet{dinov2}, where a thresholded projection on the first principal component is applied as a mask.
In the interest of reproducibility, we provide links for downloading normalized attention maps for all attributions in our GitHub repository.

\setlength{\fsz}{0.2\linewidth}
\begin{figure}[tb]
\centering
\footnotesize
\begin{tblr}{
  colspec={Q[c,m]Q[c,m]Q[c,m]Q[c,m]Q[c,m]},
  rowsep=2pt,
  colsep=2pt,
  row{1}={font=\scriptsize},
  column{1}={font=\scriptsize},
  }
& { Token.\ Image} & \textsc{Att.\ Flow} & \textsc{Proto.\ PCA} &  LIME (SLIC) \\
 ViT &
\adjustbox{valign=m}{\includegraphics[width=\fsz]{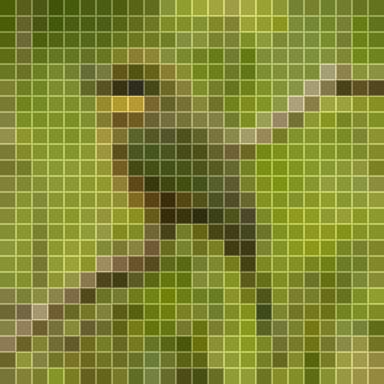}} &
\adjustbox{valign=m}{\includegraphics[width=\fsz]{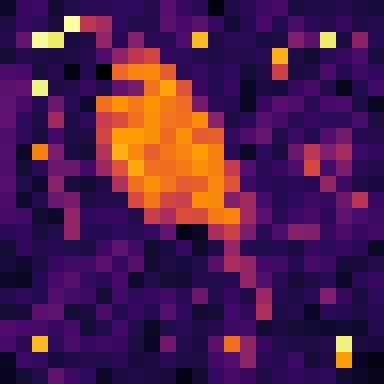}} &
\adjustbox{valign=m}{\includegraphics[width=\fsz]{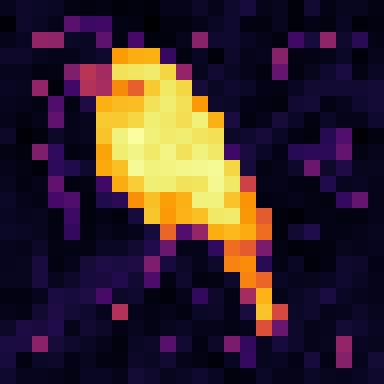}} &
\adjustbox{valign=m}{\includegraphics[width=\fsz]{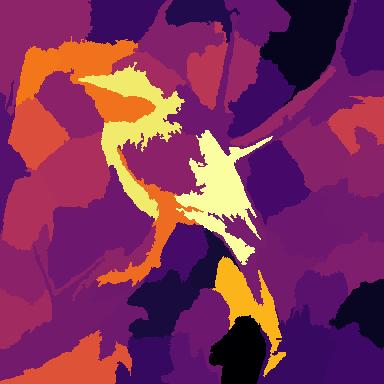}} \\
 RViT &
\adjustbox{valign=m}{\includegraphics[width=\fsz]{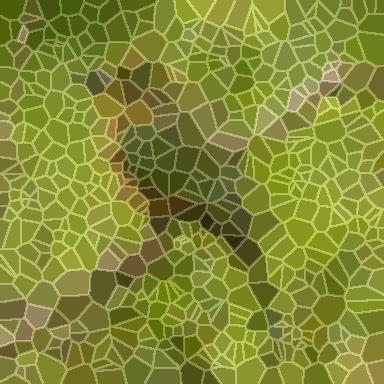}} &
\adjustbox{valign=m}{\includegraphics[width=\fsz]{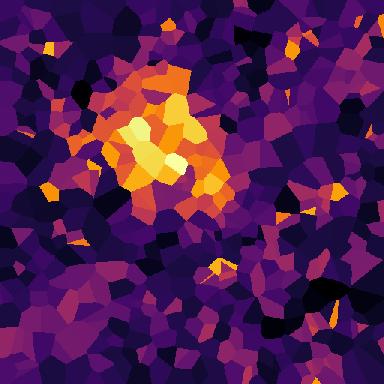}} &
\adjustbox{valign=m}{\includegraphics[width=\fsz]{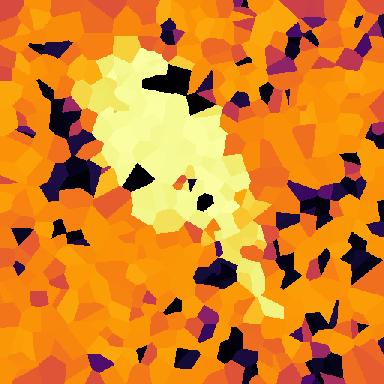}} &
\adjustbox{valign=m}{\includegraphics[width=\fsz]{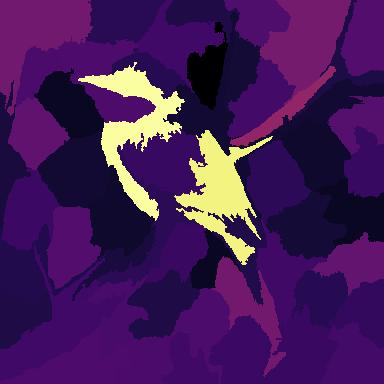}} \\
 SPiT &
\adjustbox{valign=m}{\includegraphics[width=\fsz]{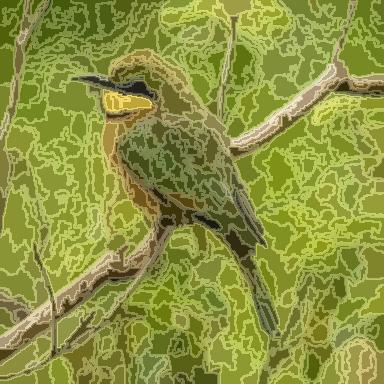}} &
\adjustbox{valign=m}{\includegraphics[width=\fsz]{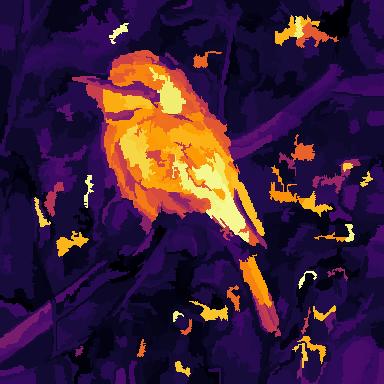}} &
\adjustbox{valign=m}{\includegraphics[width=\fsz]{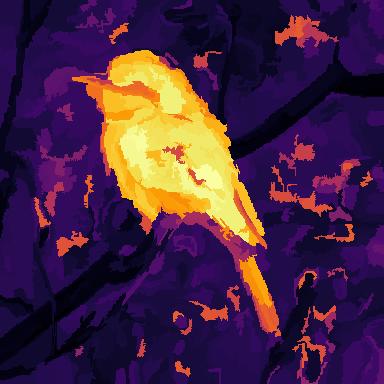}} &
\adjustbox{valign=m}{\includegraphics[width=\fsz]{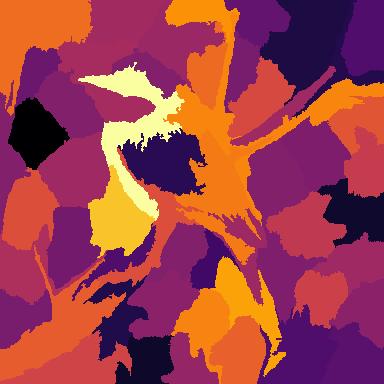}} \\
%%%
 ViT &
\adjustbox{valign=m}{\includegraphics[width=\fsz]{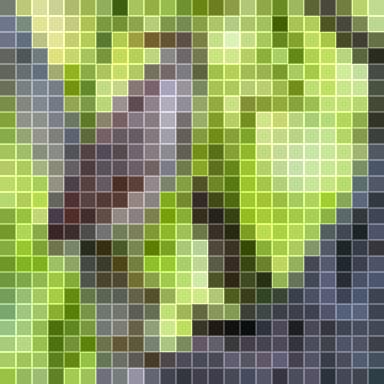}} &
\adjustbox{valign=m}{\includegraphics[width=\fsz]{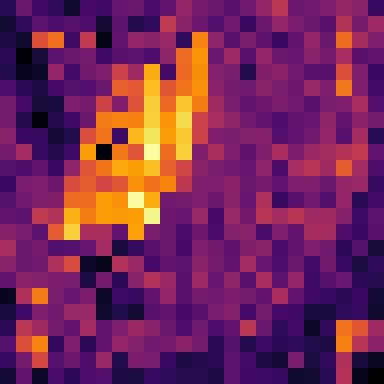}} &
\adjustbox{valign=m}{\includegraphics[width=\fsz]{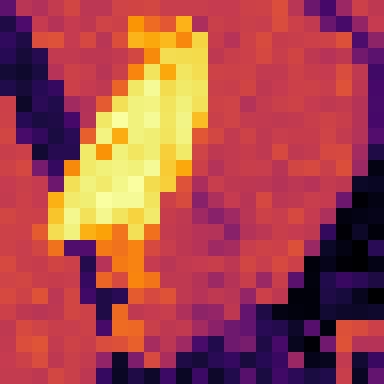}} &
\adjustbox{valign=m}{\includegraphics[width=\fsz]{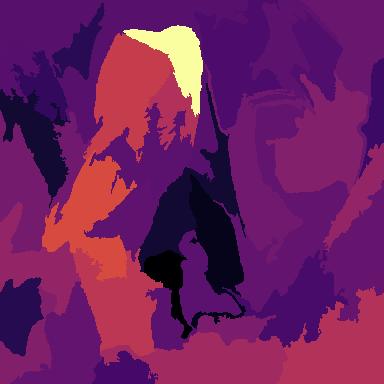}} \\
 RViT &
\adjustbox{valign=m}{\includegraphics[width=\fsz]{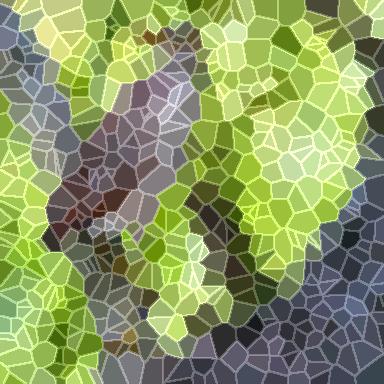}} &
\adjustbox{valign=m}{\includegraphics[width=\fsz]{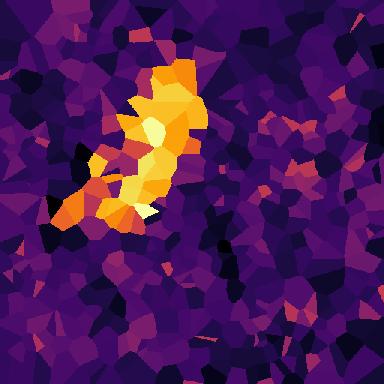}} &
\adjustbox{valign=m}{\includegraphics[width=\fsz]{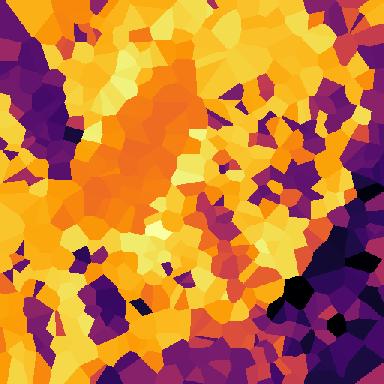}} &
\adjustbox{valign=m}{\includegraphics[width=\fsz]{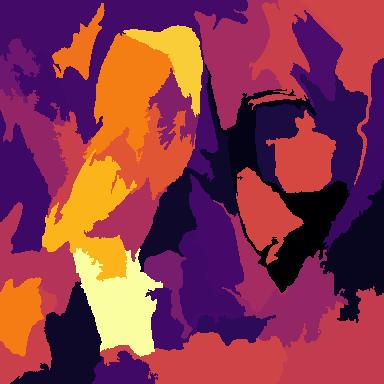}} \\
 SPiT &
\adjustbox{valign=m}{\includegraphics[width=\fsz]{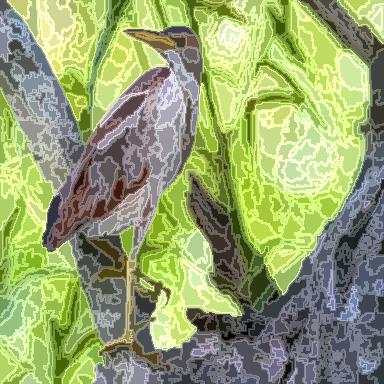}} &
\adjustbox{valign=m}{\includegraphics[width=\fsz]{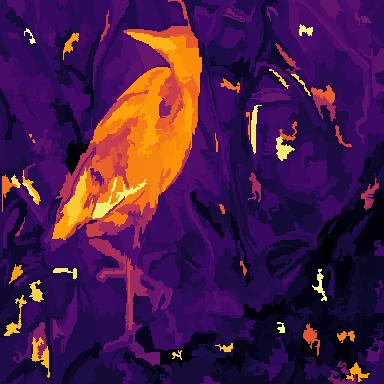}} &
\adjustbox{valign=m}{\includegraphics[width=\fsz]{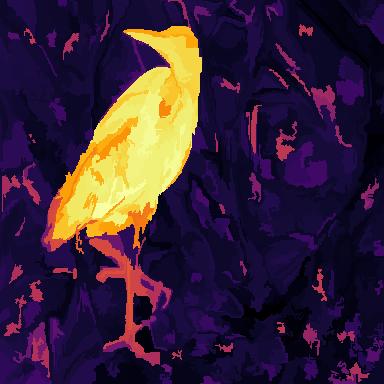}} &
\adjustbox{valign=m}{\includegraphics[width=\fsz]{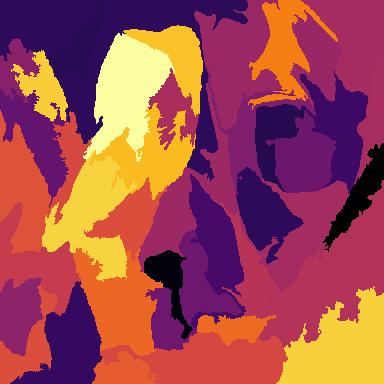}} \\
\end{tblr}
\caption{
Visualization of feature attributions for prediction ``\textit{bee eater}'' and ``\textit{bittern}'' with different tokenization strategies: square partitions (ViT), random Voronoi tesselation (RViT) and superpixels (SPiT).
% img1: ILSVRC2012_val_00016145
% img2: ILSVRC2012_val_00000278
% img3: ILSVRC2012_val_00000390
}
\label{fig:attention-maps}
\end{figure}

\setlength{\fsz}{0.2\linewidth}
\begin{figure}[tb]
\centering
\footnotesize
\begin{tblr}{
  colspec={Q[c,m]Q[c,m]Q[c,m]Q[c,m]Q[c,m]},
  rowsep=2pt,
  colsep=2pt,
  row{1}={font=\scriptsize},
  column{1}={font=\scriptsize},
  }
& { Token.\ Image} & \textsc{Att.\ Flow} & \textsc{Proto.\ PCA} &  LIME (SLIC) \\
 ViT &
\adjustbox{valign=m}{\includegraphics[width=\fsz]{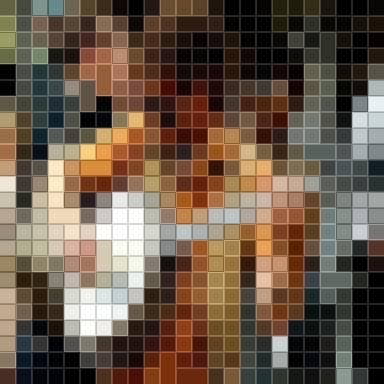}} &
\adjustbox{valign=m}{\includegraphics[width=\fsz]{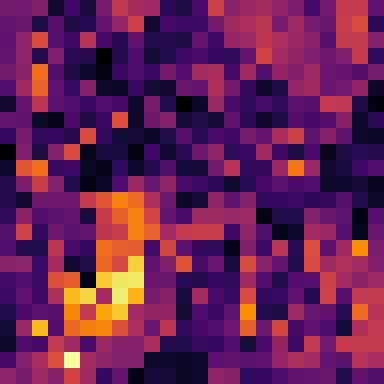}} &
\adjustbox{valign=m}{\includegraphics[width=\fsz]{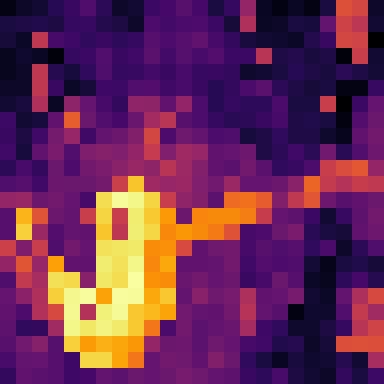}} &
\adjustbox{valign=m}{\includegraphics[width=\fsz]{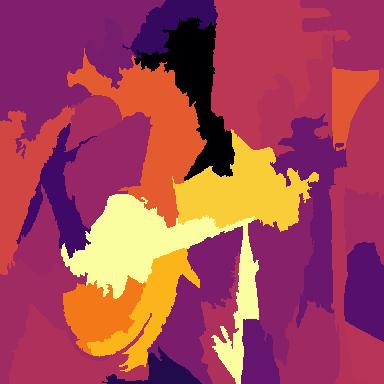}} \\
 RViT &
\adjustbox{valign=m}{\includegraphics[width=\fsz]{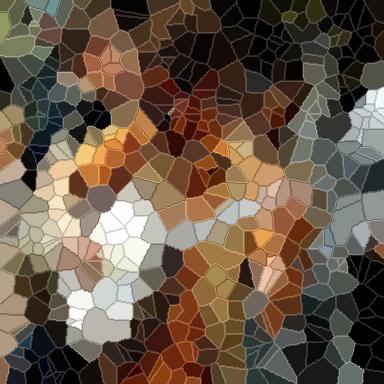}} &
\adjustbox{valign=m}{\includegraphics[width=\fsz]{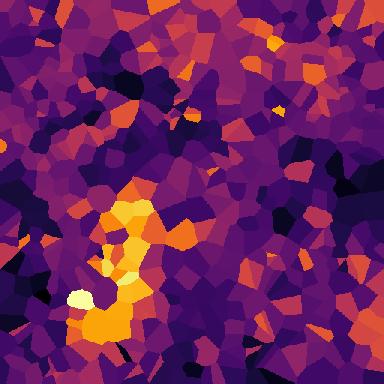}} &
\adjustbox{valign=m}{\includegraphics[width=\fsz]{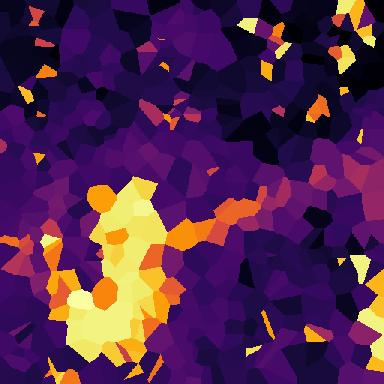}} &
\adjustbox{valign=m}{\includegraphics[width=\fsz]{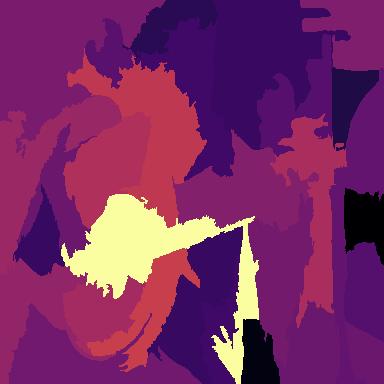}} \\
 SPiT &
\adjustbox{valign=m}{\includegraphics[width=\fsz]{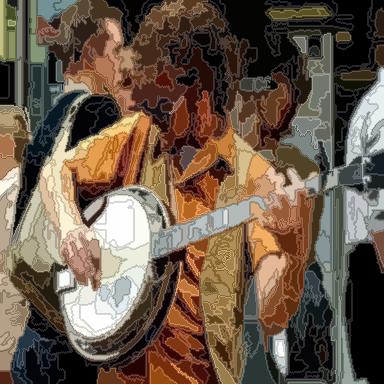}} &
\adjustbox{valign=m}{\includegraphics[width=\fsz]{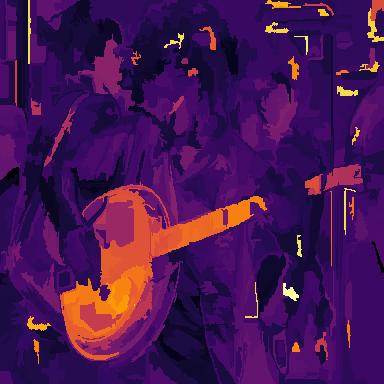}} &
\adjustbox{valign=m}{\includegraphics[width=\fsz]{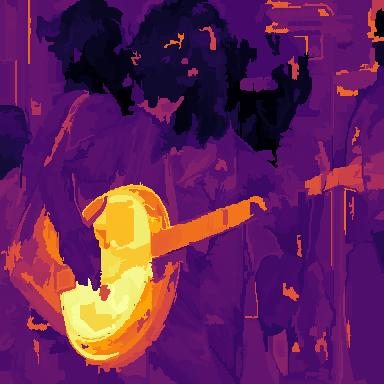}} &
\adjustbox{valign=m}{\includegraphics[width=\fsz]{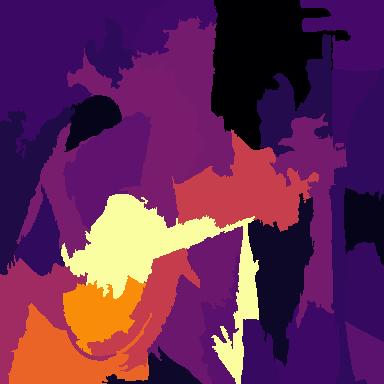}} \\
%%%
 ViT &
\adjustbox{valign=m}{\includegraphics[width=\fsz]{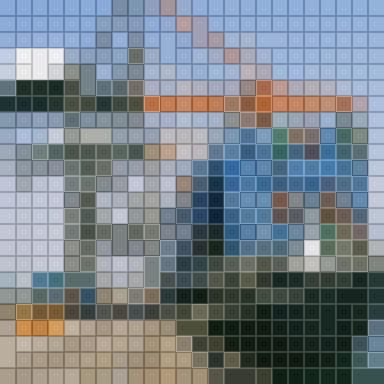}} &
\adjustbox{valign=m}{\includegraphics[width=\fsz]{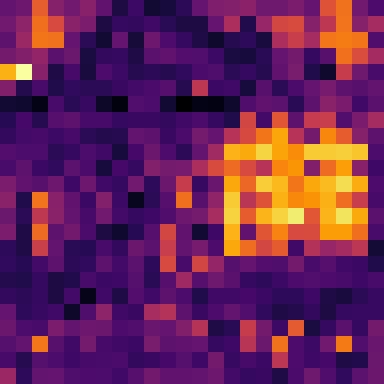}} &
\adjustbox{valign=m}{\includegraphics[width=\fsz]{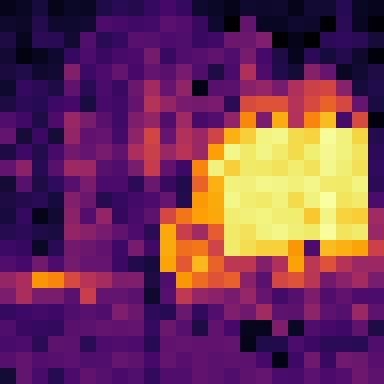}} &
\adjustbox{valign=m}{\includegraphics[width=\fsz]{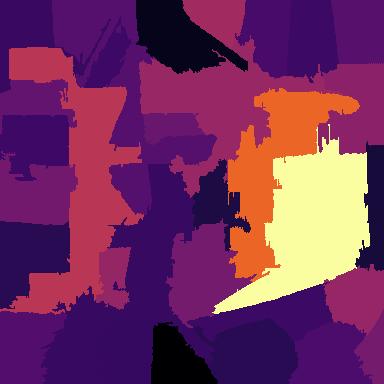}} \\
 RViT &
\adjustbox{valign=m}{\includegraphics[width=\fsz]{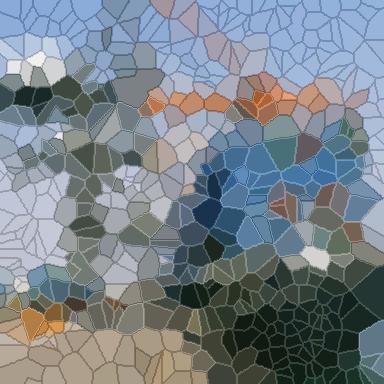}} &
\adjustbox{valign=m}{\includegraphics[width=\fsz]{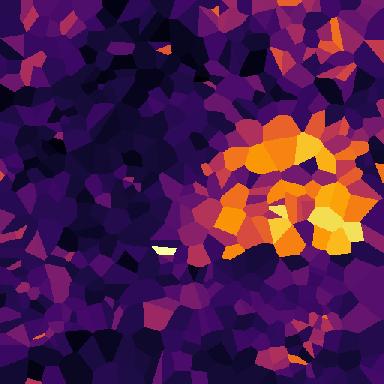}} &
\adjustbox{valign=m}{\includegraphics[width=\fsz]{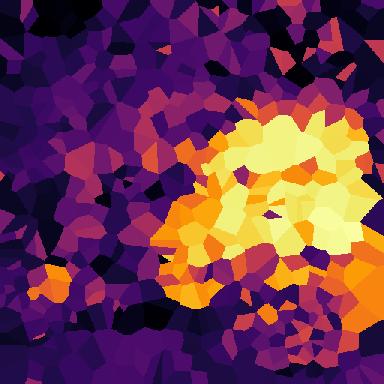}} &
\adjustbox{valign=m}{\includegraphics[width=\fsz]{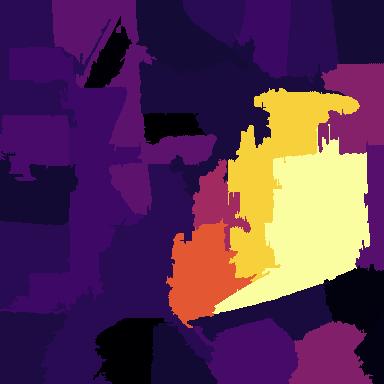}} \\
 SPiT &
\adjustbox{valign=m}{\includegraphics[width=\fsz]{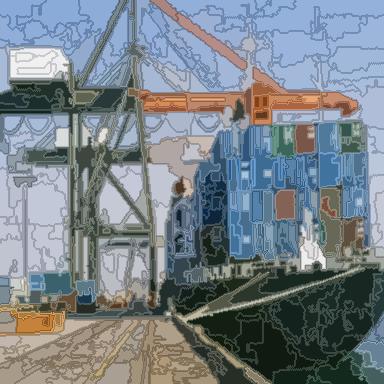}} &
\adjustbox{valign=m}{\includegraphics[width=\fsz]{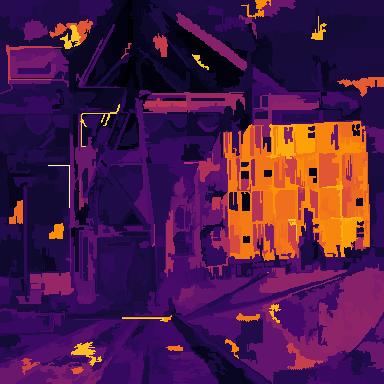}} &
\adjustbox{valign=m}{\includegraphics[width=\fsz]{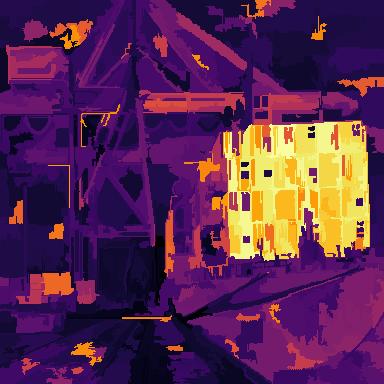}} &
\adjustbox{valign=m}{\includegraphics[width=\fsz]{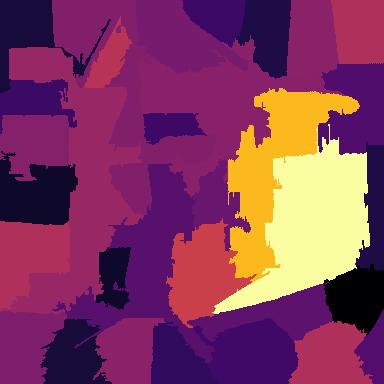}} \\
\end{tblr}
\caption{
Visualization of feature attributions for prediction ``\textit{banjo}'' and ``\textit{container ship}'' with different tokenization strategies: square partitions (ViT), random Voronoi tesselation (RViT) and superpixels (SPiT).
% img1: ILSVRC2012_val_00000388
% img2: ILSVRC2012_val_00000403
}
\label{fig:attention-maps-2}
\end{figure}

\setlength{\fsz}{0.2\linewidth}
\begin{figure}[tb]
\centering
\footnotesize
\begin{tblr}{
  colspec={%Q[c,m]
    Q[c,m]Q[c,m]Q[c,m]Q[c,m]},
  rowsep=1.5pt,
  colsep=1.5pt,
  hlines = {1.5pt, white},
  vlines = {1.5pt, white},
  row{1}={font=\scriptsize},
  column{1}={font=\scriptsize},
  cell{2,5}{4}={bg=WildStrawberry},
  cell{3-5}{3}={bg=WildStrawberry},  
}
% &
{ Token.\ Image} & \textsc{Att.\ Flow} & \textsc{Proto.\ PCA} &  LIME (SLIC) \\
 % SPiT &
\adjustbox{valign=m}{\includegraphics[width=\fsz]{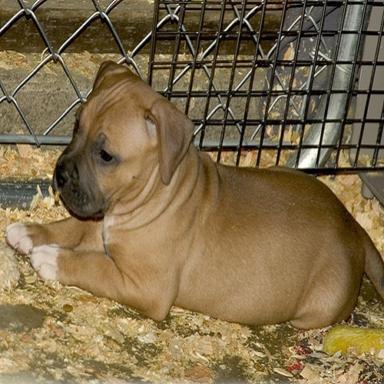}} &
\adjustbox{valign=m}{\includegraphics[width=\fsz]{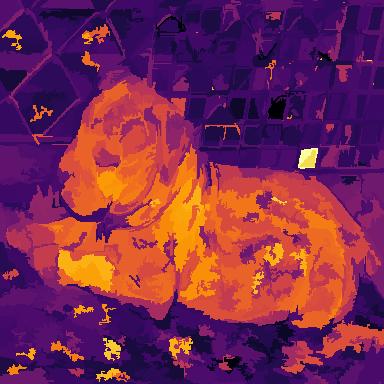}} &
\adjustbox{valign=m}{\includegraphics[width=\fsz]{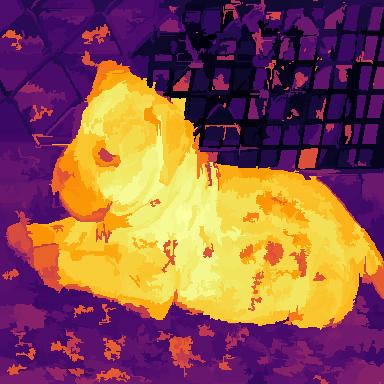}} &
\adjustbox{valign=m}{\includegraphics[width=\fsz]{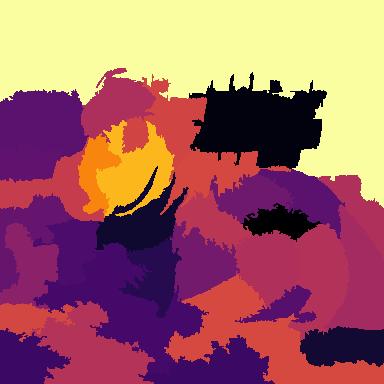}} \\
 % SPiT &
\adjustbox{valign=m}{\includegraphics[width=\fsz]{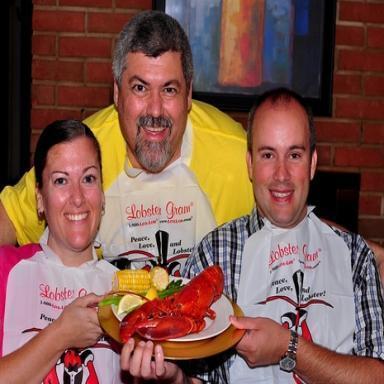}} &
\adjustbox{valign=m}{\includegraphics[width=\fsz]{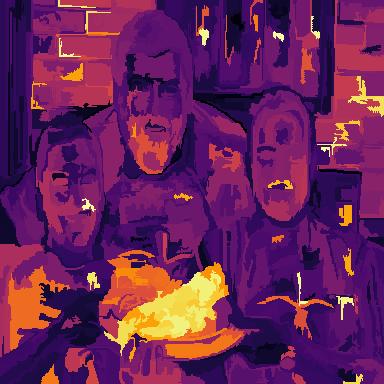}} &
\adjustbox{valign=m}{\includegraphics[width=\fsz]{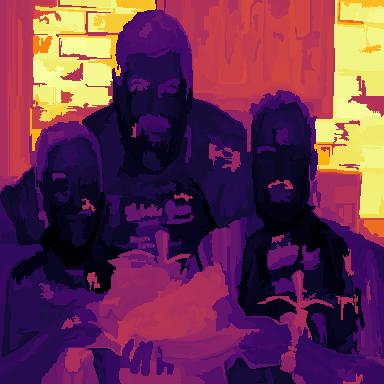}} &
\adjustbox{valign=m}{\includegraphics[width=\fsz]{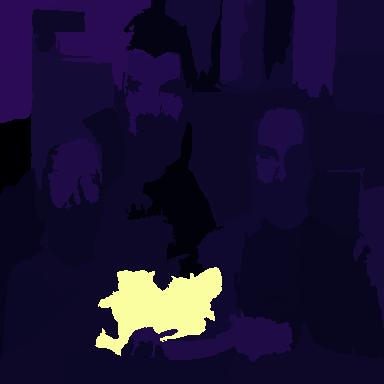}} \\
 % SPiT &
\adjustbox{valign=m}{\includegraphics[width=\fsz]{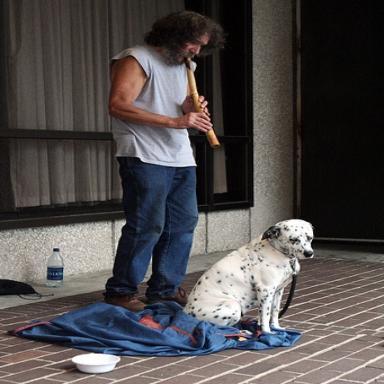}} &
\adjustbox{valign=m}{\includegraphics[width=\fsz]{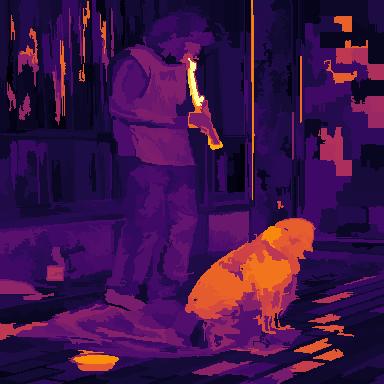}} &
\adjustbox{valign=m}{\includegraphics[width=\fsz]{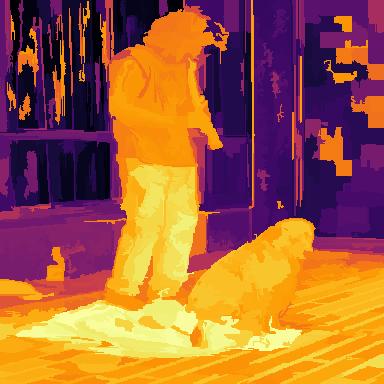}} &
\adjustbox{valign=m}{\includegraphics[width=\fsz]{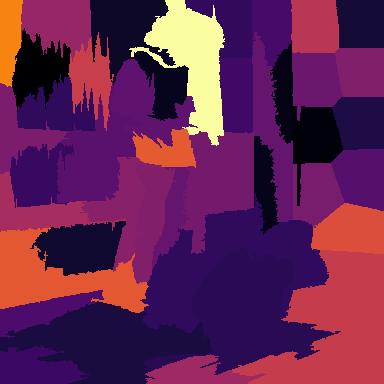}} \\
 % SPiT &
\adjustbox{valign=m}{\includegraphics[width=\fsz]{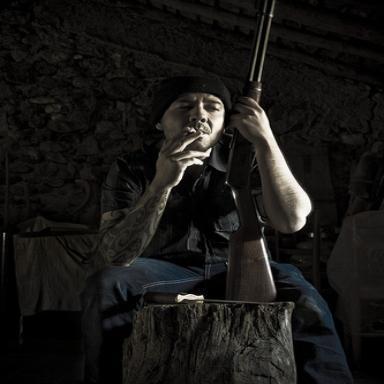}} &
\adjustbox{valign=m}{\includegraphics[width=\fsz]{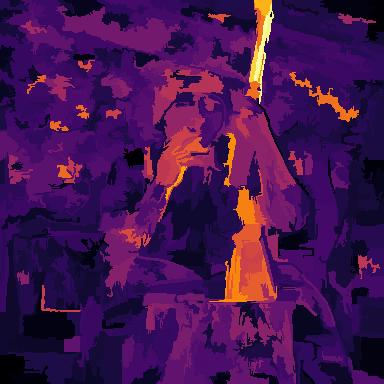}} &
\adjustbox{valign=m}{\includegraphics[width=\fsz]{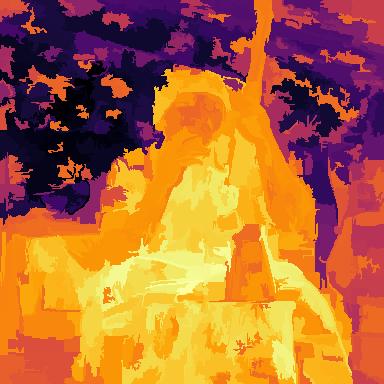}} &
\adjustbox{valign=m}{\includegraphics[width=\fsz]{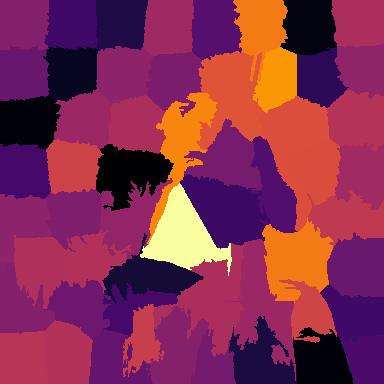}} \\
\end{tblr}
\caption{
Examples of edge cases for attribution maps with SPiT. Row 1 demonstrates a case where LIME fails to provide coherent attributions for the prediction ``\emph{Staffordshire terrier}'', and row 2--3 shows cases where the PCA prototype attributions fail for predictions ``\emph{lobster}'' and ``\emph{flute}'', respectively.
Row 4 shows a case where attributions for both LIME and PCA prototypes are inadequate for the predicted label ``\emph{rifle}''.
% img1: ILSVRC2012_val_00000191
% img2: ILSVRC2012_val_00000111
% img3: ILSVRC2012_val_00000140
% img4: ILSVRC2012_val_00000244
}
\label{fig:attention-maps-cont}
\end{figure}

\setlength{\fsz}{0.19\linewidth}
\begin{figure}[tb]
\centering
\footnotesize
\begin{tblr}{
  colspec={Q[c,m]Q[c,m]Q[c,m]Q[c,m]Q[c,m]},
  rowsep=1.5pt,
  colsep=1.5pt,
  column{1} = {font=\scriptsize},
  row{1} = {font=\scriptsize},
}
& {\bfs $q = 0.05$} & \bfs $q=0.10$ & \bfs $q=0.20$ & \bfs $q=0.50$ \\
ViT  / \textsc{Comp} &
\adjustbox{valign=m}{\includegraphics[width=\fsz]{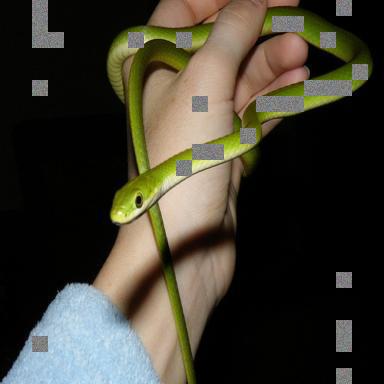}} &
\adjustbox{valign=m}{\includegraphics[width=\fsz]{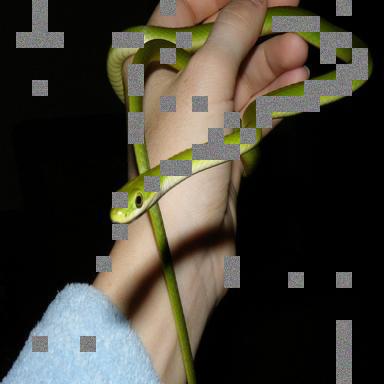}} &
\adjustbox{valign=m}{\includegraphics[width=\fsz]{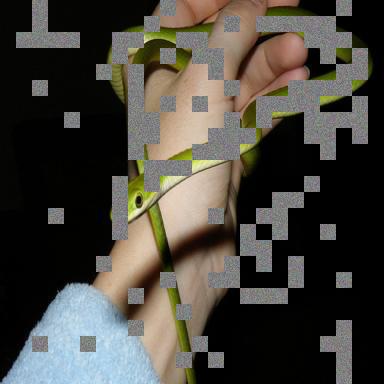}} &
\adjustbox{valign=m}{\includegraphics[width=\fsz]{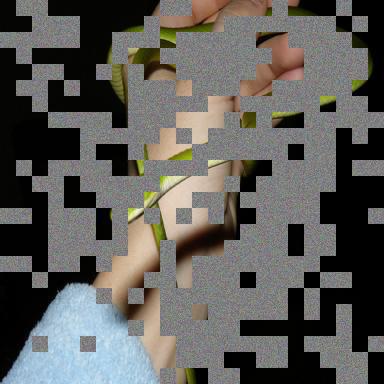}} \\
ViT  / \textsc{Suff} &
\adjustbox{valign=m}{\includegraphics[width=\fsz]{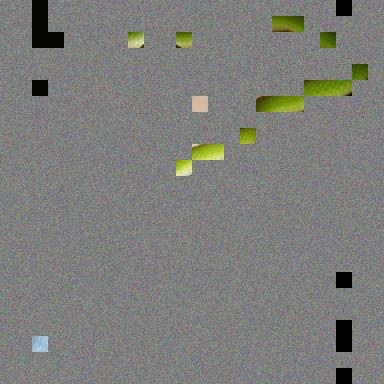}} &
\adjustbox{valign=m}{\includegraphics[width=\fsz]{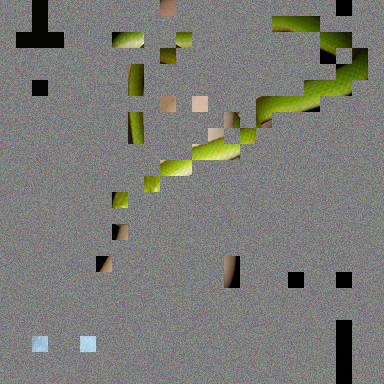}} &
\adjustbox{valign=m}{\includegraphics[width=\fsz]{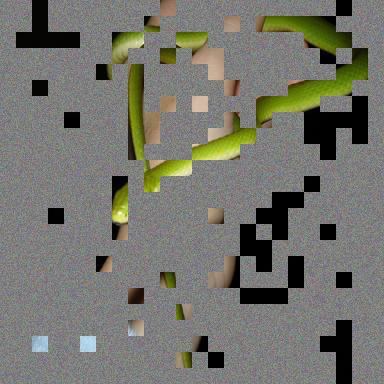}} &
\adjustbox{valign=m}{\includegraphics[width=\fsz]{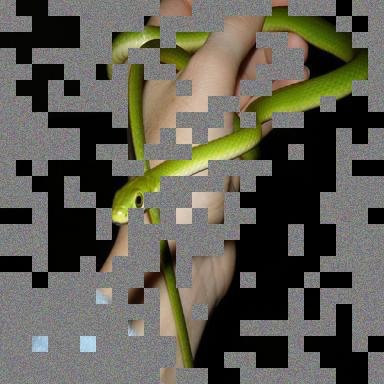}} \\
RViT  / \textsc{Comp} &
\adjustbox{valign=m}{\includegraphics[width=\fsz]{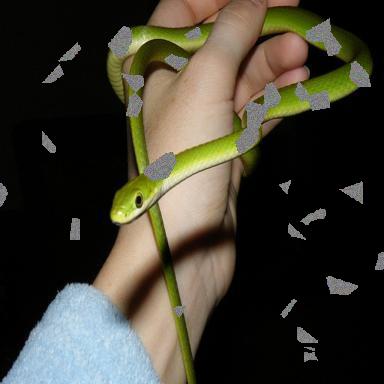}} &
\adjustbox{valign=m}{\includegraphics[width=\fsz]{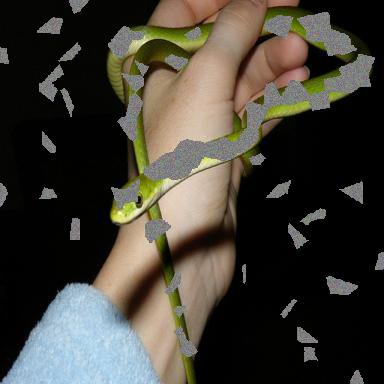}} &
\adjustbox{valign=m}{\includegraphics[width=\fsz]{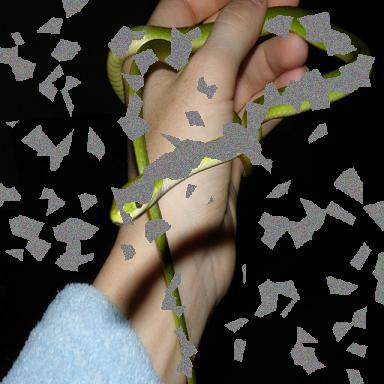}} &
\adjustbox{valign=m}{\includegraphics[width=\fsz]{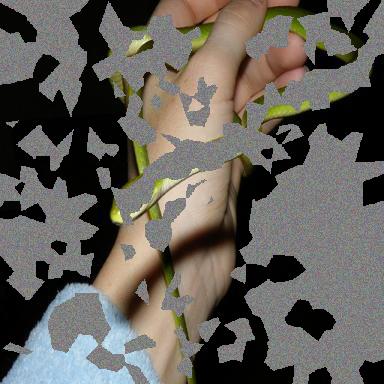}} \\
RViT  / \textsc{Suff} &
\adjustbox{valign=m}{\includegraphics[width=\fsz]{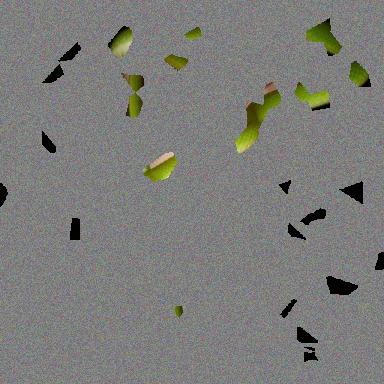}} &
\adjustbox{valign=m}{\includegraphics[width=\fsz]{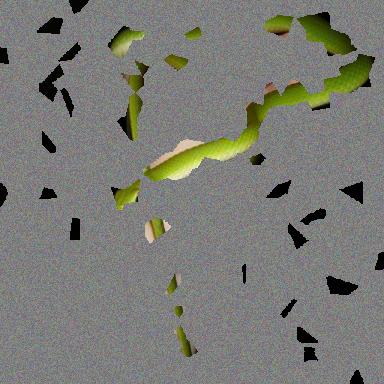}} &
\adjustbox{valign=m}{\includegraphics[width=\fsz]{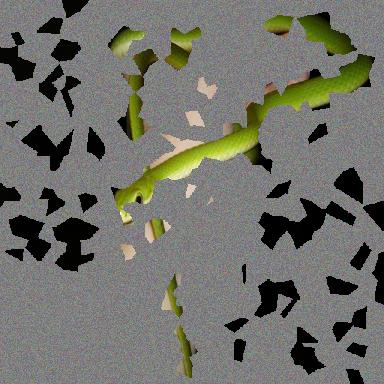}} &
\adjustbox{valign=m}{\includegraphics[width=\fsz]{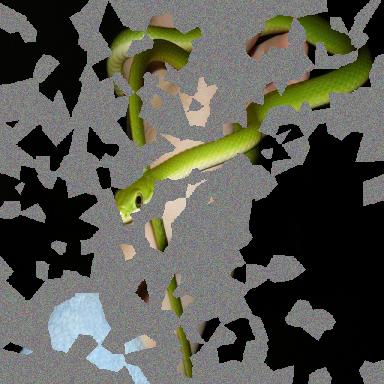}} \\
SPiT  / \textsc{Comp} &
\adjustbox{valign=m}{\includegraphics[width=\fsz]{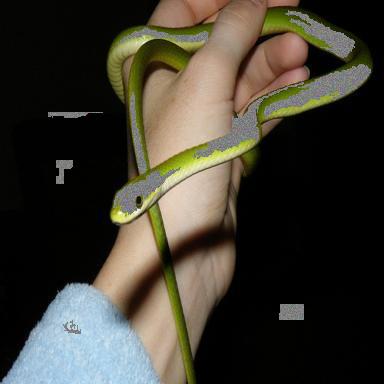}} &
\adjustbox{valign=m}{\includegraphics[width=\fsz]{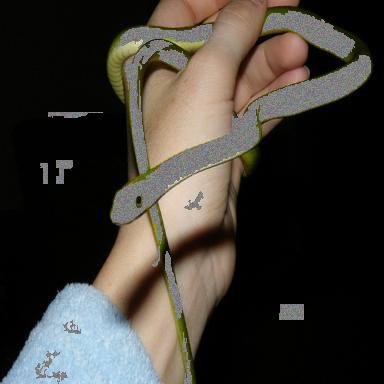}} &
\adjustbox{valign=m}{\includegraphics[width=\fsz]{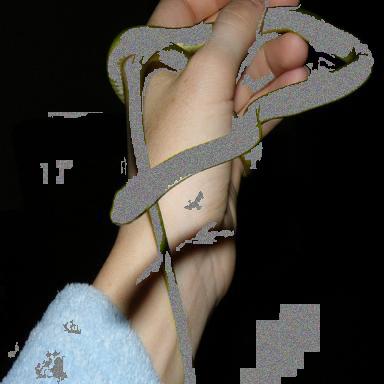}} &
\adjustbox{valign=m}{\includegraphics[width=\fsz]{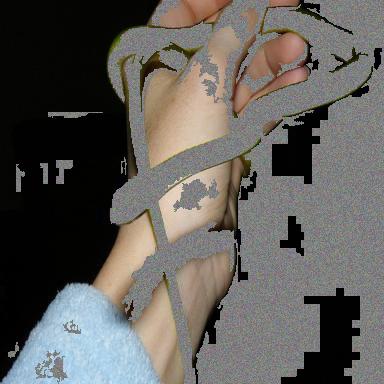}} \\
SPiT / \textsc{Suff} &
\adjustbox{valign=m}{\includegraphics[width=\fsz]{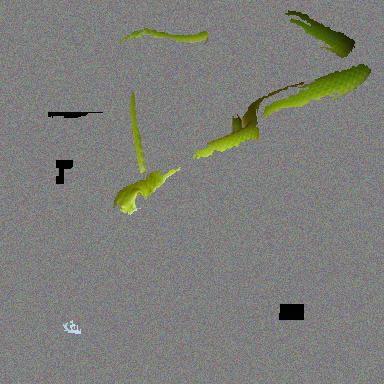}} &
\adjustbox{valign=m}{\includegraphics[width=\fsz]{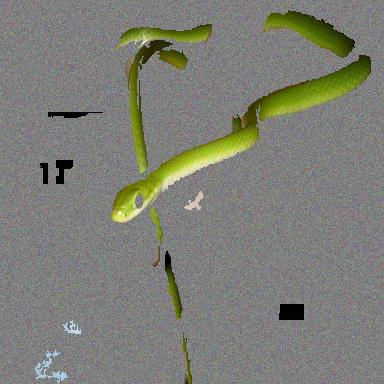}} &
\adjustbox{valign=m}{\includegraphics[width=\fsz]{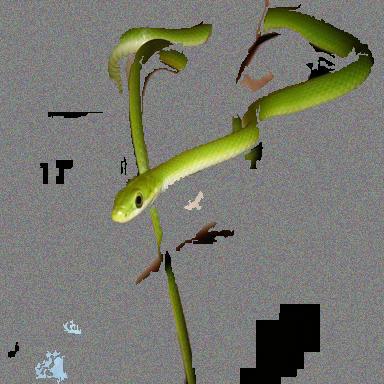}} &
\adjustbox{valign=m}{\includegraphics[width=\fsz]{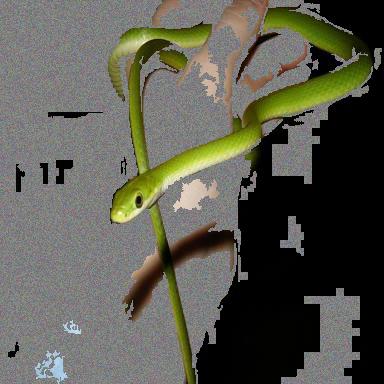}} \\
\end{tblr}
\caption{
Visualization of attention flow occlusions at different quantiles $q$ for prediction ``\textit{grass snake}''.
Note how the scaling of attention maps under superpixel tokenization improves occlusion for the predicted class.
% img1: ILSVRC2012_val_00016145
}
\label{fig:occlusion}
\end{figure}

\setlength{\fsz}{0.19\linewidth}
\begin{figure}[tb]
\centering
\footnotesize
\begin{tblr}{
  colspec={Q[c,m]Q[c,m]Q[c,m]Q[c,m]Q[c,m]},
  rowsep=1.5pt,
  colsep=1.5pt,
  column{1} = {font=\scriptsize},
  row{1} = {font=\scriptsize},
}
& {\bfs $q = 0.05$} & \bfs $q=0.10$ & \bfs $q=0.20$ & \bfs $q=0.50$ \\
ViT  / \textsc{Comp} &
\adjustbox{valign=m}{\includegraphics[width=\fsz]{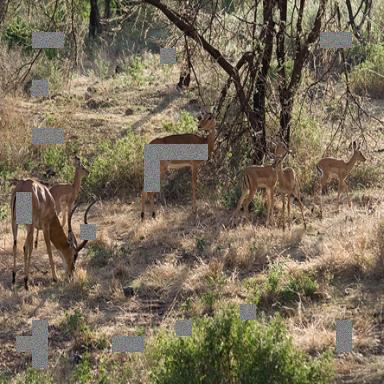}} &
\adjustbox{valign=m}{\includegraphics[width=\fsz]{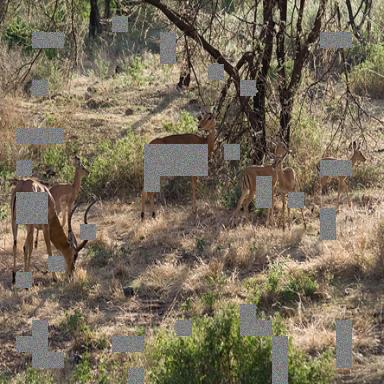}} &
\adjustbox{valign=m}{\includegraphics[width=\fsz]{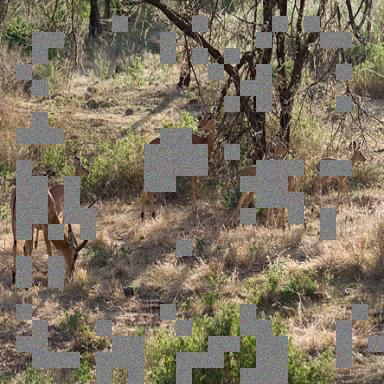}} &
\adjustbox{valign=m}{\includegraphics[width=\fsz]{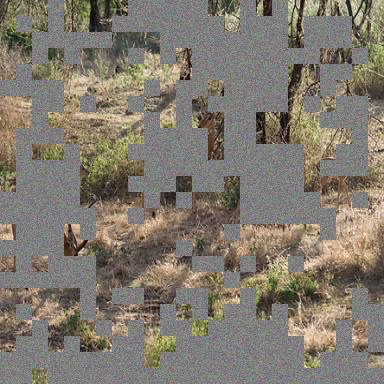}} \\
ViT  / \textsc{Suff} &
\adjustbox{valign=m}{\includegraphics[width=\fsz]{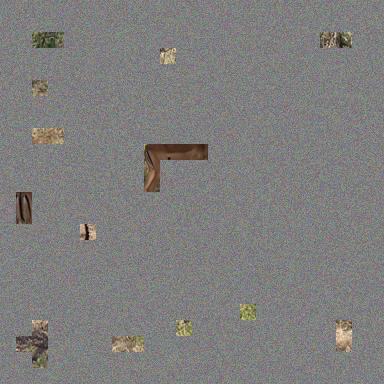}} &
\adjustbox{valign=m}{\includegraphics[width=\fsz]{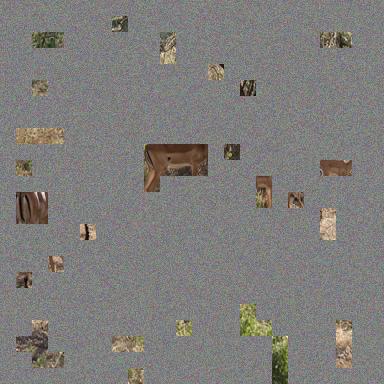}} &
\adjustbox{valign=m}{\includegraphics[width=\fsz]{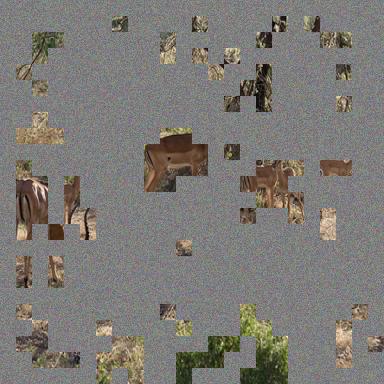}} &
\adjustbox{valign=m}{\includegraphics[width=\fsz]{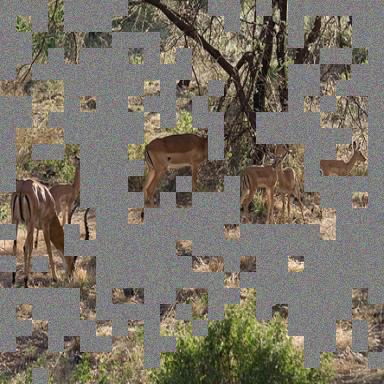}} \\
RViT  / \textsc{Comp} &
\adjustbox{valign=m}{\includegraphics[width=\fsz]{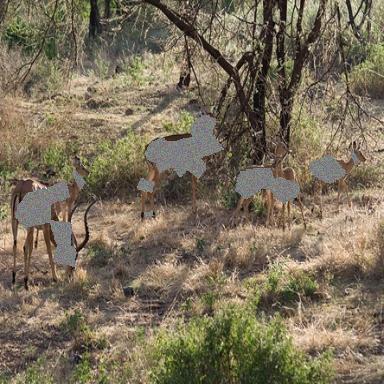}} &
\adjustbox{valign=m}{\includegraphics[width=\fsz]{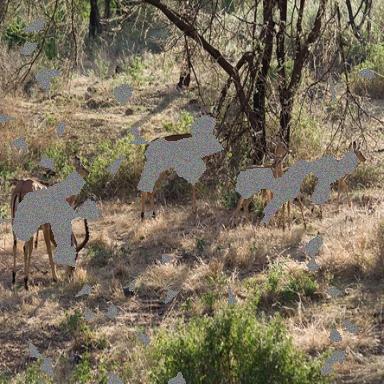}} &
\adjustbox{valign=m}{\includegraphics[width=\fsz]{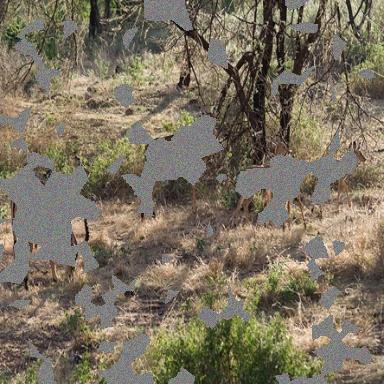}} &
\adjustbox{valign=m}{\includegraphics[width=\fsz]{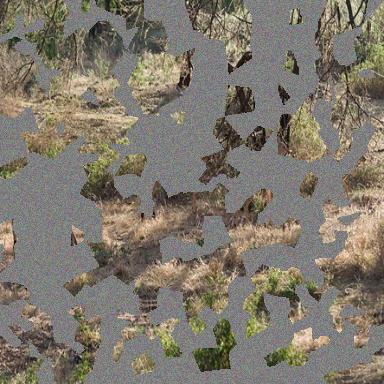}} \\
RViT  / \textsc{Suff} &
\adjustbox{valign=m}{\includegraphics[width=\fsz]{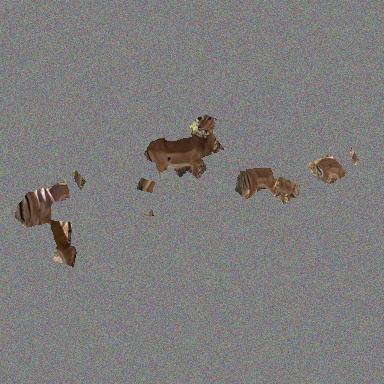}} &
\adjustbox{valign=m}{\includegraphics[width=\fsz]{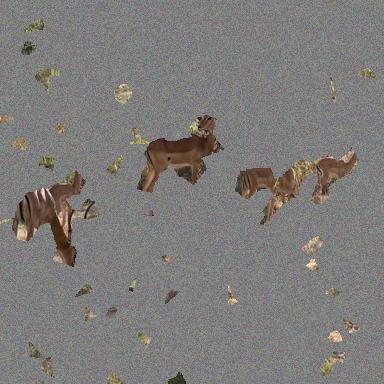}} &
\adjustbox{valign=m}{\includegraphics[width=\fsz]{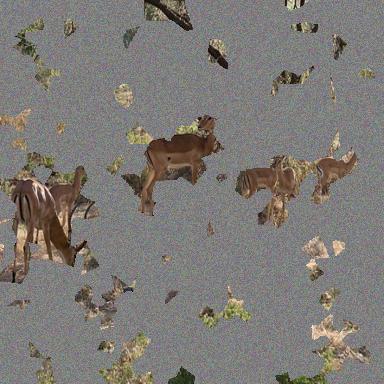}} &
\adjustbox{valign=m}{\includegraphics[width=\fsz]{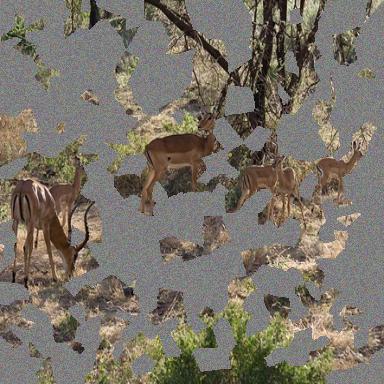}} \\
SPiT  / \textsc{Comp} &
\adjustbox{valign=m}{\includegraphics[width=\fsz]{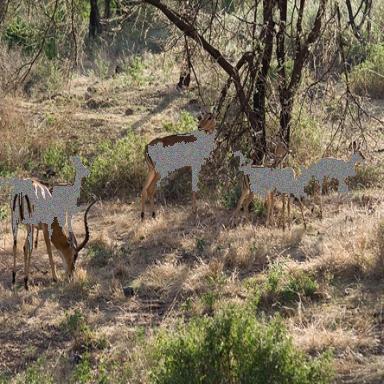}} &
\adjustbox{valign=m}{\includegraphics[width=\fsz]{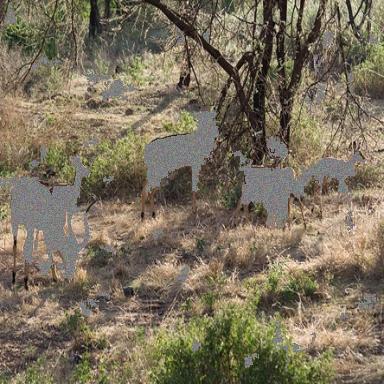}} &
\adjustbox{valign=m}{\includegraphics[width=\fsz]{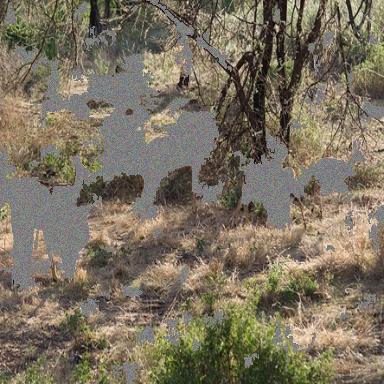}} &
\adjustbox{valign=m}{\includegraphics[width=\fsz]{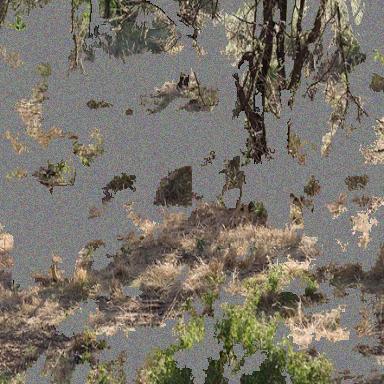}} \\
SPiT / \textsc{Suff} &
\adjustbox{valign=m}{\includegraphics[width=\fsz]{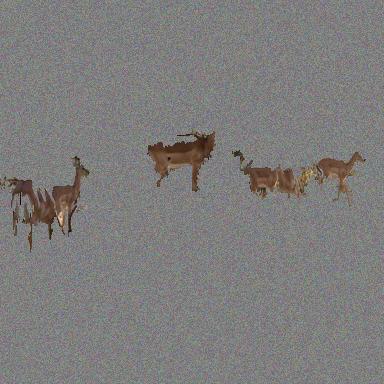}} &
\adjustbox{valign=m}{\includegraphics[width=\fsz]{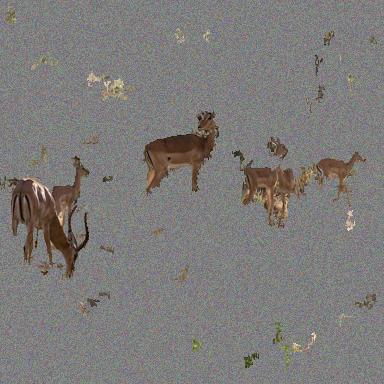}} &
\adjustbox{valign=m}{\includegraphics[width=\fsz]{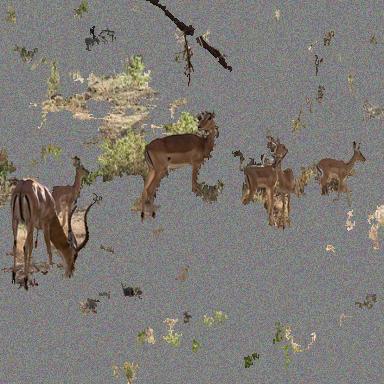}} &
\adjustbox{valign=m}{\includegraphics[width=\fsz]{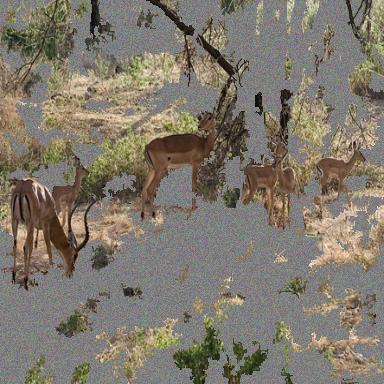}} \end{tblr}
\caption{
Visualization of attention flow occlusions at different quantiles $q$ for prediction ``\textit{impala}''.
Note how the scaling of attention maps under superpixel tokenization improves occlusion for the predicted class.
% img1: ILSVRC2012_val_00000127
}
\label{fig:occlusion2}
\end{figure}

To quantify the faithfulness of the attributions for each model, we used comprehensiveness and sufficiency as proposed by \maincitet{erasercompsuff}.
Given a sequence of quantiles $Q \in [0, 1]$ from an attribution, these metrics are given by
\begin{align}
    \textsc{Comp}_{Q \mid x, \Phi} = \frac{1}{\lvert Q \rvert} \sum_{q \in Q} \big(\Phi(x;\theta) - \Phi(x \setminus x_{>q};\theta)\big), \\
    \textsc{Suff}_{Q \mid x, \Phi} = \frac{1}{\lvert Q \rvert} \sum_{q \in Q} \big(\Phi(x;\theta) - \Phi(x \setminus x_{\leq q}; \theta)\big).
\end{align}
The benefit of these metrics is that they are symmetrical, and invariant to the scaling of the attributions due to using quantiles to produce the masks.
Following the procedure outlined by \maincitet{erasercompsuff} we set the quantiles to $Q = (0.01, 0.05, 0.2, 0.5)$.
Figs.~\ref{fig:attention-maps} and \ref{fig:attention-maps-2} show additional attributions, while Figs.~\ref{fig:occlusion} and~\ref{fig:occlusion2} illustrate the occlusions with the selected quantiles.
While we show that SPiT produces strong attributions, the proposed method is by no means free of failure cases.
We find it informative to also include visualizations of the limiting edge cases for attributions in Fig.~\ref{fig:attention-maps-cont}.

\section{Unsupervised Salient Segmentation Details}
\label{sec:segmentation}
\begin{table*}[tb]
  \colorlet{downcolor}{black!55}
  \newcommand{\dc}{\color{downcolor}}
  \sisetup{detect-all,
    uncertainty-separator=\pm,
    table-format=1.3,
    uncertainty-mode=separate,
  }
  \caption{Extended unsupervised salient segmentation results. Models including extensive decoders and postprocessing are colored in \textcolor{downcolor}{gray}.} 
  \label{tab:salient_segmentation_ext}
  \scriptsize
  \centering
  \begin{tabular}{l@{ }@{ }l@{ }@{ }l@{ }@{ }SSSSSSSSS}
    \toprule
    & & & \multicolumn{3}{c}{\textsc{ECSSD}} & \multicolumn{3}{c}{\textsc{DUTS}} & \multicolumn{3}{c}{\textsc{DUT-OMRON}} \\
    \cmidrule(r){4-6} 
    \cmidrule(r){7-9} 
    \cmidrule(r){10-12}
    Model & Method & Postproc. & {$\max F_\beta$}  & {IoU} & {Acc.} & {$\max F_\beta$}  & {IoU} & {Acc.} & {$\max F_\beta$}  & {IoU} & {Acc.} \\
    \midrule
    \dc DINO-B14\tdgg & \dc TokenCut & \dc BL
    & \dc 0.874 & \dc 0.772 & \dc 0.934
    & \dc 0.755 & \dc 0.624 & \dc 0.914
    & \dc 0.697 & \dc 0.618 & \dc 0.897\\
    \dc DINO-S8  & \dc SelfMask & \dc MF        
    & \dc 0.894 & \dc 0.779 & \dc 0.943 
    & \dc 0.789 & \dc 0.648 & \dc 0.938 
    & \dc 0.733 & \dc 0.609 & \dc 0.923 \\
    \dc DINO-S8  & \dc SelfMask & \dc MF+BL     
    & \dc 0.911 & \dc 0.803 & \dc 0.951 
    & \dc 0.819 & \dc 0.694 & \dc 0.949 
    & \dc 0.774 & \dc 0.677 & \dc 0.939 \\
    \dc DINO-S8 & \dc MOVE & \dc Seg+MF    
    & \dc 0.921 & \dc 0.835 & \dc 0.956 
    & \dc 0.829 & \dc 0.728 & \dc 0.954 
    & \dc 0.756 & \dc 0.666 & \dc 0.933 \\
    \dc DINO-S8 & \dc MOVE & \dc Seg+MF+BL 
    & \dc 0.917 & \dc 0.800 & \dc 0.952 
    & \dc 0.827 & \dc 0.687 & \dc 0.952 
    & \dc 0.766 & \dc 0.665 & \dc 0.937 \\
    DINO-B14\tdgg & TokenCut &\xmark
    & 0.803 & 0.712 & 0.918
    & 0.672 & 0.576 & 0.903
    & 0.600 & 0.533 & 0.880 \\
    SPiT-B16 & TokenCut & \xmark
    & 0.903 & 0.773 & 0.934
    & 0.771 & 0.639 & 0.894
    & 0.711 & 0.564     & 0.868\\
    \bottomrule
    \multicolumn{11}{l}{\tiny \tdgg As reported by \maincitet{tokencut}.}
  \end{tabular}
\end{table*}

The TokenCut~\maincitep{tokencut} framework proposes to use a normalized cut~\maincitep{normalizedcut} over the key features without class tokens in the last self-attention layer of the network.
A soft adjacency $A_{\mathrm{TC}}$ is computed using cosine similarities, which are thresholded using a small threshold $\tau_{\mathrm{TC}} = 1/3$ to estimate adjacency over the complete graph over token features.
The normalized cut is performed by extracting the Fiedler vector; the second smallest eigenvector of the graph Laplacian, and gives a bipartition of the graph into foreground and background elements. The original paper~\maincitep{tokencut} uses DINO~\maincitep{dino} as a pre-trained base model.

We found that extracting the key tokens from the last self-attention operator in the network is less effective than simply using the final features for the SPiT framework.
In TokenCut, the saliency map is refined using postprocessing with a bilateral solver, however, in the SPiT framework this step is clearly redundant.
Instead, we simply standardize the Fiedler vector using its mean and standard deviation, and map the result on the segmentations from the SPiT tokenizer.
For certain images, the foreground and background elements could be swapped under the standard unsupervised normalized cut method.
From our experiments on interpretability, we found that simply taking the class token for the full image, and comparing it using cosine similarity to class tokens (produced given the saliency mask) will accurately provide a robust estimate of which element is the foreground and the background.

\section{Feature Correspondences}
\label{sec:ftcorr}

The work by \maincitet{dino} and \maincitet{dinov2} established certain emergent properties in self-supervised models, where the tokenized features of ViT trained with self-supervised methods provide inherent interpretability and inter-image feature correspondence. 
Given our results on feature attributions from Section~\mainref{sec:dense_evaluation}, we perform experiments to visualize feature correspondences to see if similar emergent properties can be observed from supervised training with superpixel tokenization.

\subsubsection{Method:}
A sequence of support images $(\xi_n)_{n=1}^N$ are selected, with labels such that $y_n \neq y_m$ for all $1 \leq m,n \leq N$, as a set of features to search from. 
Furthermore, these images are selected such that the WordNet~\maincitep{wordnet} hypernym of the labels are the same, \eg, `dog' in Fig.~\ref{fig:featcorr_ext}.
We compute the normalized superpixel token features ${\big(z_n : z_n = \Phi(\xi_n) / \lVert \Phi(\xi_n) \rVert\big)_{n=1}^N}$ where $\Phi$ is our SPiT-B16 model with gradient including feature extraction. 
We omit the class tokens, and compute correspondences via cross-attention for each pair $A^{\times}_{mn} = \sigma(z_m z_n^\intercal)$ where $\sigma$ is the softmax with a temperature of $0.01$. 
For visualizing the high-dimensional features, we compute a three component PCA to produce pseudocolors $c_n = \mathrm{PC}_3(z_n) \in \mathbb{R}^{3 \times k}$ for $z_n \in \mathbb{R}^{k\times d}$, where $\mathrm{PC}_3(\cdot)$ extracts the first $3$ principal components.
These are normalized to $[0,1]$, and mapped to RGB channels directly in the respective order.\footnote{We use \texttt{torch.pca\_lowrank}, which has nondeterministic behaviour. 
Slight deviations in pseudocolors could therefore occur when reproducing the visualizations.}
The idea is to visualize the correspondences using PCA pseudocolors from the source image mapped to a target image.
The principal components are thresholded using normalized cut, \cf Section~\ref{sec:segmentation}, and the feature correspondences are computed via 
\begin{align}
    c_{m \rightarrow n} = (A^{\times}_{mn})^\intercal c_m ,
\end{align}
such that $c_{m \rightarrow n}$ are the projected feature correspondences from the cross attention $A^{\times}_{mn}$ over the full feature space using pseudocolors from the source image $m$ into the target image $n$.
In other words, for each superpixel in the target image, we mix the pseudocolors of the corresponding superpixels in the source image and visualize them as the transferred pseudocolors.
Given that these correspondences are directed due to the softmax operator, we compute the correspondences for every support image to illustrate the effect of using different source image mappings.

In Fig.~\ref{fig:featcorr_ext}, we see that these feature correspondences pick up on the nuances of the different breeds of dogs, and are able to map similar parts between images, even with multiple instances of dogs in the same image.
In Fig.~\ref{fig:featcorr_ext2}, we extend the experiment to a broader class of mammals with similar, albeit slightly less clear correspondence.
In particular, the second row of Fig.~\ref{fig:featcorr_ext2} illustrates a case where the feature correspondences exhibit less structure.

Notably, our model has not been trained with contrastive self-supervised approaches, and the features are derived from a model trained only supervisedly on IN1k.
Moreover, the class tokens are removed before computing the cross attention and PCA, which confirms that the tokenized features themselves are informative for discriminative tasks.

\setlength{\fsz}{0.165\linewidth}
\begin{figure}[tb]
\centering
\footnotesize
\begin{tblr}{
  colspec={Q[c,m]Q[c,m]Q[c,m]Q[c,m]Q[c,m]},
  % make no space between rows
  abovesep=0pt,
  belowsep=0pt,
  stretch=0,
  % set equal spacing between cols and rows
  rowsep=1.5pt,
  colsep=1.5pt,
  column{1} = {font=\scriptsize},
  row{1} = {font=\scriptsize},
}
\SetCell[c=5]{c} \textbf{Original Images} \\
\includegraphics[width=\fsz, height=\fsz]{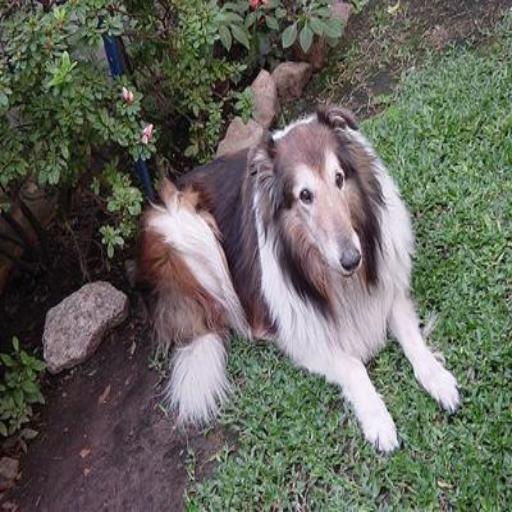} &
\includegraphics[width=\fsz, height=\fsz]{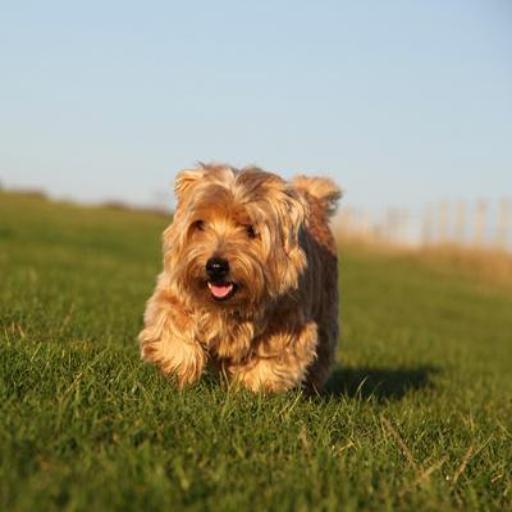} &
\includegraphics[width=\fsz, height=\fsz]{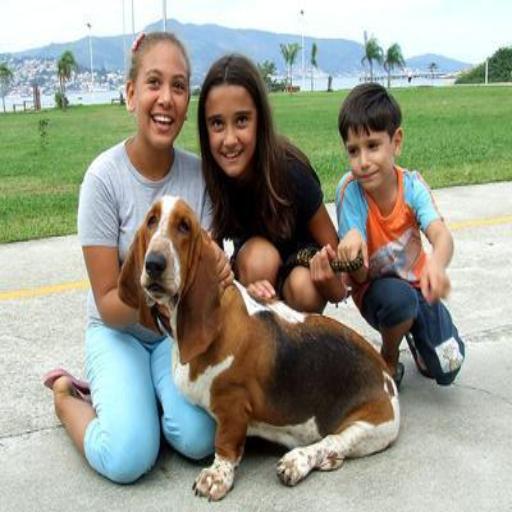} &
\includegraphics[width=\fsz, height=\fsz]{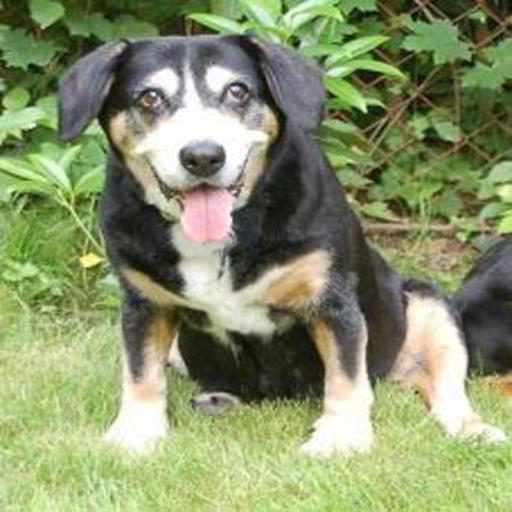} &
\includegraphics[width=\fsz, height=\fsz]{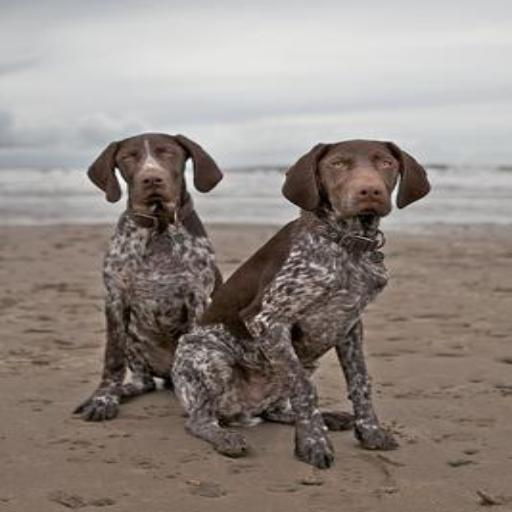}
\end{tblr}
\begin{tblr}{
  colspec={Q[c,m]Q[c,m]Q[c,m]Q[c,m]Q[c,m]},
  % make no space between rows
  abovesep=0pt,
  belowsep=0pt,
  stretch=0,
  % set equal spacing between cols and rows
  rowsep=1.5pt,
  colsep=1.5pt,
  column{1} = {font=\scriptsize},
  column{2} = {leftsep=10pt},
  row{1} = {font=\scriptsize},
}
\SetCell[c=1]{c} \textbf{Source} & 
\SetCell[c=4]{c} \textbf{Feature Correspondences} \\
\includegraphics[width=\fsz, height=\fsz]{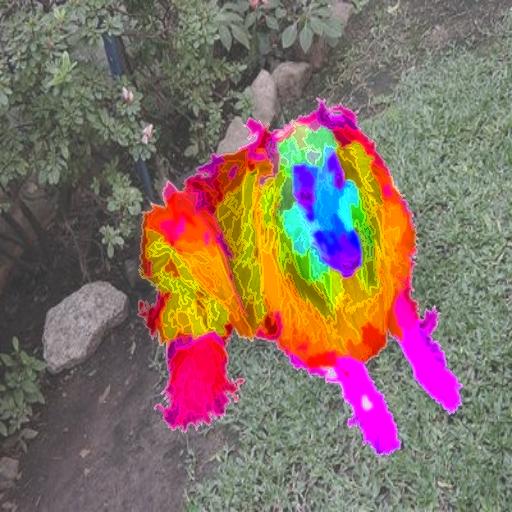} &
\includegraphics[width=\fsz, height=\fsz]{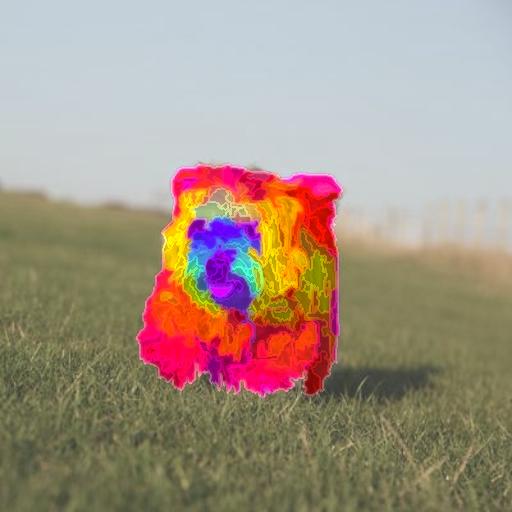} &
\includegraphics[width=\fsz, height=\fsz]{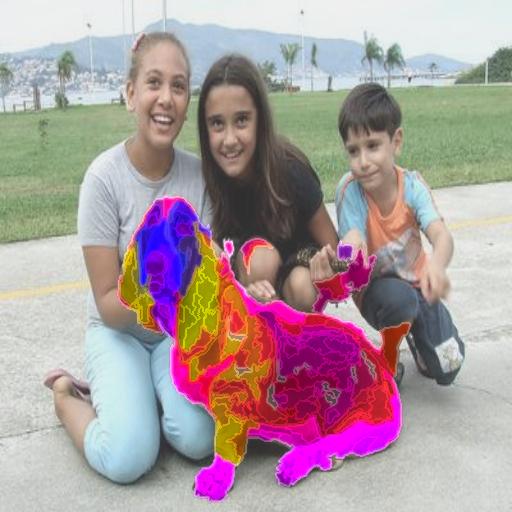} &
\includegraphics[width=\fsz, height=\fsz]{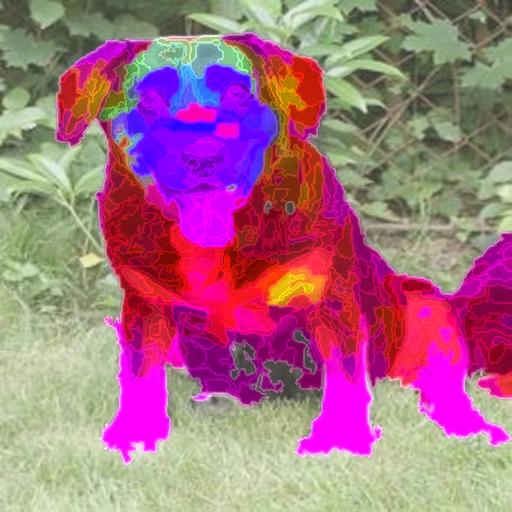} &
\includegraphics[width=\fsz, height=\fsz]{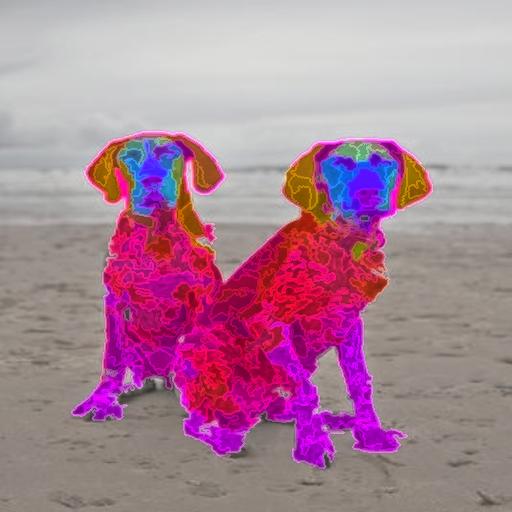} \\
\includegraphics[width=\fsz, height=\fsz]{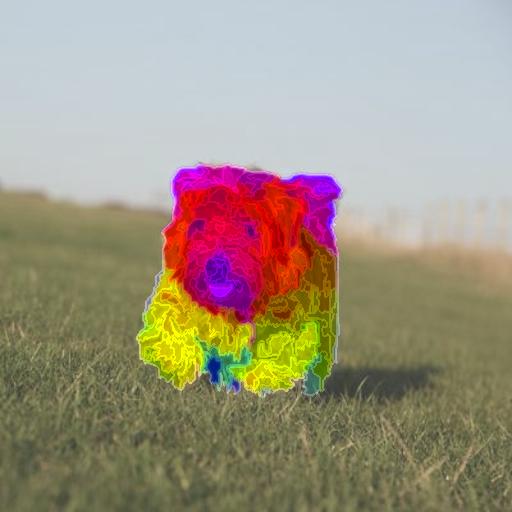} &
\includegraphics[width=\fsz, height=\fsz]{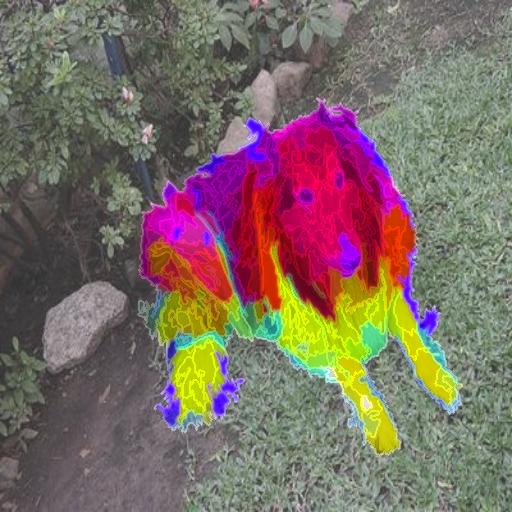} &
\includegraphics[width=\fsz, height=\fsz]{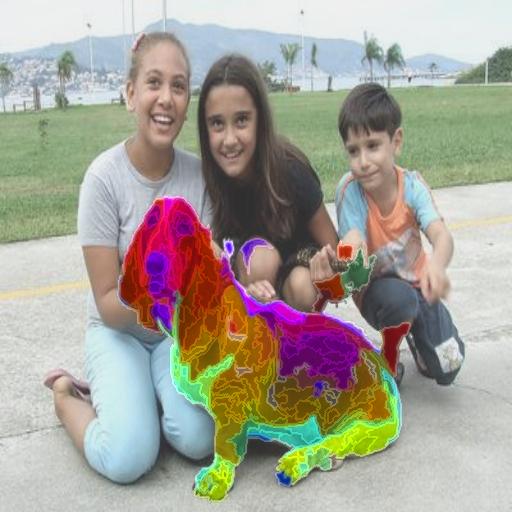} &
\includegraphics[width=\fsz, height=\fsz]{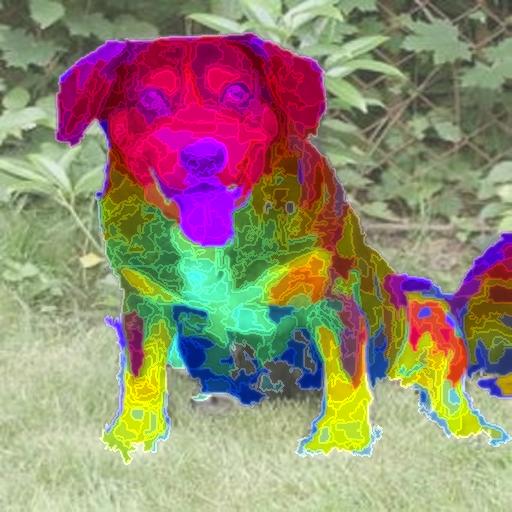} &
\includegraphics[width=\fsz, height=\fsz]{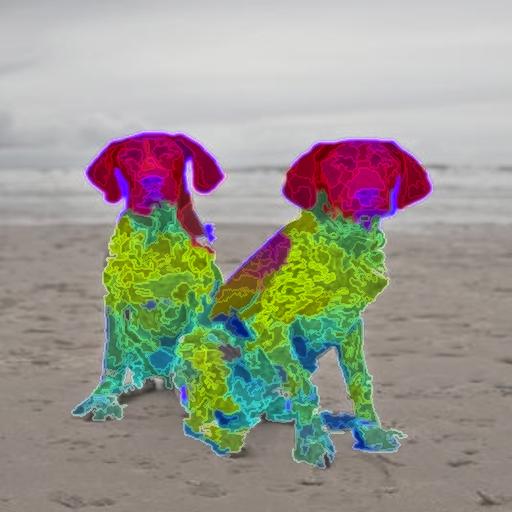} \\
\includegraphics[width=\fsz, height=\fsz]{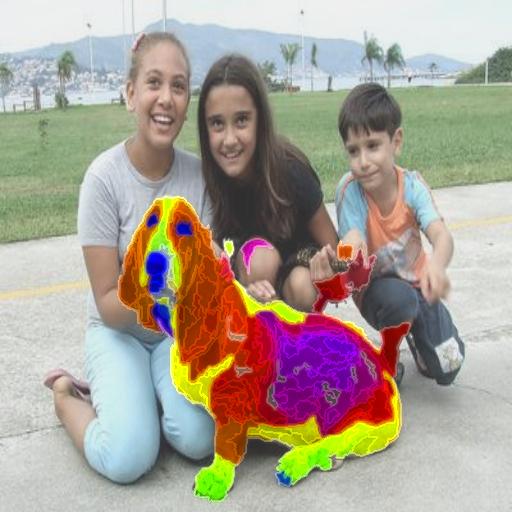} &
\includegraphics[width=\fsz, height=\fsz]{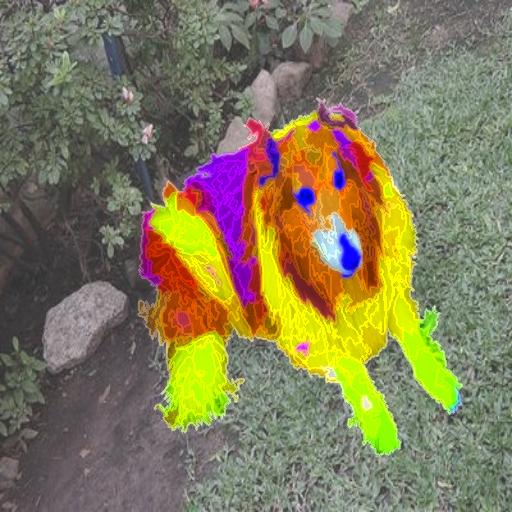} &
\includegraphics[width=\fsz, height=\fsz]{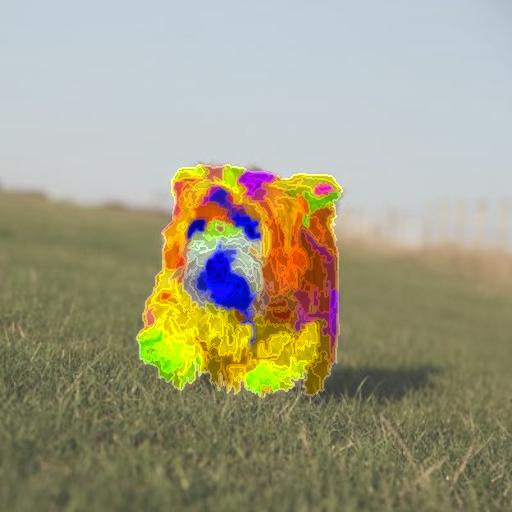} &
\includegraphics[width=\fsz, height=\fsz]{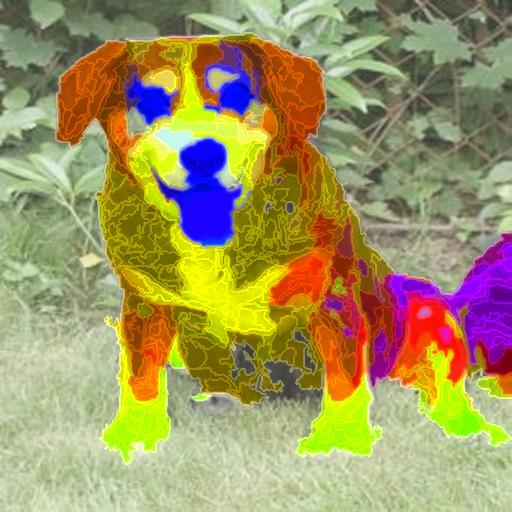} &
\includegraphics[width=\fsz, height=\fsz]{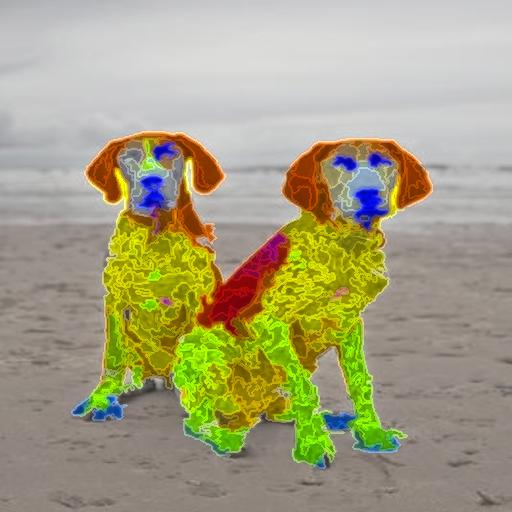} \\
\includegraphics[width=\fsz, height=\fsz]{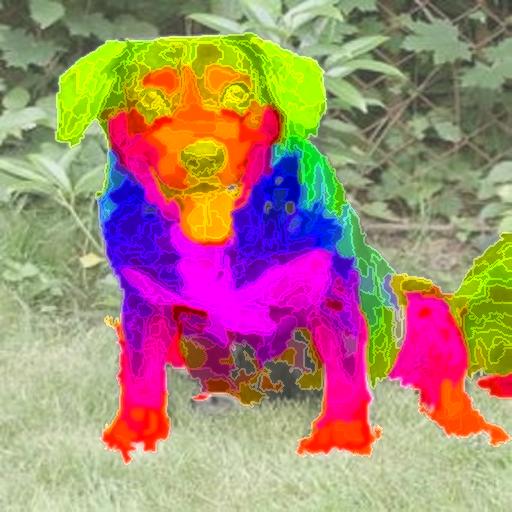} &
\includegraphics[width=\fsz, height=\fsz]{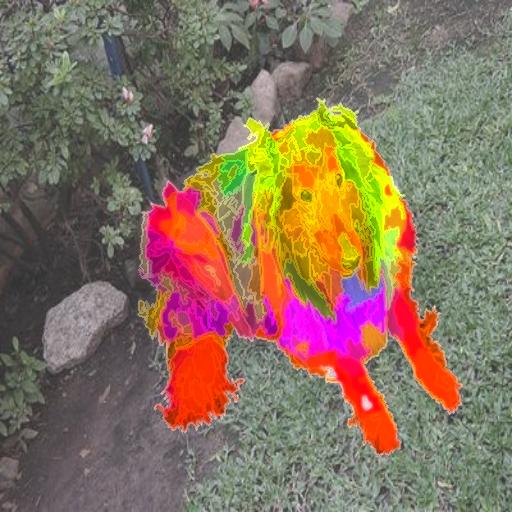} &
\includegraphics[width=\fsz, height=\fsz]{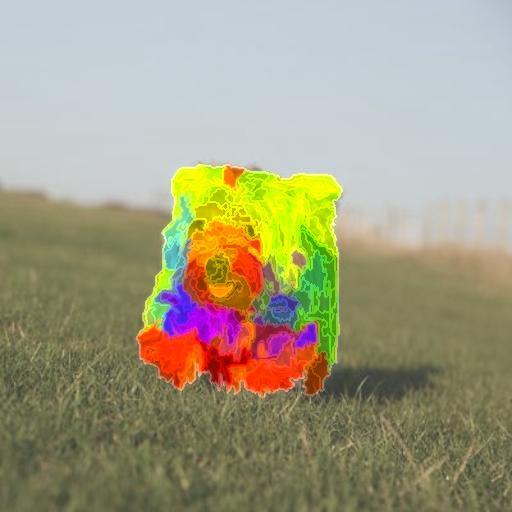} &
\includegraphics[width=\fsz, height=\fsz]{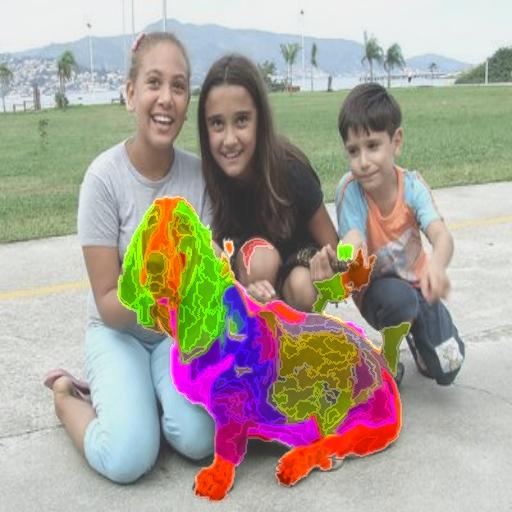} &
\includegraphics[width=\fsz, height=\fsz]{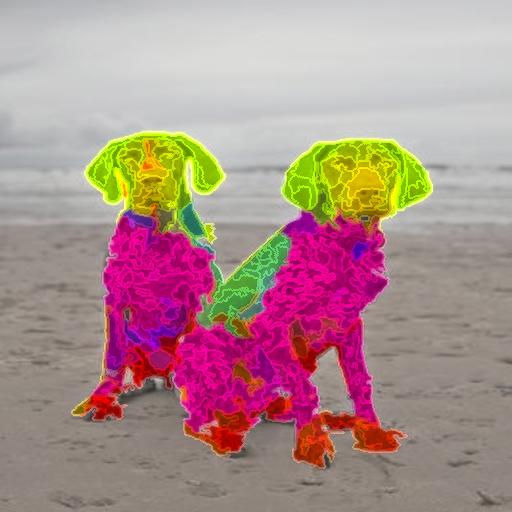} \\
\includegraphics[width=\fsz, height=\fsz]{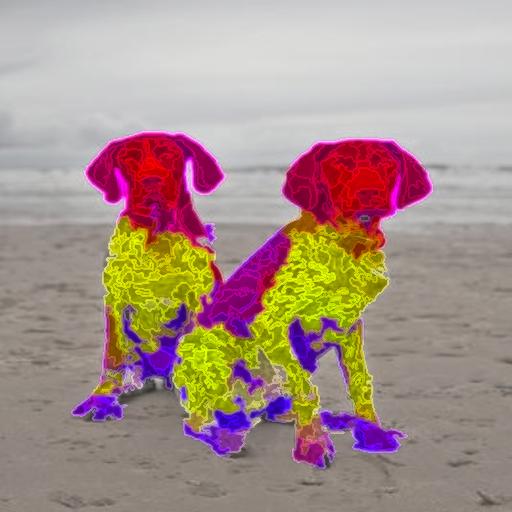} &
\includegraphics[width=\fsz, height=\fsz]{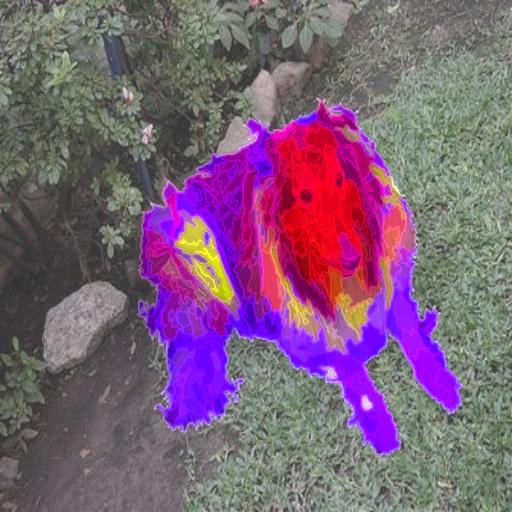} &
\includegraphics[width=\fsz, height=\fsz]{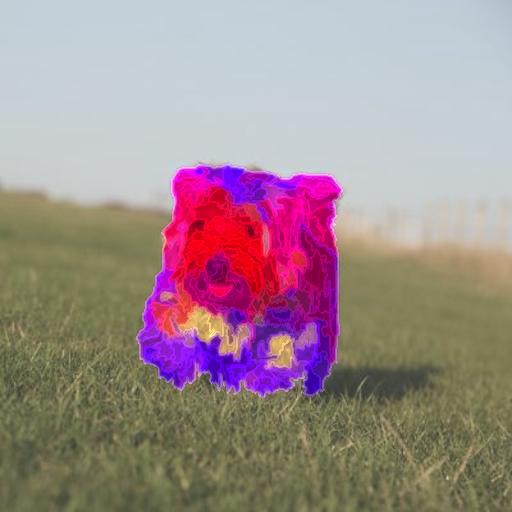} &
\includegraphics[width=\fsz, height=\fsz]{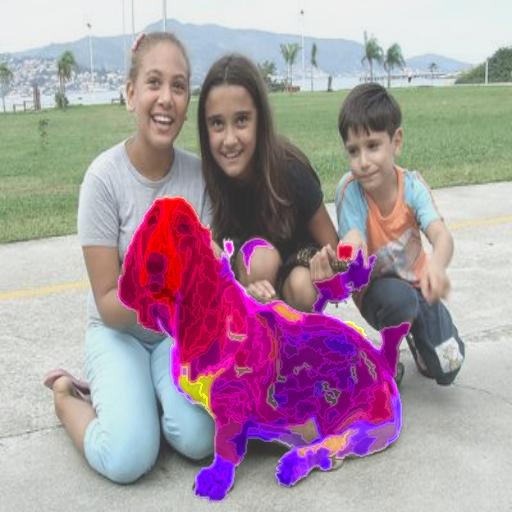} &
\includegraphics[width=\fsz, height=\fsz]{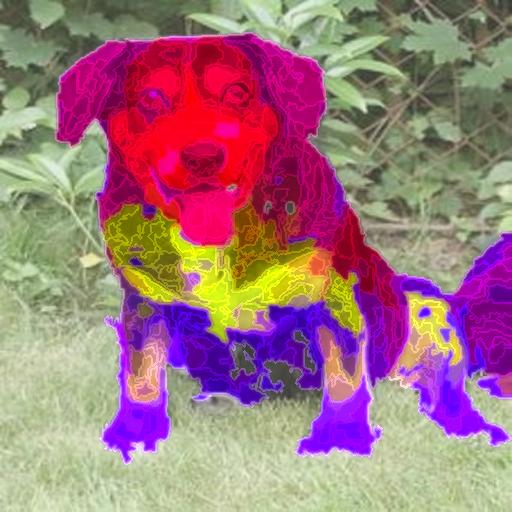}
\end{tblr}
\caption{
Visualization of feature correspondences from source features from superpixel tokens (left) to target images (right). Features are mapped via single head normalized cross attention between tokenized images, using pseudocolors from low rank PCA with three components. Images contain different classes (breeds) under the common hypernym ``\textit{domestic dog, canis familiaris}''.
% img1: ILSVRC2012_val_00000127
}
\label{fig:featcorr_ext}
\end{figure}

\setlength{\fsz}{0.165\linewidth}
\begin{figure}[tb]
\centering
\footnotesize
\begin{tblr}{
  colspec={Q[c,m]Q[c,m]Q[c,m]Q[c,m]Q[c,m]},
  % make no space between rows
  abovesep=0pt,
  belowsep=0pt,
  stretch=0,
  % set equal spacing between cols and rows
  rowsep=1.5pt,
  colsep=1.5pt,
  column{1} = {font=\scriptsize},
  row{1} = {font=\scriptsize},
}
\SetCell[c=5]{c} \textbf{Original Images} \\
\includegraphics[width=\fsz, height=\fsz]{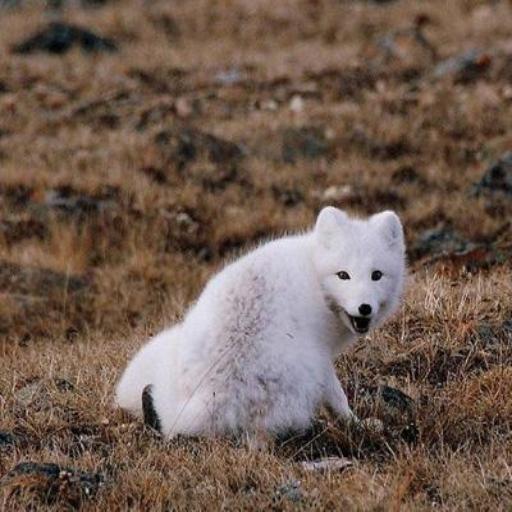} &
\includegraphics[width=\fsz, height=\fsz]{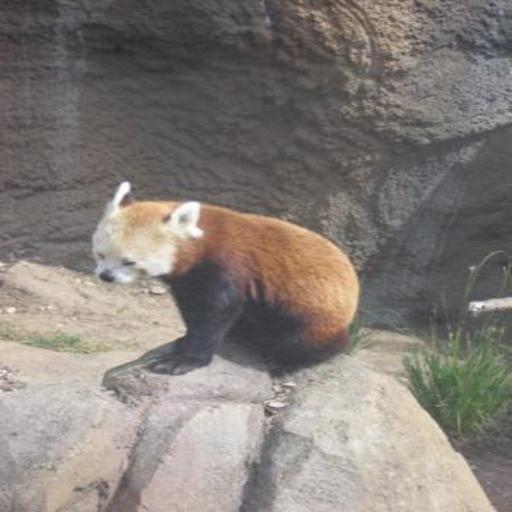} &
\includegraphics[width=\fsz, height=\fsz]{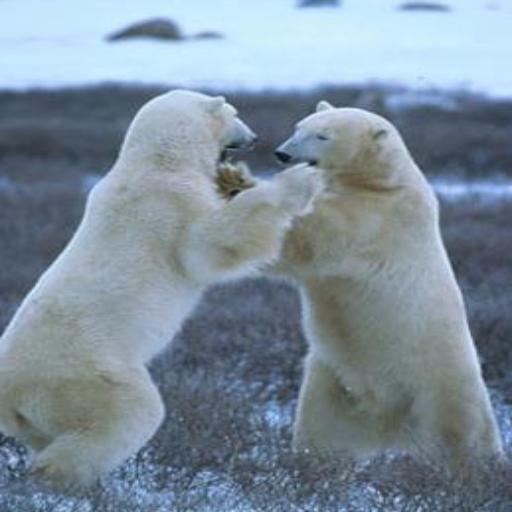} &
\includegraphics[width=\fsz, height=\fsz]{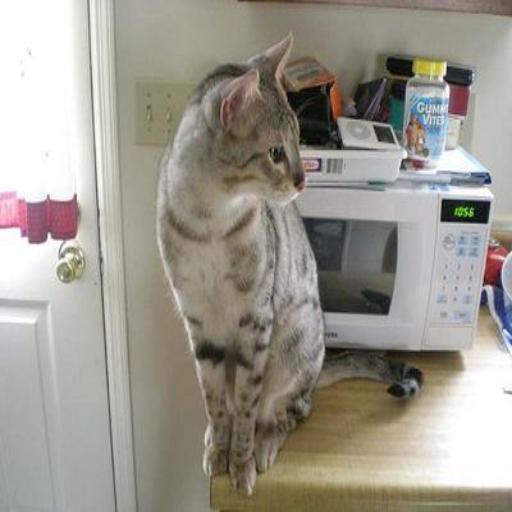} &
\includegraphics[width=\fsz, height=\fsz]{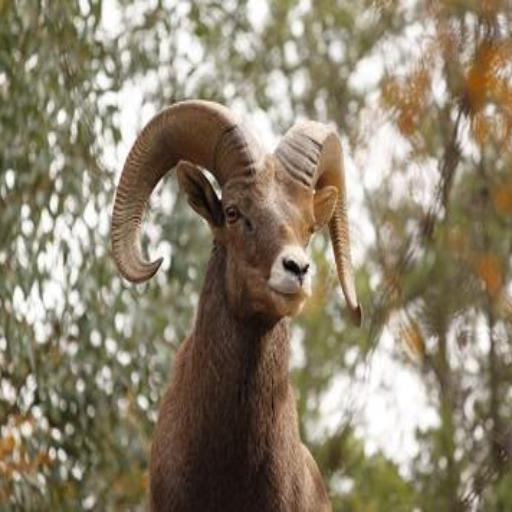}
\end{tblr}
\begin{tblr}{
  colspec={Q[c,m]Q[c,m]Q[c,m]Q[c,m]Q[c,m]},
  % make no space between rows
  abovesep=0pt,
  belowsep=0pt,
  stretch=0,
  % set equal spacing between cols and rows
  rowsep=1.5pt,
  colsep=1.5pt,
  column{1} = {font=\scriptsize},
  column{2} = {leftsep=10pt},
  row{1} = {font=\scriptsize},
}
\SetCell[c=1]{c} \textbf{Source} & 
\SetCell[c=4]{c} \textbf{Feature Correspondences} \\
\includegraphics[width=\fsz, height=\fsz]{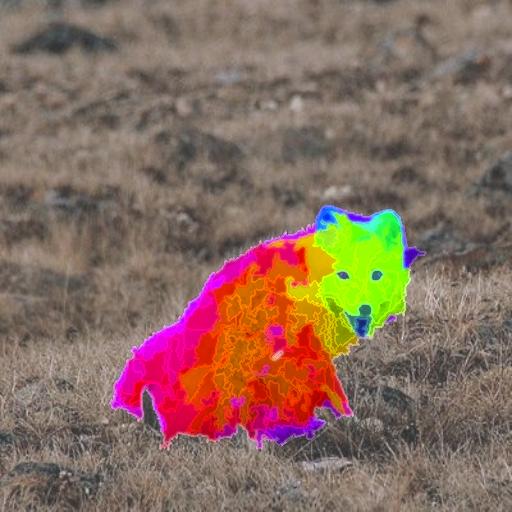} &
\includegraphics[width=\fsz, height=\fsz]{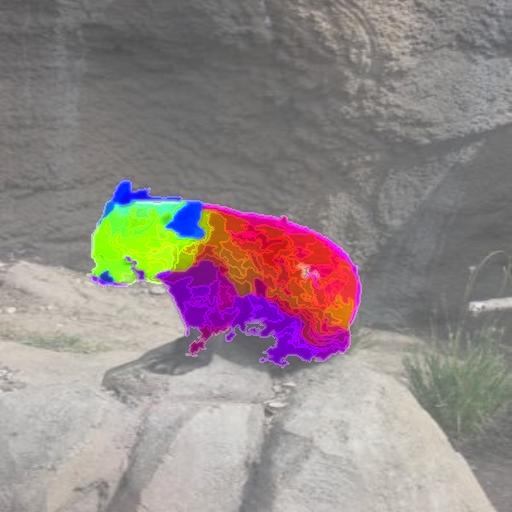} &
\includegraphics[width=\fsz, height=\fsz]{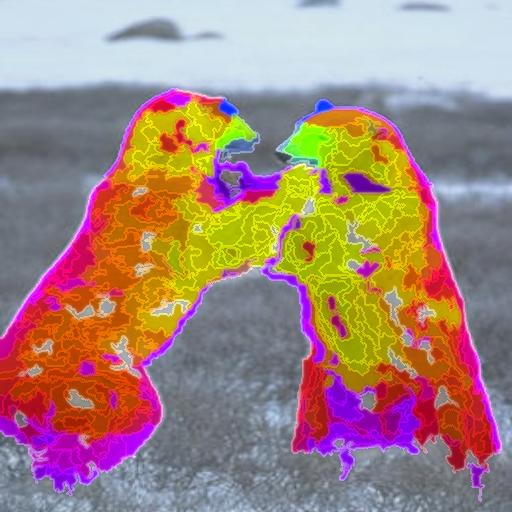} &
\includegraphics[width=\fsz, height=\fsz]{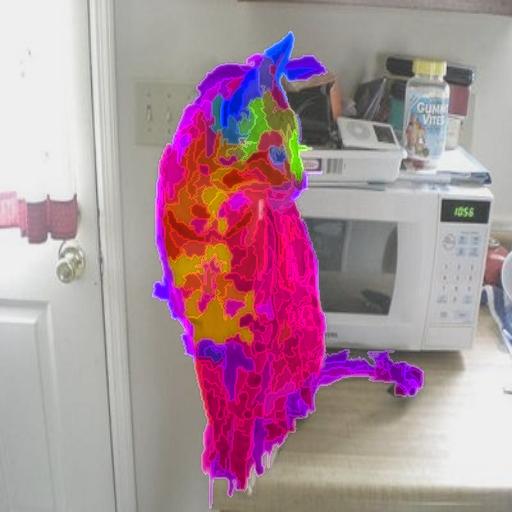} &
\includegraphics[width=\fsz, height=\fsz]{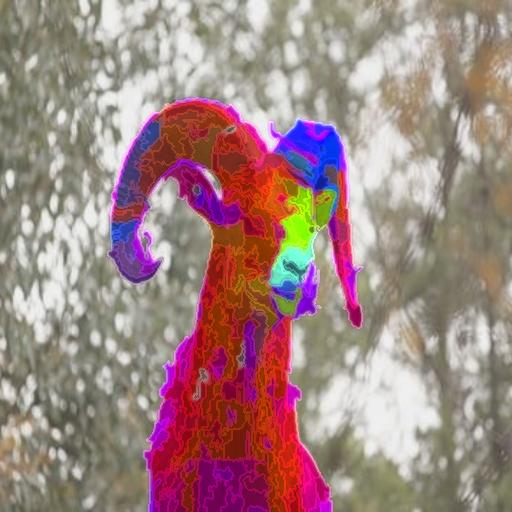} \\
\includegraphics[width=\fsz, height=\fsz]{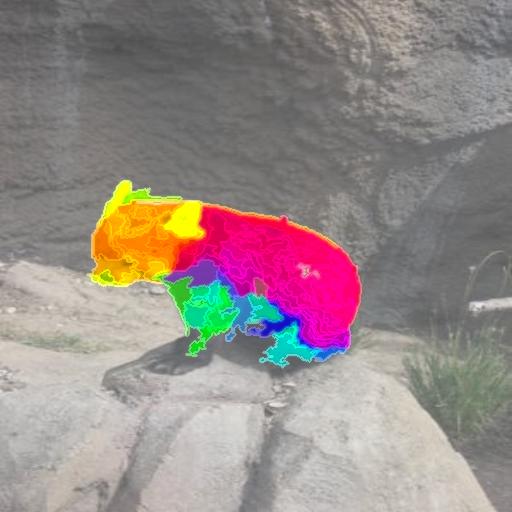} &
\includegraphics[width=\fsz, height=\fsz]{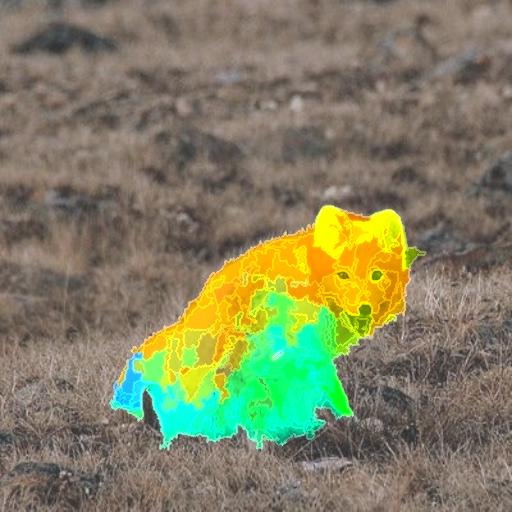} &
\includegraphics[width=\fsz, height=\fsz]{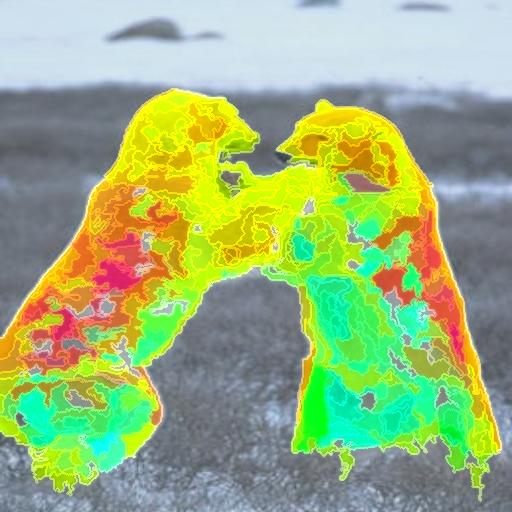} &
\includegraphics[width=\fsz, height=\fsz]{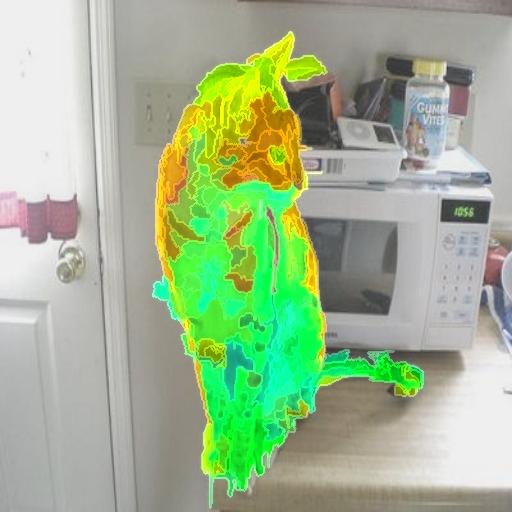} &
\includegraphics[width=\fsz, height=\fsz]{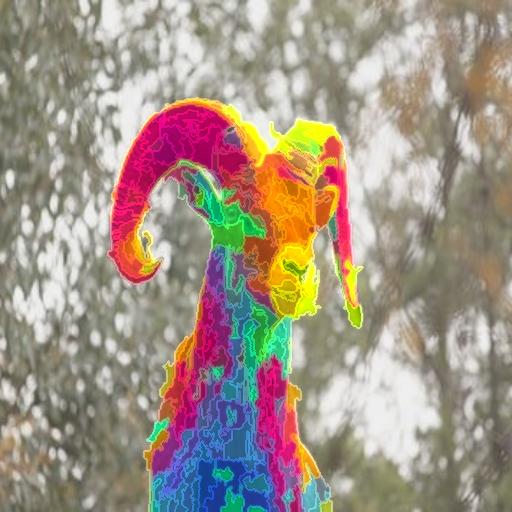} \\
\includegraphics[width=\fsz, height=\fsz]{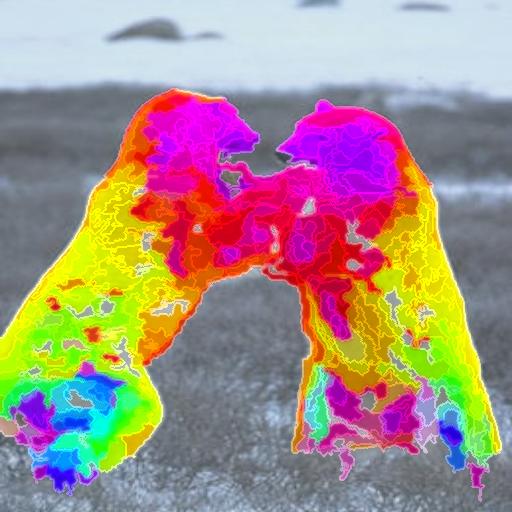} &
\includegraphics[width=\fsz, height=\fsz]{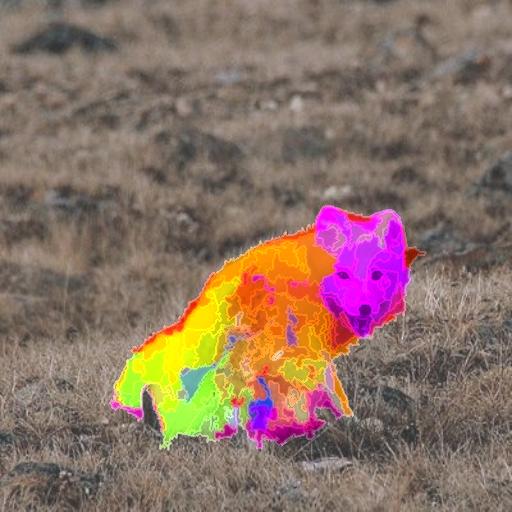} &
\includegraphics[width=\fsz, height=\fsz]{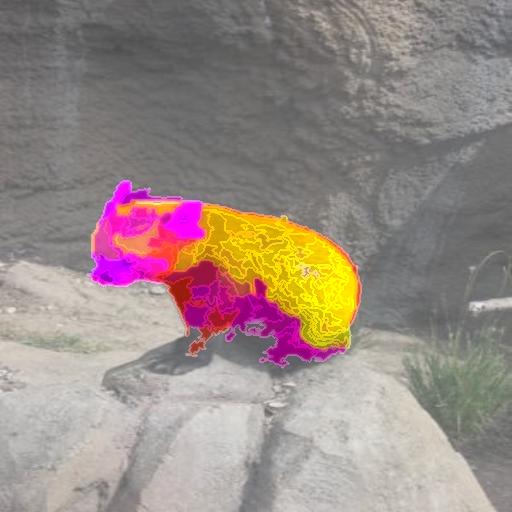} &
\includegraphics[width=\fsz, height=\fsz]{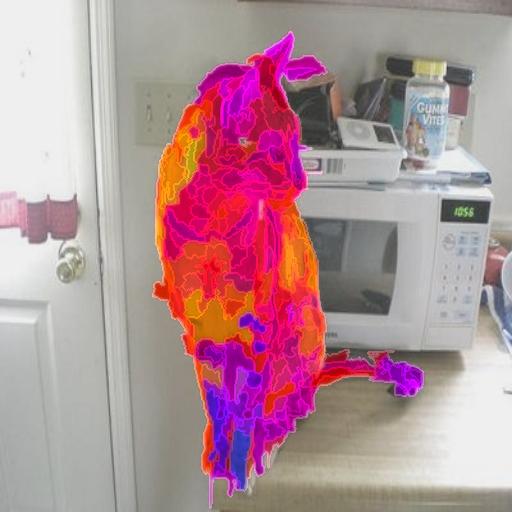} &
\includegraphics[width=\fsz, height=\fsz]{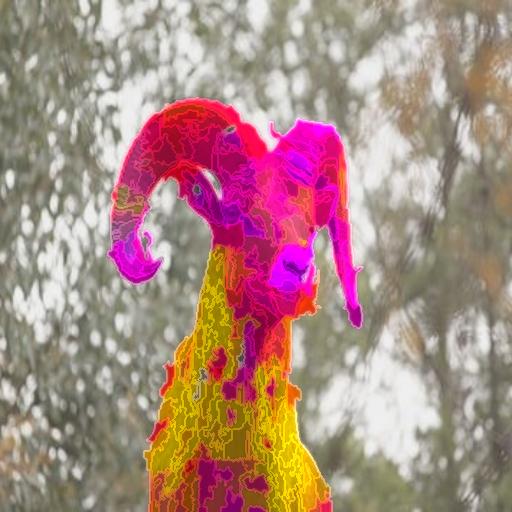} \\
\includegraphics[width=\fsz, height=\fsz]{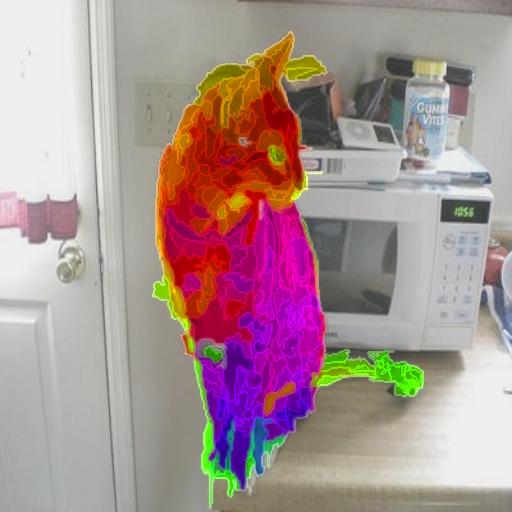} &
\includegraphics[width=\fsz, height=\fsz]{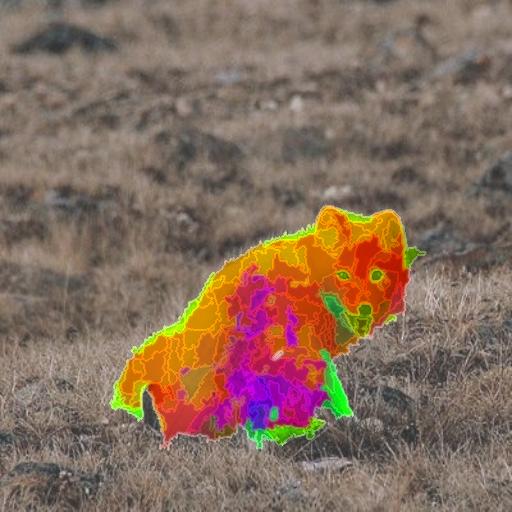} &
\includegraphics[width=\fsz, height=\fsz]{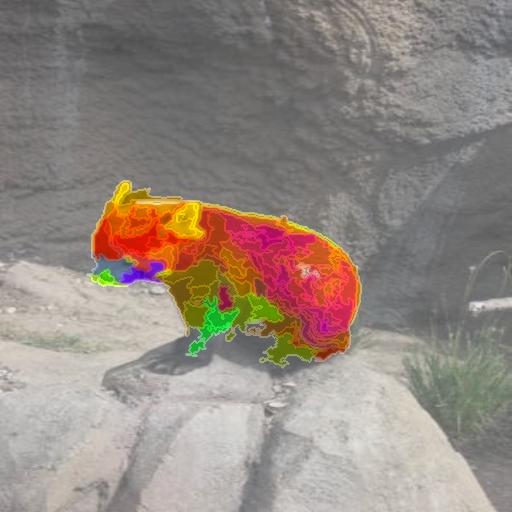} &
\includegraphics[width=\fsz, height=\fsz]{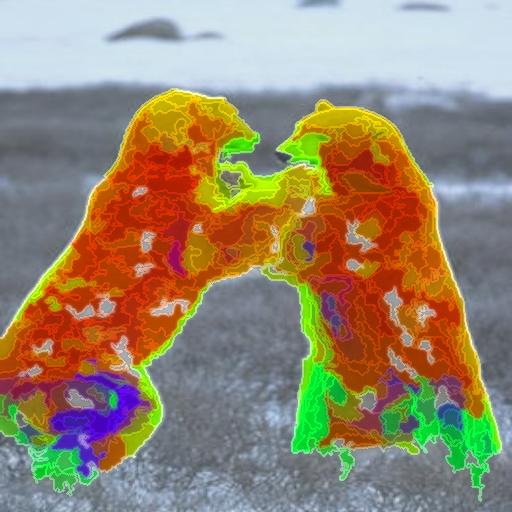} &
\includegraphics[width=\fsz, height=\fsz]{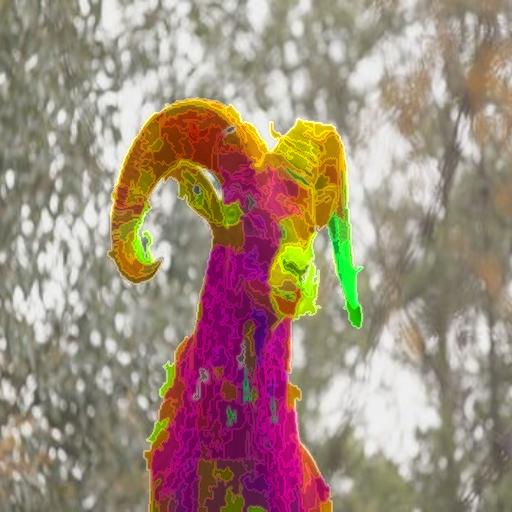} \\
\includegraphics[width=\fsz, height=\fsz]{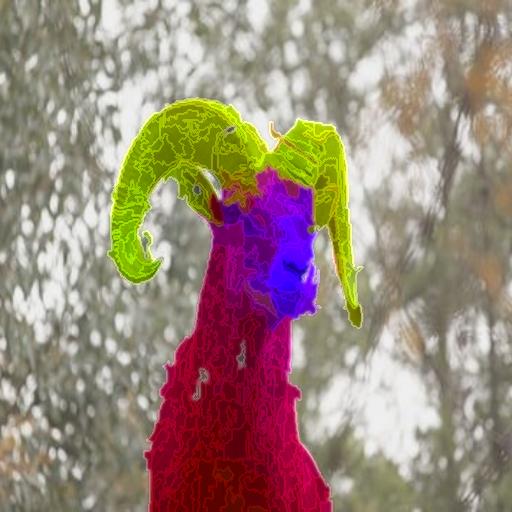} &
\includegraphics[width=\fsz, height=\fsz]{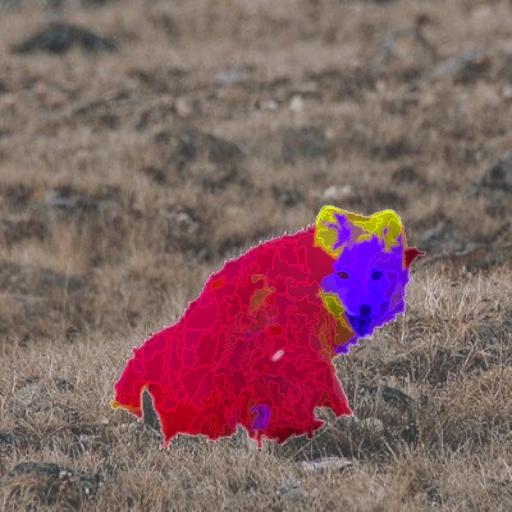} &
\includegraphics[width=\fsz, height=\fsz]{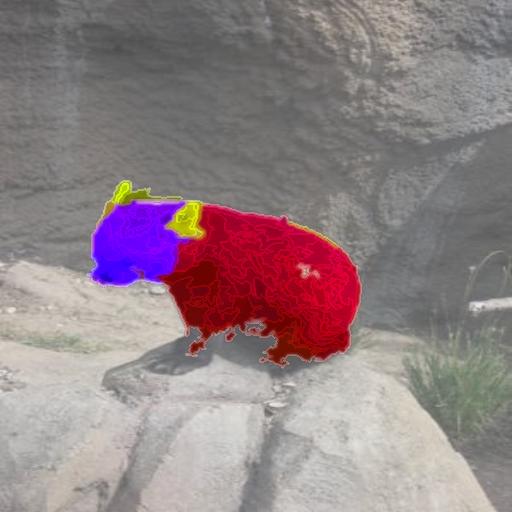} &
\includegraphics[width=\fsz, height=\fsz]{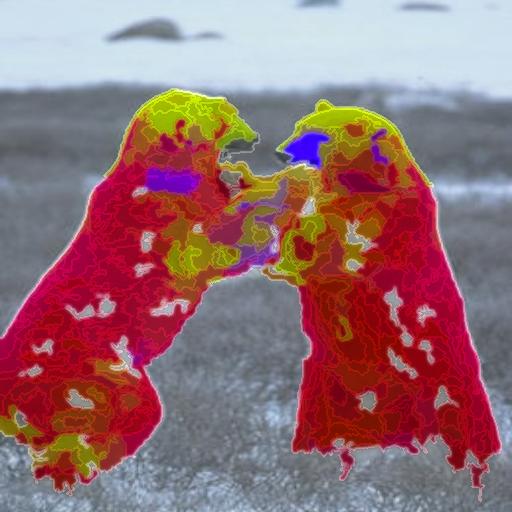} &
\includegraphics[width=\fsz, height=\fsz]{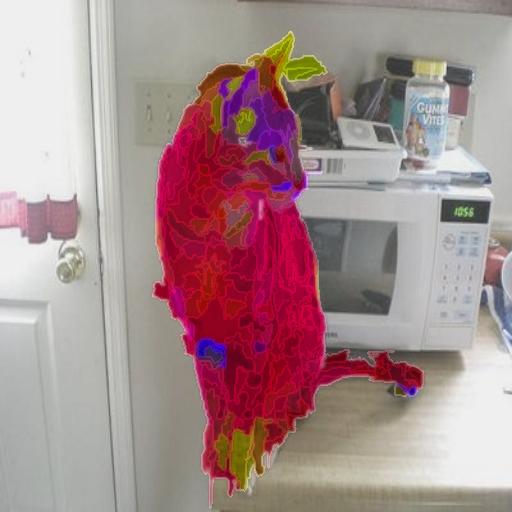}
\end{tblr}
\caption{
Visualization of feature correspondences from source features from superpixel tokens (left) to target images (right). Images contain different classes (species) under the common hypernym ``\textit{mammal}''. The second row (red panda) illustrates a case where the visualized feature mappings exhibit less structure than in the other examples.
% img1: ILSVRC2012_val_00000127
}
\label{fig:featcorr_ext2}
\end{figure}

\section{Extended Discussion on Classification}
\label{sec:ext-discussion}

\begin{table*}[tb]
  \sisetup{detect-all,
    uncertainty-separator=\pm,
    table-format=1.4(4),
    uncertainty-mode=separate,
  }
  \caption{Results w.\ CI (95\%) for models with RViT tokenizers (5 runs).} 
  \label{tab:voronoi_uncertainty}
  \scriptsize
  \centering
  \begin{tabular}{l@{ }@{ }l@{ }@{ }l@{ }@{ }cSSSS}
    \toprule
    \multicolumn{4}{c}{ViT Model} & \multicolumn{1}{c}{\textsc{IN1k} } & \multicolumn{1}{c}{\textsc{INReaL} } & \multicolumn{1}{c}{\textsc{Cifar100} } & \multicolumn{1}{c}{\textsc{Caltech256} } \\
    \cmidrule(r){1-4} \cmidrule(r){5-5} \cmidrule(r){6-6} \cmidrule(r){7-7} \cmidrule(r){8-8}
    Name   &  Tok.  & Feat.  & Grad.  & {Lin.}  & {Lin.}  & {Lin.} & {Lin.} \\
    \midrule
    RViT-S16 & RV    & Intp.  &  \xmark  & .7669(0.0002)  & 0.8285(0.0003) & 0.8557(0.0028) & 0.8521(0.0007)  \\
    RViT-S16 & RV    & Intp.  &  \cmark  & .7593(0.0003)  & 0.8183(0.0002) & 0.8563(0.0032) & 0.8558(0.0006) \\
    \midrule
    RViT-B16 & RV    & Intp.  &  \xmark  & .7878(0.0002)  & 0.8436(0.0002) & 0.8941(0.0043) & 0.8731(0.0007)\\
    RViT-B16 & RV    & Intp.  &  \cmark  & .7892(0.0002)  & 0.8414(0.0001) & 0.8875(0.0030) & 0.8644(0.0006)\\
    \bottomrule
  \end{tabular}
\end{table*}

Certain interesting observations can be made from our results in Table~\mainref{tab:results-main}.
Firstly, random Voronoi tessellations perform better than data-driven superpixels for gradient excluding features, and
despite its inherent stochasticity, tokenization with random Voronoi tessellations proves to be a relatively effective strategy, and demonstrate surprisingly consistent results over prediction tasks as reported in Table~\ref{tab:voronoi_uncertainty}.
To account for the stochasticity in validation, we compute accuracy scores over five runs and report 95\% confidence intervals in Table~\ref{tab:voronoi_uncertainty}.
We find that the segmentations based on the Voronoi tessellations produces remarkably consistent results over the validation set.

Additionally we note that gradient including tokenizers perform comparatively worse for small (S) models.
This is particularly noteworthy, since the gradient features are essentially an added set of features to the model.
We speculate that this could be an artifact of over-fitting on information-dense features, at the expense of the utility of the canonical pixel features.

\else
    \AddToHook{enddocument/afteraux}{%
        \immediate\write18{
        cp output.aux main.aux
        }%
    }
\fi
\end{document}